\documentclass[10pt]{article} %
\usepackage[preprint]{tmlr}

\usepackage{amsmath,amsfonts,bm}

\def\eqref#1{equation~\ref{#1}}

\def\1{\bm{1}}

\DeclareMathAlphabet{\mathsfit}{\encodingdefault}{\sfdefault}{m}{sl}
\SetMathAlphabet{\mathsfit}{bold}{\encodingdefault}{\sfdefault}{bx}{n}

\DeclareMathOperator{\Tr}{Tr}

\usepackage[hidelinks]{hyperref}
\usepackage{url}

\usepackage{subcaption}
\usepackage{common37}

\allowdisplaybreaks

\hypersetup{
    colorlinks,
    linkcolor={blue!50!black},
    citecolor={blue!50!black},
    urlcolor={blue!50!black}
}

\title{Revisiting Mixture Policies \\in Entropy-Regularized Actor-Critic}

\author{\name Jiamin He \email jiamin12@ualberta.ca \\
      \addr University of Alberta \& Amii
      \AND
      \name Samuel Neumann \email sfneuman@ualberta.ca \\
      \addr University of Alberta \& Amii
      \AND
      \name Jincheng Mei \email jcmei@google.com \\
      \addr Google DeepMind
      \AND
      \name Adam White \email amw8@ualberta.ca \\
      \addr University of Alberta \& Amii \\
      Canada CIFAR AI Chair
      \AND
      \name Martha White \email whitem@ualberta.ca \\
      \addr University of Alberta \& Amii \\
      Canada CIFAR AI Chair
}

\def\month{MM}  %
\def\year{YYYY} %
\def\openreview{\url{https://openreview.net/forum?id=XXXX}} %

\begin{document}

\maketitle

\begin{abstract}
  Mixture policies theoretically offer greater flexibility than unimodal policies in continuous action reinforcement learning, but the practical benefits of this complexity remain elusive. Mixture policies are notably absent from most state-of-the-art algorithms, raising a fundamental question: Is the added representational overhead useful? We show that increased flexibility can theoretically enhance solution quality and entropy robustness. Yet standard algorithms like SAC do not leverage these advantages. A core issue is the lack of a low-variance reparameterization trick for mixtures, a luxury Gaussian policies enjoy. We propose a marginalized reparameterization (MRP) estimator to address this, proving it offers lower variance than the standard likelihood-ratio (LR) approach. Our experiments across Gym MuJoCo, DeepMind Control Suite, and MetaWorld show that MRP mixture policies significantly outperform their LR ones, and reach parity (sometimes better) with Gaussian counterparts. In addition, we do find several cases where MRP mixture policies exhibit clear empirical advantages. In this paper, we provide a clearer understanding of the trade-offs involved, elevating MRP mixture policies from theoretical curiosity to a practical tool.
\end{abstract}

\section{Introduction}
\label{intro}

Policy gradient methods are widely used in online reinforcement learning (RL), particularly for continuous action spaces, and yet, there are many design decisions in these methods that remain underexplored. One is the choice of policy parameterization. 
Gaussian policies are by far the most common parameterization \citep{williams1992simple,lillicrap2015continuous,schulman2017proximal}, or its bounded variants like squashed Gaussian policies \citep{haarnoja2018soft}. There are a handful of works exploring other distributions, including beta policies \citep{chou2017improving} and the family of heavy-tailed policies \citep{kobayashi2019student,bedi2024sample}. Using mixture policies, such as a conditional Gaussian mixture model for the policy, remains largely unexplored.

Yet there are reasons that this increased flexibility from mixture policies could be beneficial. A more flexible policy class may contain better optimal policies. For example, when the environment is partially observable, the optimal policy could be stochastic and multimodal \citep{sutton2018reinforcement}. Even in fully observable settings, it is common to use entropy-regularized objectives, which prefer stochastic policies; multimodal rather than unimodal policies may be better in this regime. The increased flexibility may also facilitate exploration. Several works have shown that heavy-tailed policies may improve learning through better exploration compared with Gaussian policies \citep{kobayashi2019student,bedi2024sample}. 
Similarly, mixture policies offer a complementary approach by enabling mode-directed exploration, maintaining high probability for multiple promising actions.

Though underexplored, some work has looked at more flexible policy classes. Implicit policies use deep generative models (e.g., energy-based models \citep{haarnoja2017reinforcement,messaoud2024sac}, normalizing flows \citep{tang2018implicit,mazoure2020leveraging}, diffusion models \citep{wang2023diffusion}). Compared with these more complex implicit policies, policies using parametric distributions like mixture models have two benefits: they are simpler to train, and the explicit densities are useful in entropy-regularized RL. Otherwise, several unpublished works briefly touch on mixture policies. The first version of Soft Actor-Critic \citep[SAC;][]{haarnoja2018softarxiv} did use mixture policies, but \citeauthor{haarnoja2018softarxiv} did not pursue this further nor provide insights on this choice. We hypothesize that the lack of reparameterization (RP) gradient estimators for mixture policies may explain why later versions of SAC switched to a single Gaussian, as the RP estimator was found to perform better (see Footnote 3, p. 67 of \citealp{haarnoja2018acquiring}). Later, \citet{hou2020off} tried to avoid reparameterization of the whole mixture policy by using a separate objective for the weighting policy, but they found little improvement from their approach. Finally, \citet{baram2021maximum} explored the utility of the upper and lower bounds of the mixture model's entropy, without considering a learnable weighting policy. Appendix \ref{sec:app_related_works} provides an extended discussion on related work including latent variable policies \citep{zhang2023latent} which also uses RP.

In this work, we provide the first RP estimator specifically for mixture policies that does not compromise flexibility and is empirically viable. We first prove that mixture policies provide improved solution quality. They achieve comparable or better objective values and are more robust to larger entropy regularization, because stationary points may not exist for Gaussian policies but do for Gaussian mixture policies. 
We derive the Marginalized RP estimator (MRP) for mixture policies and prove that it has lower variance than the standard likelihood-ratio estimator.
RP estimators have been critical for the practical success of SAC, and we provide a similarly effective RP estimator for mixture policies.

We empirically study mixture policies and the proposed MRP estimator with SAC across a broad suite of benchmarks. In synthetic bandit experiments designed to mimic multimodal critic functions, mixture policies more often find the maximal peak than base policies. 
On larger benchmarks---including 7 Gym MuJoCo, 10 DeepMind Control Suite, 30 MetaWorld, 10 MyoSuite, and 6 classic control environments---mixture policies perform on par with base policies and significantly outperforms other estimators from the literature.  MRP does improve performance in environments with unshaped rewards because the critic is less smooth and exhibits more peaks, allowing the policy to exploit multiple modes more effectively. Taken together, we are the first to provide a practical and performant actor-critic algorithm with mixture policies with theoretically-sound reparameterization.

\section{Problem Formulation}
\label{background}

We consider the standard Markov decision process (MDP) problem setting. An MDP is defined by $\langle\cS, \cA, p, d_0, r, \gamma \rangle$, where $\cS$ is the state space, $\cA$ is the action space, $p$ is the transition function, $d_0$ is the initial state distribution, $r$ is the reward function, and $\gamma$ is the discount factor. In this paper, we consider $\cA$ to be continuous, and $r$ to be deterministic and bounded by $[-r_{\max},r_{\max}]$. 
The agent's goal is to find a policy $\pi$ that maximizes the \textit{expected return} from the start states:
\begin{equation}
\label{eq:objective_0}
    J_0(\pi) \doteq \bE_\pi\big[\textstyle\sum_{t=0}^\infty\gamma^t r(S_{t},A_{t})\big],
\end{equation}
where the expectation is with respect to the initial state distribution, transition function, and policy.

Oftentimes, the agent optimizes the \textit{entropy-regularized objective} that promotes stochastic policies:
\begin{align}
\begin{split}    
\label{eq:objective}
    J(\pi) \doteq \bE_\pi\big[\textstyle\sum_{t=0}^\infty \gamma^t \big(r(S_t,A_t) + \alpha \cH(\pi(\cdot|S_t))\big) \big]
    = \bE_{s\sim d_0,a\sim\pi(\cdot|s)} \big[ Q_\pi(s, a) - \alpha \log \pi(a|s) \big],
\end{split}
\end{align}
where $\alpha$ is the entropy scale, $\cH(q)\doteq-\int q(x)\log q(x)\,dx$ is the differential entropy for distribution $q(x)$, and $Q_{\pi}(s,a)\doteq \bE_\pi[ \sum_{t=0}^\infty \gamma^t (r(S_t,A_t) + \alpha \gamma \cH(\pi(\cdot|S_{t+1}))) ]$ is the soft action-value function.

Soft Actor-Critic \citep[SAC;][]{haarnoja2018soft} learns $\pi$ by maximizing a surrogate of \Cref{eq:objective}:
\begin{equation}
\label{eq:objective_surrogate}
    \hat J(\ppolicy) =\bE_{S_t\sim\mathcal{B},A_t\sim\ppolicy(\cdot|S_t)} \big[ Q_\qparams(S_t, A_t) - \alpha \log\ppolicy(A_t|S_t) \big],
\end{equation}
where $\mathcal{B}$ is a buffer of collected data,  $Q_\qparams$ is an estimate of $Q_\ppolicy$. In this work, we focus on the role of policy parameterization; we refer the reader to the original paper for other details on SAC.

We can obtain an unbiased sample $\hat \nabla_\pparams \hat J(\ppolicy)$ of the gradient of \Cref{eq:objective_surrogate} in two ways. One is the \textit{likelihood-ratio} (LR) gradient estimator \citep{williams1992simple}:
\begin{align}
\label{eq:gradient_likelihood}
    \hat \nabla_\pparams \hat J(\ppolicy) = 
    \nabla_\pparams \log\ppolicy(A_t|S_t)  \big(Q_\qparams(S_t,A_t) - \alpha
    \log\ppolicy(A_t|S_t) \big).
\end{align}
The LR estimator often suffers from high variance, and a baseline can reduce variance.
When the action is reparameterizable as $A_t\!=\!f_\pparams(\epsilon_t;S_t)$ with $\epsilon_t$ sampled from a prior distribution $p(\cdot)$. Alternatively we can use the \textit{reparameterization} (RP) gradient estimator \citep{heess2015learning}:
\begin{align}
    \hat \nabla_\pparams \hat J(\ppolicy) = 
    \nabla_\pparams \big(Q_\qparams(S_t, f_\pparams(\epsilon_t;S_t)) - \alpha \log\ppolicy(f_\pparams(\epsilon_t;S_t)|S_t)\big). 
    \label{eq:gradient_reparam} 
\end{align}
\textit{Gaussian policies} are a common choice when the action space is continuous:
\begin{equation}
\label{eq:gaussian_policy}
    \ppolicy(a|s) = \mathcal{N}(a; \mu_\pparams(s), \sigma_\pparams(s)^2),
\end{equation}
where $\mu_\pparams(s)$ is the mean and $\sigma_\pparams(s)$ is the standard deviation. Gaussian policies have infinite support, but the action space is typically bounded in practice. To address the bias of clipping actions, \textit{squashed Gaussian policies} use $\tanh$ to transform the unbounded support to a bounded interval:
\begin{equation}
\label{eq:squashed_gaussian_policy}
   \ppolicy(a|s) = \mathcal{N}(\tanh^{-1}(a); \mu_\pparams(s), \sigma_\pparams(s)^2) \quad\text{ for } a\in[0,1].
\end{equation}
In this paper, we study \textit{mixture policies} with $N\in\bN^+$ components:
\begin{align*}
    \pmixturepolicy (a | s) = \textstyle \sum_{\compidx=1}^N \pweightpolicy (\compidx|s) \pcomppolicy (a|s,k),
\end{align*}
where $\pi^\weight_{\pparams}$ is the \textit{weighting policy} and $\pi^\component_{\pparams}$ with different $k$ are the \textit{component policies}. 
When needed, we also explicitly write $\pi^\mixture_{\pparams^\mixture} (a | s) = \textstyle \sum_{\compidx=1}^N \pi^\weight_{\pparams^\weight} (\compidx|s) \pi^\component_{\pparams^\component_\compidx} (a|s)$, where $\pparams^\mixture=\left[\smash{\pparams_1^\component}^\top, \cdots, \smash{\pparams_N^\component}^\top, \smash{\pparams^\weight}^\top\right]^\top$. The weighting policy is usually parameterized as a softmax policy, while the component policies may be any continuous policies. When the component policies are Gaussian (\Cref{eq:gaussian_policy}), we refer to the resulting policy as the \textit{Gaussian mixture (GM) policy}. In this case, we call the Gaussian policy the \textit{base policy}. Similarly, when the base policy is a squashed Gaussian (\Cref{eq:squashed_gaussian_policy}), the resulting mixture policy is the \textit{squashed Gaussian mixture (SGM) policy}. 
{ 
While we focus on Gaussian-based policies in our analysis and empirical evaluations, many of our results also apply to other policy classes.
}

\section{Mixture Policies Robustness to Entropy Regularization}
\label{sec:motivation_in_bandit}

Policy parameterization influences the set of stationary points of the entropy-regularized objective in \Cref{eq:objective_surrogate}, which is non-concave \citep{agarwal2019reinforcement}. 
We first show that mixture policies improve solution quality, in terms of achieving higher regularized and unregularized objective values under entropy-constrained optimization. Then we show that mixtures are robust to higher levels of entropy, both theoretically and empirically. 
Section \ref{sec:app_theory} contains proofs for all theoretical results.

\subsection{Optimality of Stationary Points}
\label{sec:optimal_stationary_point}

We first show that the optimal stationary points, namely $\pparams^*\doteq\arg\max_{\pparams\in\{\pparams\mid\nabla_\pparams J(\pi_\pparams)=\vect 0\}} J(\pi_\pparams)$, of the mixture policy is at least as good as or better than the base policy in Proposition \ref{thrm:entropy_regularized_optimality}.
\begin{proposition}
\label{thrm:entropy_regularized_optimality}
When both $\pi^b_{\pparams^{b,*}}$ and $\pi^m_{\pparams^{m,*}}$ exist, then $J(\pi^m_{\pparams^{m,*}}) \ge J(\pi^b_{\pparams^{b,*}})$.
\end{proposition}
The inequality is likely strict when the return landscape is multimodal, as the mixture policy can maintain high returns while increasing its entropy by splitting its density into different modes.

The next natural question is how their optimal stationary points compare regarding the unregularized objective (the expected return) $J_0(\ppolicy)$. In general, it is difficult to guarantee $J_0(\pi^m_{\pparams^{m,*}}) \ge J_0(\pi^b_{\pparams^{b,*}})$. However, when the entropy is imposed as a constraint instead of regularization, the mixture policy is guaranteed to be at least as good as the base policy, as shown in Proposition \ref{thrm:entropy_constrained_optimality}.
\begin{proposition}
\label{thrm:entropy_constrained_optimality}
Consider entropy-constrained policy optimization $\max_{\pparams}J_0(\pi_\pparams)$ subject to $\cH(\pi_\pparams)\ge H$ for some $H>0$ and define the optimal solution as $\pparams'$, then $J_0(\pi^m_{\smash{\pparams^{m,}}'}) \ge J_0(\pi^b_{\smash{\pparams^{b,}}'})$.
\end{proposition}

\subsection{Non-existence of Stationary Points Under Strong Entropy Regularization}
\label{sec:stationary_point_theory}

This section shows that the mixture policy may have stationary points in scenarios where the base policy does not. We focus on the bandit setting where the objective is:
\begin{equation}
\label{eq:objective_bandit}
J(\ppolicy) = \mathbb{E}_{a\sim\ppolicy}\big[r(a) - \alpha\log \ppolicy(a)\big].
\end{equation}
Proposition \ref{thrm:sad_fact_about_gaussian} shows that for sufficiently large entropy scales, there are no stationary points for Gaussian policies.
\begin{proposition}
\label{thrm:sad_fact_about_gaussian}
Assume $r:\mathcal{A}\to\bR$ is an integrable function on $\mathcal{A}=\bR$. For all $\alpha > {2}r_{\max}$, $J(\pi_{\mu,\sigma})=\bE_{a\sim \mathcal{N}(\mu,\sigma)} \big[r(a) - \alpha \log \mathcal{N}(a;\mu,\sigma) \big] $ does not have any stationary point.
\end{proposition}
On the other hand, Gaussian mixture (GM) policies are less sensitive to entropy regularization. Specifically, we show that for every stationary point of the regularized objective with a base policy, there exists a corresponding set of stationary points for the mixture policy. 

\begin{proposition}
\label{cor:alpha_relation}
The minimum $\alpha$ after which the mixture policy no longer has a stationary point 
is at least as large as that of the base policy, i.e.,
$
\alpha_{\min}^{\pi^m} \ge \alpha_{\min}^{\pi^b},
$
where $\alpha_{\min}^\pi=\inf\{\alpha \mid \nabla_{\pparams}J(\ppolicy)\neq \vect 0,\, \forall \pparams\}$ for policy $\ppolicy$.
\end{proposition}
\begin{remark}
    
    While we only compare the mixture policy with the base policy in \Cref{thrm:entropy_regularized_optimality,thrm:entropy_constrained_optimality,cor:alpha_relation}, the results can be generalized to compared mixture policies with different $N$.
\end{remark}
In fact, instances where the inequality in \Cref{cor:alpha_relation} is strict are not hard to construct. The bimodal bandit in Figure~\ref{fig:analytical_bimodal_bandit} (Left) is one such example.

\begin{figure*}[t]
    \centering
    \resizebox{\textwidth}{!}{
    \includegraphics[height=3cm]{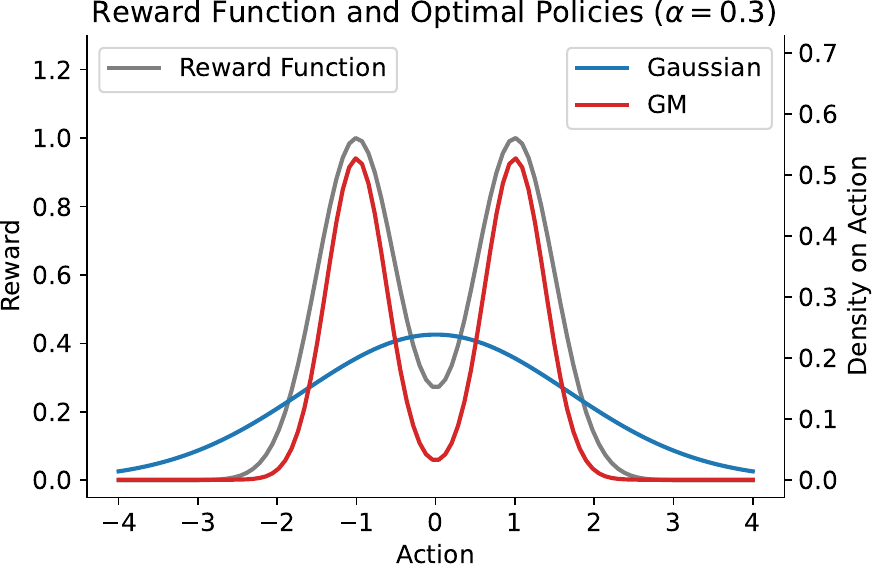}
    \includegraphics[height=3cm]{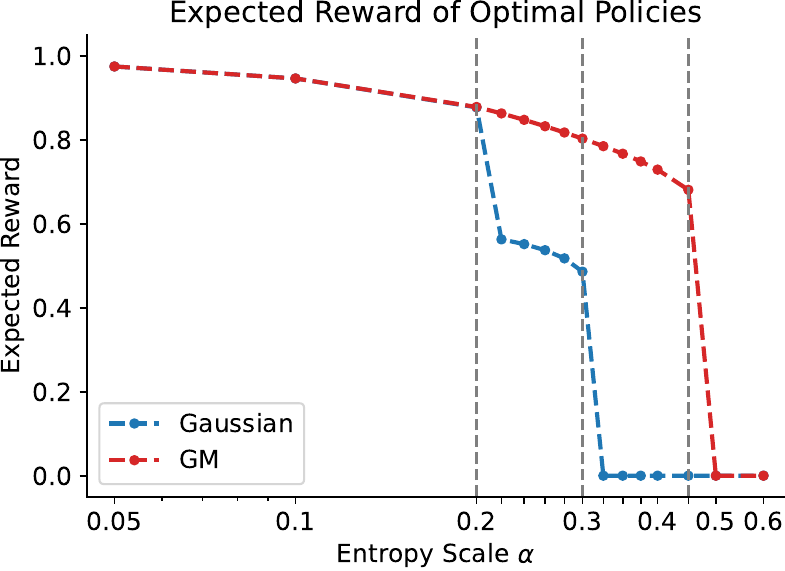}
    \includegraphics[height=3cm]{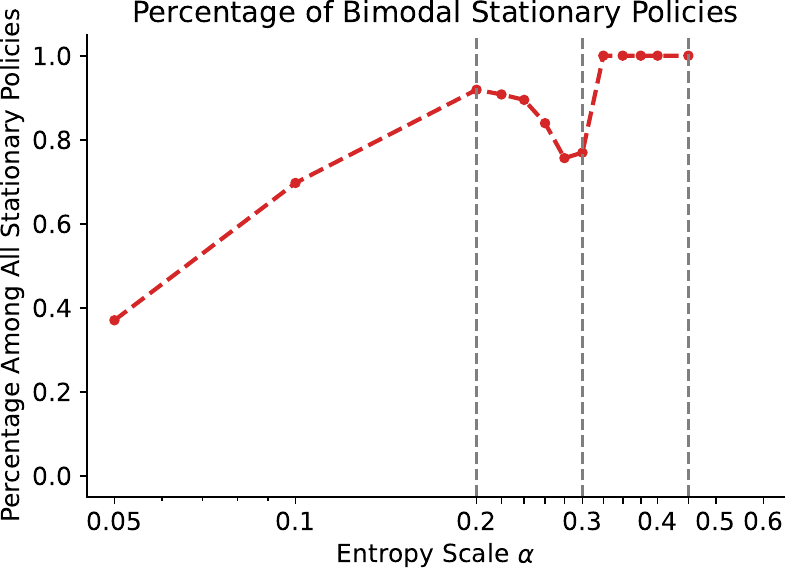}
    }
    \caption{
    Stationary points of Gaussian and two-component Gaussian Mixture (GM) policies in a bimodal bandit. We use local gradient-based optimization on \Cref{eq:objective_bandit} (100 trials; details in \Cref{sec:app_numerical}). 
    \textbf{Left:} Bandit reward function and optimal policies at $\alpha=0.3$. 
    \textbf{Middle:} Expected reward of optimal policies at different $\alpha$; 
    zero means no stationary points found. 
    \textbf{Right:} Percentage of bimodal policies among stationary GM policies. 
    Dashed vertical lines mark different levels of entropy regularization, each producing qualitatively distinct effects on Gaussian and GM policies.
    }
    \label{fig:analytical_bimodal_bandit}
\end{figure*}

\textbf{Robustness to entropy regularization.}
Figure \ref{fig:analytical_bimodal_bandit} (Middle) plots the expected reward $J_0(\pi_{\pparams^*})$ of the optimal stationary points for the two policies. When $\alpha\ge 0.325$, the Gaussian policy admits no stationary points, while the GM policy retains stationary points up to $\alpha \le 0.45$. Moreover, for $\alpha\in[0.225,0.3]$, the optimal stationary GM policy outperforms its Gaussian counterpart. In this regime, the GM policy captures both modes of the reward function with two separate components, while the Gaussian policy collapses to a single mode centered between them (\Cref{fig:analytical_bimodal_bandit}, Left).

\textbf{Preference for multimodality increases with larger entropy regularization.}
\Cref{fig:analytical_bimodal_bandit} (Right) shows a clear trend: as $\alpha$ increases, the frequency of bimodal policies also increases.
This suggests that a larger entropy scale helps prevent mode collapse in mixture policies, i.e., the loss of multimodality when the policy becomes effectively unimodal.

\section{The Marginalized Reparameterization (MRP) Estimator}
\label{sec:reparam_mixture_policies}

While there is no general guarantee that the RP estimator is better than the LR estimator, it is shown to have lower variance under certain assumptions \citep{xu2019variance}, be more robust in high dimensional settings \citep{mohamed2020monte}, and perform better empirically \citep{mohamed2020monte,zhang2024model}.
Further, the LR estimator often requires a baseline to reduce variance for it to work well. In RL, this baseline is typically either a learned state-value function or the average of sampled estimated action values. The former requires an additional function approximator, adding complexity and making the algorithm more sensitive, while the latter is prone to instability in high-dimensional settings, as we demonstrate in Section \ref{sec:exp_lr_vs_rp}. Thus, the RP estimator is often preferred in high-dimensional continuous control \citep{haarnoja2018soft,fujimoto2018addressing}.

Despite this advantage, reparameterization for mixture policies remains underexplored, as their softmax weighting policies are not directly reparameterizable. Prior work circumvents this issue by either using a fixed weighting policy \citep{baram2021maximum} or learning the weighting policy with a separate objective \citep{hou2020off}. As a result, little is known about how to reparameterize mixture policies without sacrificing flexibility or altering the algorithm---a gap we address next.

\subsection{Deriving the MRP Estimator}
\label{sec:half_rp_estimator}

We first extend the reparameterization policy gradient theorem \citep{lan2022model} to the case of mixture policies where we reparameterize only the components in \Cref{thrm:half_reparam_entropy_pg}.
Here, we assume the component policies are reparameterized as
$\pcomppolicy(a|s, \compidx) = p(\epsilon)$, 
where $a=f_\pparams(\epsilon; s, \compidx)$.
We also define the (discounted) occupancy measure under $\pi$ as $d_\pi(s)\doteq\bE_\pi\big[\sum_{t=0}^\infty\gamma^t\bI(S_t=s)\big]$, which quantifies the expected discounted state visitation frequency under $\pi$.

\begin{assumption}
\label{assm:nice_set}
$\mathcal{S}$ and $\mathcal{A}$ are compact.
\end{assumption}
\begin{assumption}
\label{assm:nice_function}
$p(s'|s,a)$, $d_0(s)$, $r(s,a)$, $f_\pparams(\epsilon;s,k)$, $f_\pparams^{-1}(a;s,k)$, $\pweightpolicy(k|s)$, $\pcomppolicy(a|s,k)$, $p(\epsilon)$, and their derivatives are continuous in variables $s$, $a$, $s'$, $\pparams$, and $\epsilon$.
\end{assumption}
\begin{theorem}[Entropy-Regularized Half-Reparameterization Policy Gradient Theorem]
\label{thrm:half_reparam_entropy_pg}
Under Assumptions \ref{assm:nice_set} and \ref{assm:nice_function}, we have
\begin{align*}
    \nabla_\pparams J(\pmixturepolicy)
    = \bE_{s\sim d_\pmixturepolicy,k\sim\pweightpolicy(\cdot|s), \epsilon\sim p} \Big[& 
    \nabla_\pparams \log\pweightpolicy(k|s) \big(Q_\pmixturepolicy(s, f_\pparams(\epsilon;s,k)) - \alpha \log\pmixturepolicy(f_\pparams(\epsilon;s,k)|s)\big) \\
    & + {
    \nabla_\pparams f_\pparams(\epsilon;s,k) \nabla_a \big(Q_\pmixturepolicy(s,a) - \alpha \log\pmixturepolicy(a|s)\big)|_{a=f_\pparams(\epsilon;s,k)}
    }\Big].
\end{align*}
\end{theorem}

Similarly, 
we can obtain the \textit{half-reparameterization} (HalfRP) estimator for SAC's objective:
\begin{align}
    \hat \nabla_\pparams \hat J(\pmixturepolicy)
    =& \nabla_\pparams \log\pweightpolicy(K_t|S_t) \big( Q_\qparams(S_t, f_\pparams(\epsilon_t; S_t, K_t)) - \alpha \log\pmixturepolicy(f_\pparams(\epsilon_t; S_t, K_t) | S_t) \big) \nonumber \\
    & + \nabla_\pparams \big( Q_\qparams(S_t, f_\pparams(\epsilon_t; S_t, K_t)) - \alpha \log\pmixturepolicy(f_\pparams(\epsilon_t; S_t, K_t)|S_t) \big). 
    \label{eq:gradient_half_reparam}
\end{align}

While the HalfRP estimator still suffers from high variance 
(see \Cref{sec:experiments_estimators}), we propose a better alternative by marginalizing over the mixing weights in the HalfRP estimator--in other words, we marginalize the random variable $K_t$. This gives rise to a new policy gradient theorem and estimator.
\begin{theorem}[Entropy-Regularized Marginalized-Reparameterization Policy Gradient Theorem]
\label{thrm:expected_reparam_entropy_pg}
Under Assumptions \ref{assm:nice_set} and \ref{assm:nice_function}, we have
\begin{align*}
    \nabla_\pparams J(\pmixturepolicy)
    = \bE_{s\sim d_\pmixturepolicy, \epsilon\sim p} \big[ \textstyle
    \nabla_\pparams \sum_{k=1}^N \pweightpolicy(k|s) \big(Q_\pmixturepolicy(s, f_\pparams(\epsilon;s,k)) - \alpha \log\pmixturepolicy(f_\pparams(\epsilon;s,k)|s)\big) \big].
\end{align*}
\end{theorem}
The corresponding \textit{marginalized-reparameterization} (MRP) estimator for SAC's objective is
\begin{align}
    \hat \nabla_\pparams \hat J(\pmixturepolicy)
    = \textstyle \nabla_\pparams \sum_{k=1}^N \pweightpolicy(k|S_t) \big( Q_\qparams(S_t, f_\pparams(\epsilon_t; S_t, k)) - \alpha \log\pmixturepolicy(f_\pparams(\epsilon_t; S_t, k) | S_t) \big).
    \label{eq:gradient_expected_reparam}
\end{align}

To our knowledge, \Cref{eq:gradient_expected_reparam,eq:gradient_half_reparam} provide the first two unbiased RP-based gradient estimators for mixture policies. 
Next, we establish conditions for a provably lower variance of the MRP estimator compared with the LR estimator.

\subsection{Variance Reduction Properties of the MRP Estimator}
\label{sec:var_reduction_mrp}

For clarity, we focus on the bandit setting without regularization: $J(\ppolicy)\!=\!\mathbb{E}_{a\sim\ppolicy}[r(a)]$. This setting allows us to isolate the effect of mixture reparameterization and provide clean, interpretable variance comparisons.
The entropy term can be included by redefining $r(a)$ to be the sum of the reward and sample entropy. The extension to the MDP setting is straightforward by considering state-conditioned variance. 

We analyze the gradient of the MRP and LR estimators with respect to the {distribution parameters} and the weighting probabilities of Gaussian mixture policies. In this case, the mixture policy can be expressed as $\pmixturepolicy(a) = \sum_{k=1}^N \pi^\weight_{\pparams^\weight} (\compidx) \rawpcomppolicyatidx (a) = \sum_{k=1}^N w_\compidx \cN(a;\mu_\compidx, \sigma_\compidx^2)$ where $\mu_\compidx$, $\sigma_\compidx$, and $w_\compidx$ may depend on shared parameters implicitly and $\pparams=\left[\smash{\pparams^{\component_1}}^\top, \cdots, \smash{\pparams^{\component_N}}^\top, \smash{\pparams^\weight}^\top\right]^\top = \left[ [\mu_1, \sigma_1]^\top, \cdots, [\mu_N, \sigma_N]^\top, [w_1, \cdots, w_N]^\top \right]^\top\in \bR^{3N}$. The corresponding estimators are
\begin{align*}
    \text{LR: }\ \hat \nabla_\pparams^\text{LR} J(\pmixturepolicy) &= \nabla_\pparams \log \pmixturepolicy(A) r(A), \,\, \\
    \text{MRP: }\ \hat \nabla_\pparams^\text{MRP} J(\pmixturepolicy) &= \nabla_\pparams \textstyle \sum_{\compidx=1}^N \pi^\weight_{\pparams^\weight}(\compidx) r\bigl(f_{\pparams^{\component_\compidx}}(\epsilon)\bigr),
\end{align*}
where $A\sim \pmixturepolicy(\cdot)$, $\epsilon\sim \cN(0,1)$, $\pi^\weight_{\pparams^\weight} (\compidx)=w_\compidx$, and $f_{\pparams^{\component_\compidx}}(\epsilon)=\mu_\compidx+\epsilon\sigma_\compidx$.
Define $\mixturepolicy\doteq\pmixturepolicy$, $\comppolicyatidx\doteq \rawpcomppolicyatidx$, 
$\hat \partial_{\theta_i^{\component_\compidx}} \cdot \doteq \big[ \hat \nabla_{\pparams^{\component_\compidx}} \cdot \big]_i$, and $\hat \partial_{\weight_\compidx} \cdot \doteq \big[ \hat \nabla_{\pparams^{\weight}} \cdot \big]_\compidx$, where $\theta_i^{\component_\compidx}$ and $\weight_\compidx$ are the $i$-th distribution parameter and the weight of the $k$-th component policy, respectively. For Gaussian component policies, $\theta_i^{\component_\compidx} \in \{ \mu_\compidx, \sigma_\compidx \}$. 
We present MRP's variance reduction properties in \Cref{prop:variance_reduction_main}.

\begin{assumption} \label{assm:variance_reduction_condition_full_main}
    $r: \bR \to \bR$ is twice differentiable with first and second derivatives $r'$ and $r''$. 
    For all $\compidx\in\{1,\cdots,N\}$ and $\varphi(a)\in\{ \phi_1(a)=(a-\mu_\compidx)r(a), \psi_1(a)=r'(a), \phi_2(a)=\bigl((a-\mu_\compidx)^2-\sigma_\compidx^2\bigr) r(a), \psi_2(a)=(a-\mu_\compidx)r'(a) \}$, $\varphi$ has finite variance and its first derivative $\varphi'$ is absolutely integrable under $\comppolicyatidx$: $\bV_{\comppolicyatidx(A)}\bigl(\varphi(A)\bigr)< \infty$ and $\bE_{\comppolicyatidx(A)}\bigl[ \left| \varphi'(A) \right| \bigr] < \infty$, with $\comppolicyatidx(a)=\cN(a;\mu_\compidx,\sigma_\compidx)$.
    Further, it holds for $i\in\{1,2\}$ that
    \begin{align*}
        \textstyle \sum_{k=1}^N \left(
        \bE_{\comppolicyatidx(A)} \bigl[ \phi_i'(A) \bigr]^2 
        - \sigma_\compidx^4 \bE_{\comppolicyatidx(A)} \left[ \psi_i'(A)^2 \right] 
        \right) &\ge 0.
    \end{align*}
\end{assumption}

\begin{assumption} \label{assm:when_importance_sampling_is_bad_main}
    Define $\rho^{\component_\compidx}_{\mixture}(A) =\frac{\comppolicyatidx(A)}{\mixturepolicy(A)}$, the sum of the variance of the importance-sampling LR estimators over all components is larger than that of the on-policy LR estimators: For $\theta_i^{\component_\compidx}\in\{\mu_\compidx,\sigma_\compidx\}$
    \begin{align}
        \textstyle \sum_{k=1}^N \Big(
        \bV_{\mixturepolicy(A)} \big( \rho^{\component_\compidx}_{\mixture}(A) \hat \partial_{\theta_i^{\component_\compidx}}^\text{LR} J(\rawpcomppolicyatidx) \big) \!-\!\bV_{\comppolicyatidx(A)} \big( \hat \partial_{\theta_i^{\component_\compidx}}^\text{LR} J(\rawpcomppolicyatidx) \big)
        \Big) \ge 0 \label{eq:variance_assm_gradient} \\
        \textstyle \sum_{k=1}^N \left(
        \bV_{\mixturepolicy(A)} \left( \rho^{\component_\compidx}_{\mixture}(A) r(A) \right) - 
        \bV_{\comppolicyatidx(A)} \bigl( r(A) \bigr)
        \right) \ge 0. \label{eq:variance_assm_reward}
    \end{align}
\end{assumption}

\begin{proposition} \label{prop:variance_reduction_main}
    Under Assumptions \ref{assm:variance_reduction_condition_full_main} and \ref{assm:when_importance_sampling_is_bad_main}, the trace of the covariance matrix of the MRP estimator is smaller than that of the LR estimator:
    \begin{align*}
        \Tr\left( \bC_{\cN(\epsilon;0,1)} \big( \hat \nabla_{\pparams}^\text{MRP} J(\pmixturepolicy) \big) \right) \!\le\! \Tr\left( \bC_{\mixturepolicy(A)} \big( \hat \nabla_{\pparams}^\text{LR} J(\pmixturepolicy) \big) \right).
    \end{align*}
\end{proposition}

\begin{remark}
    \Cref{prop:variance_reduction_main} builds on the marginal variance bound for the mean derivative estimator of Gaussians in \citet{gal2016uncertainty}. Here, we extend it to the standard deviation derivative estimator and Gaussian mixtures. Instead of deriving a bound for marginal variance as in 
    \citet{gal2016uncertainty}, we derive a bound for the traces of the estimators' covariance matrices for two reasons. First, the trace provides a single scalar value that summarizes the multivariate variability of the gradient, which is common in the literature \citep{miller2017reducing,xu2019variance}. Second, by directly targeting the trace, we can relax the assumptions so that the required conditions do not need to hold for each individual component. For example, in \Cref{assm:when_importance_sampling_is_bad_main}, it is possible that sampling from the mixture policy $\mixturepolicy$ might reduce variance for some of its components but the total variance across all components is still higher when some components induce dominating high-variance.
\end{remark}

\begin{remark}
    \Cref{prop:variance_reduction_main} is under the univariate setting, but can be extended to multivariate. For example, we can extend the analysis of to multivariate Gaussians in \citet{xu2019variance} by using \Cref{prop:mixture_component_variance_lr,prop:mixture_component_variance_mrp} in \Cref{sec:app_theory}.
\end{remark}

{
While \Cref{prop:variance_reduction_main} is derived for the bandit setting with univariate actions for simplicity, the practical effect of reduced gradient variance and improved stability holds empirically in high-dimensional RL tasks, as demonstrated in seven MuJoCo tasks in \Cref{sec:exp_lr_vs_rp}.
}

\begin{figure*}[tb]
  \centering
  \begin{minipage}{0.66\linewidth}
    \centering
    \resizebox{\textwidth}{!}{
    \includegraphics[height=2.5cm,trim=0 7 0 7,clip]{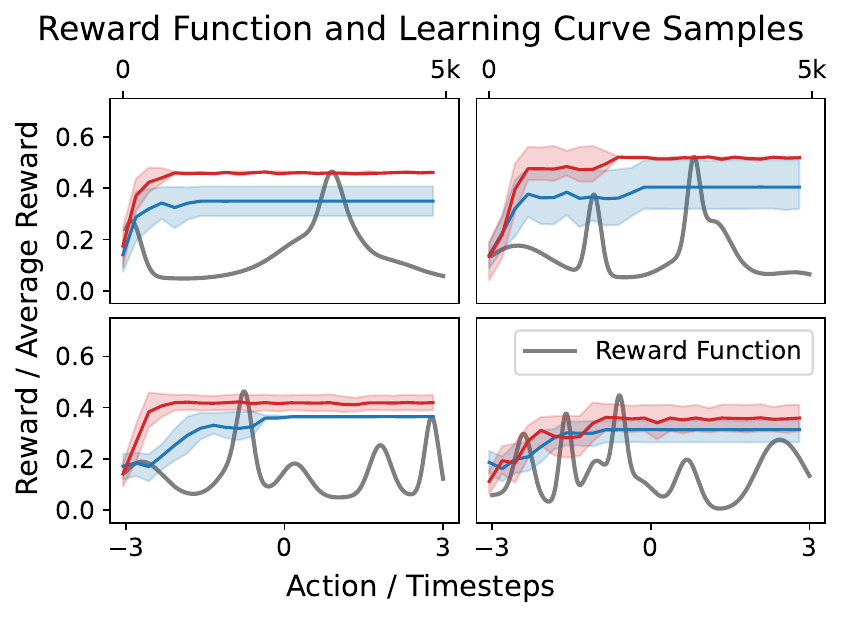}
    \hfill
    \includegraphics[height=2.5cm,trim=0 0 0 0,clip]{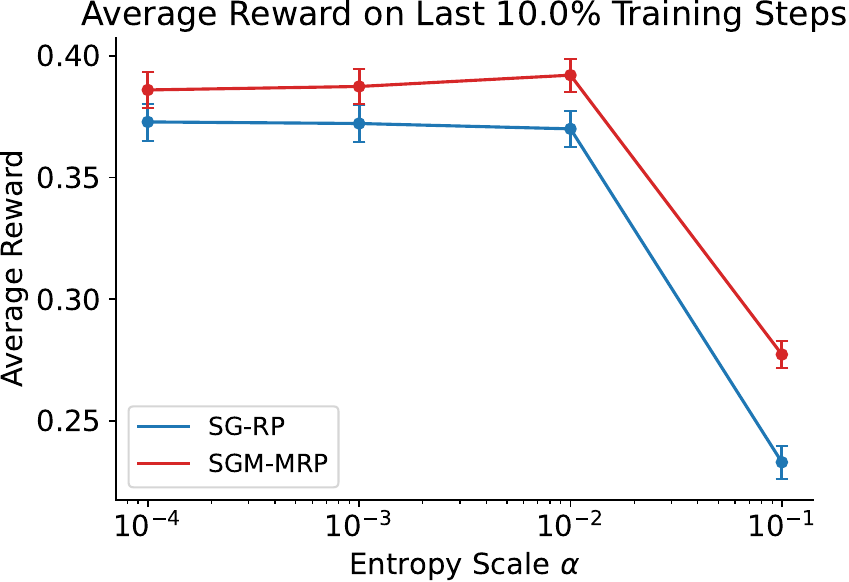}
    }
  \end{minipage}%
  \hfill
  \begin{minipage}{0.32\linewidth}
    \caption{\textbf{Left}: Examples of synthetic multimodal bandits and learning curves on them. 
    \textbf{Right}: Average performance across $100$ bandits at different entropy scales.
    All shaded areas and error bars show $95\%$ bootstrap confidence intervals (CIs).}
    \label{fig:multimodal_bandit_like}
  \end{minipage}
\end{figure*}

\begin{figure*}
    \centering
    \resizebox{\textwidth}{!}{
    \includegraphics[height=3cm,trim=0 5 0 5,clip]{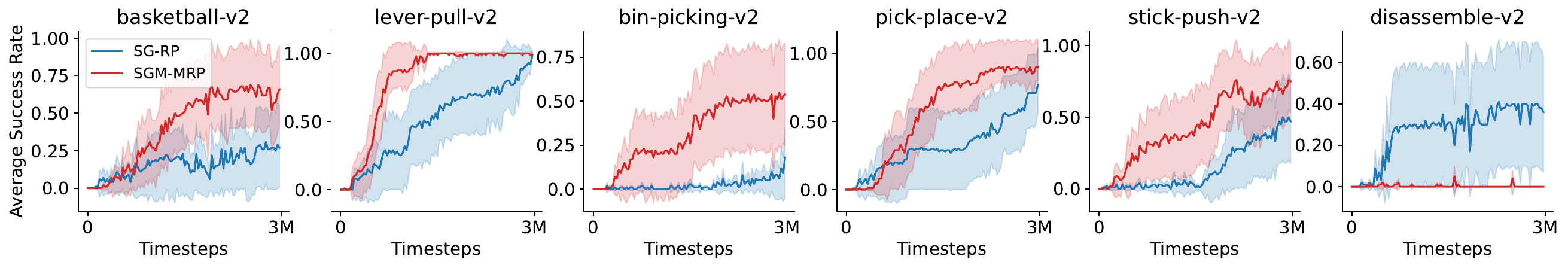}
    }
    \captionsetup{skip=6pt}
    \vspace{-0.5cm}
    \caption{Learning curves in selected MetaWorld tasks where the performance gap between mixture policies and base policies is most pronounced. 
    Shaded areas show 95\% bootstrap CIs over $10$ runs.
    }
    \label{fig:meta_world_highlight}
\end{figure*}

\section{Experiments on the Utility of Mixture Policies}
\label{sec:experiments}

We empirically compare mixture and base policies in this section, while investigating different gradient estimators for mixture policies in \Cref{sec:experiments_estimators}.
We use SAC as the base learning algorithm and the squashed Gaussian policy as the base policy
{(see \Cref{sec:app_sub_heavy_tailed_base_policy} for results using a heavy-tailed base policy)}.
We denote different SAC instances using X-Y, where X represents the policy's parameterization, and Y represents the gradient estimator. 
For squashed Gaussian (SG) policies, we consider the RP estimator (SG-RP). For squashed Gaussian mixture (SGM) policies, we consider the MRP estimator (SGM-MRP). We use $5$ components for mixture policies in all our experiments (see \Cref{sec:app_sub_num_of_components} for an ablation study). Unless otherwise noted, error bars and shaded regions indicate $95\%$ bootstrap confidence intervals.

\subsection{Do Mixture Policies Help Find Higher Peaks in the Critic?}
\label{sec:multimodal_results}

We first investigate the following question: \textit{When the critic landscape is multimodal, do mixture policies help find better peaks?} To avoid confounding factors, such as nonstationarity of the critic and the impact of the actor on the critic, we focus on the bandit setting, where the critic (i.e., the reward function) is stationary and given to the agent. Specifically, we create $100$ continuous bandits using density functions of randomly generated Gaussians. We sweep the initial actor step size and the entropy scale, running each setting on each bandit for $10$ runs. We report results of the best setting that has the highest average reward over the last $10.0\%$ training steps, as we are mostly interested in how well the agent explores. Figure \ref{fig:multimodal_bandit_like} (Left) shows a few examples of the generated bandits and the learning curves corresponding to them. More details can be found in \Cref{sec:app_sub_experimental_details_app}.

\textbf{Mixture policies explore more efficiently.} Figure \ref{fig:multimodal_bandit_like} (Right) shows the aggregated final performance of base policies (SG-RP) and mixture policies (SGM-MRP). We can see that mixture policies are better than base policies across different $\alpha$. From Figure \ref{fig:multimodal_bandit_like} (Left), we see that mixture policies can find higher peaks in the critic compared with base policies.

\textbf{Mixture policies also improve robustness.} Another observation from Figure \ref{fig:multimodal_bandit_like} (Right) is that the mixture policy is more robust to the entropy scale, which is consistent with our results presented in Section \ref{sec:motivation_in_bandit}. Specifically, other than the consistent improvement of the mixture policy over the base policy, we can see that the gap between them increase as $\alpha$ increases. This result suggests that mixture policies are possibly more robust to larger entropy scales and explore more efficiently with a moderate large entropy scale. This is not the case for the base policy in this experiment.

\subsection{Do Mixture Policies Help in Continuous Control Benchmarks?}
\label{sec:exp_continuous_control_all}

\begin{wrapfigure}{r}{0.49\textwidth}
    \centering
    \resizebox{.49\textwidth}{!}{
    \includegraphics[width=\textwidth,trim=0 5 0 0,clip]{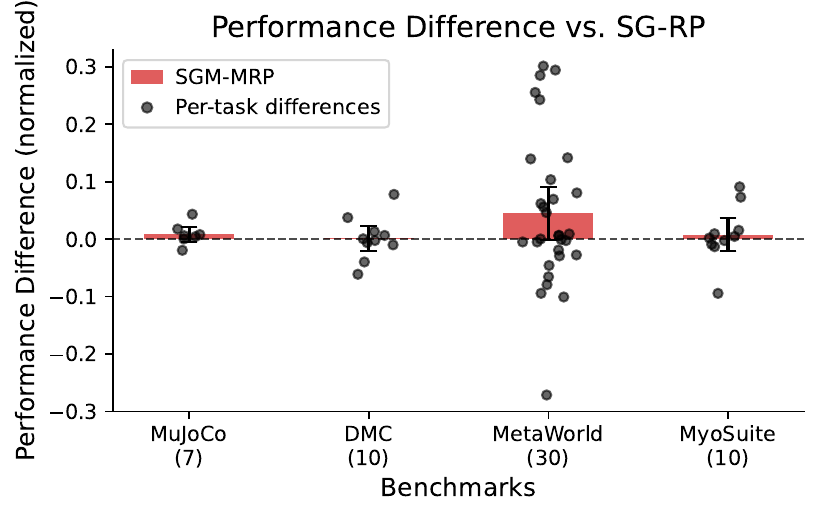}
    }
    \caption{
    Performance comparison of SGM-MRP vs. SG-RP across common benchmarks. Positive values indicates SGM-MRP outperforms SG-RP. Bars represent mean performance difference across tasks within each benchmark, with error bars denoting $95\%$ bootstrap CIs. Individual points represent performance delta for each specific task.
    }
    \label{fig:mujoco_dmc_bar_plots}
\end{wrapfigure}
We next move on to the full RL setting: \textit{Do mixture policies improve performance in common continuous control benchmarks?} Specifically, we investigate the performance of mixture policies in $57$ environments from four common continuous control benchmarks: $7$ from OpenAI Gym MuJoCo \citep{brockman2016openai}, $10$ from DeepMind Control (DMC) Suite \citep{tassa2018deepmind}, $30$ from MetaWorld \citep{yu2020meta}, and $10$ from MyoSuite \citep{caggiano2022myosuite}. 
Performance is measured using the average (normalized) return in Gym MuJoCo and DM Control or success rate in MetaWorld and MyoSuite.
We use SAC with \textit{automatic entropy tuning} and the hyperparameters reported in the second SAC paper \citep{haarnoja2018softapp}, which are tuned based on SG-RP. 
We use $10$ random seeds for each environment.
Please refer to \Cref{sec:app_sub_experimental_details_app} for more details.

\begin{figure*}[t]
    \centering
    \resizebox{\textwidth}{!}{
    \includegraphics[height=3cm,trim=0 5 0 5,clip]{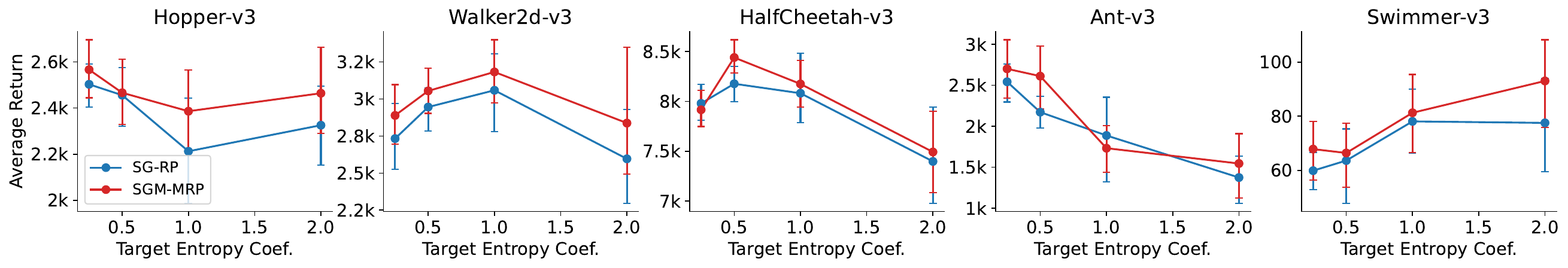}
    }
    \captionsetup{skip=2pt}
    \caption{Performance with different target entropy coefficients for mixture policies and base policies in five MuJoCo environments. Results are averaged over $10$ runs.
    }
    \label{fig:sensitivity_mujoco}
\end{figure*}

\begin{figure*}
    \centering
    \resizebox{\textwidth}{!}{
    \includegraphics[height=3cm,trim=0 5 0 5,clip]{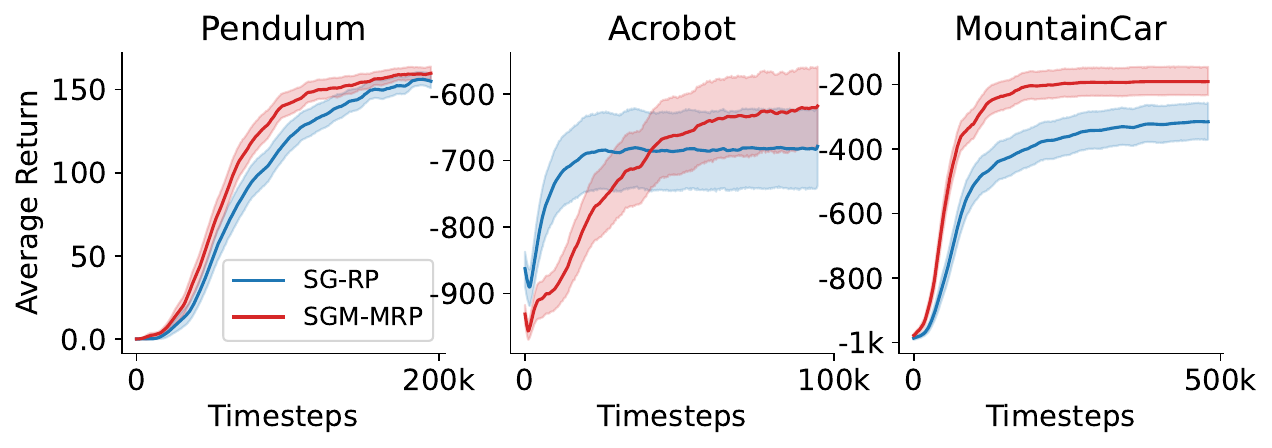}
    \includegraphics[height=3cm,trim=0 5 0 5,clip]{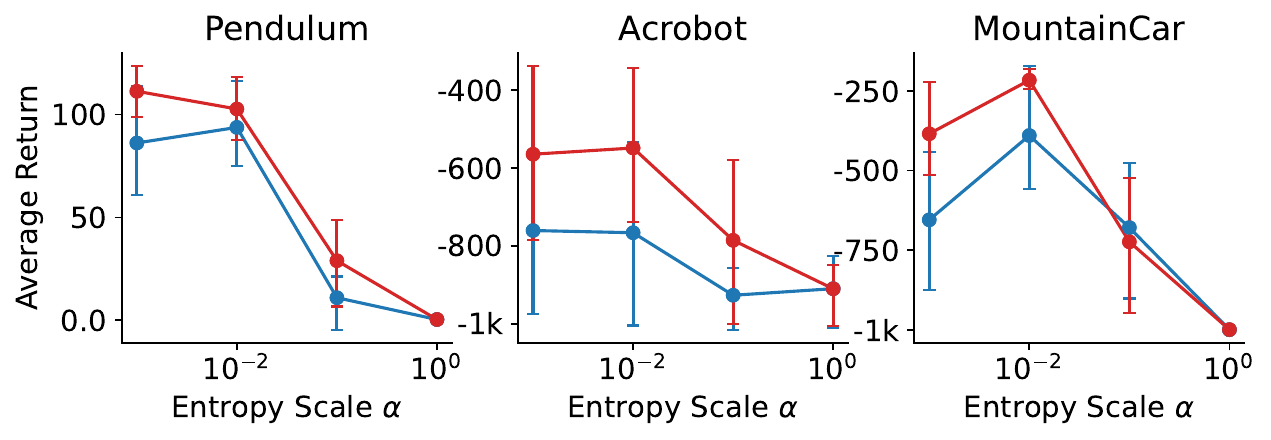}
    }
    \captionsetup{skip=4pt}
    \vspace{-0.4cm}
    \caption{Learning and sensitivity curves for classic control environments with unshaped rewards. Learning curves are averaged over $200$ runs, whereas sensitivity curves are over $10$ runs.}
    \label{fig:classic_learning_curves}
    \vspace{-0.3cm}
\end{figure*}

\textbf{Mixture policies consistently match or surpass base policies without hyperparameter tuning.}
\Cref{fig:mujoco_dmc_bar_plots} reports the performance gap between SGM-MRP and SG-RP across four benchmarks. Overall, mixture policies perform comparably to base policies and often achieve slightly higher average performance. The gains are pronounced in MetaWorld, where mixture policies more often significantly outperform base policies. \Cref{fig:meta_world_highlight} shows the learning curves in MetaWorld environments where the performance gap is larger than $20\%$. Learning curves for all tasks are reported in \Cref{fig:mujoco_dmc_learning_curves,fig:mujoco_dmc_learning_curves_additional} (in the appendix).

We also evaluate additional target-entropy values in MuJoCo environments, and the results consistently show that mixture policies outperform base policies (see \Cref{fig:sensitivity_mujoco,fig:sensitivity_mujoco_two}).

\subsection{Do Mixture Policies Help With Unshaped Rewards?}
\label{sec:exp_shaped_versus_unshaped}

While mixture policies do not appear to be generally beneficial on common benchmarks, we want to understand when mixture policies may be more helpful. 
Since most tasks in those benchmarks have shaped rewards,
we hypothesize that \textit{mixture policies are more helpful in environments with unshaped rewards compared with those with shaped rewards}. We next test this hypothesis.

To perform a more extensive empirical investigation with proper hyperparameter tuning, we use three classic control environments as they require less computation resource: Pendulum \citep{degris2012model}, Acrobot \citep{sutton2018reinforcement}, and MountainCar \citep{sutton2018reinforcement}. Specifically, we use two different variants for each of them: one with shaped rewards 
and the other one with unshaped rewards.
We refer the reader to Section \ref{sec:app_classic_control_rewards} about their versions and specific reward functions. We use SAC with a fixed entropy scale as it is reported that SAC with automatic entropy performs worse in this domain \citep{neumann2022greedy}. We sweep the entropy scale, the initial critic step size, and the initial actor step size, running each setting for $10$ runs. We report another $30$ reruns for environments with shaped rewards and $200$ reruns for those with unshaped rewards of the best setting that has the largest area under the learning curve. See \Cref{sec:app_sub_experimental_details_app} for more details.

\begin{figure*}
    \centering
    \includegraphics[width=\linewidth]{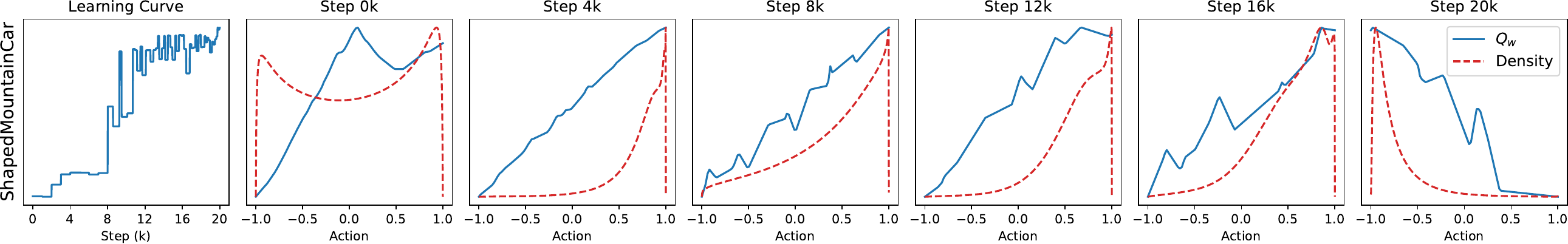}
    \includegraphics[width=\linewidth]{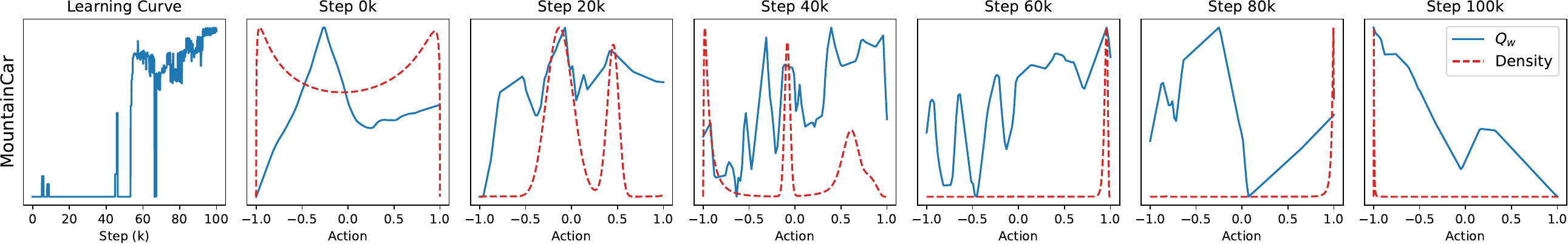}
    \caption{Learning curves, action-value estimates, and policy densities at a starting state from a sample run of SGM-MRP in two MountainCar variants. 
    The y-axes differ across plots and are not shown with ticks to highlight the shape of the curves rather than the values. 
    The learning curves show faster convergence and smoother performance in \texttt{ShapedMountainCar} with shaped rewards, whereas unshaped rewards in \texttt{MountainCar} lead to more erratic and slower learning. The action-value estimates are smoother and more stable in \texttt{ShapedMountainCar}. In contrast, they are more unstable and multimodal in \texttt{MountainCar}. Correspondingly, the density quickly becomes unimodal and concentrates on one of the boundaries in the former, while multimodal density has more occurrence in the latter, reflecting continued exploration of the mixture policy.}
    \label{fig:vis_critic_mc}
\end{figure*}

\textbf{Mixture policies improve exploration when rewards are unshaped.}
\Cref{fig:classic_learning_curves} shows results in environments with unshaped rewards, where mixture policies substantially outperform base policies and consistently achieve higher returns across entropy scales. In contrast, performance in shaped-reward environments is more stable, and the gap between the two is much smaller (see \Cref{fig:classic_learning_curves_shaped}).

\textbf{Visualization of action-value estimates and policy density.}
To understand why mixture policies deviate more from base policies in unshaped-reward settings, we examine an SGM-MRP agent in MountainCar and ShapedMountainCar (\Cref{fig:vis_critic_mc}). In both environments, the agent must decide which direction to accelerate to gain momentum (\Cref{fig:mc_starting_state}). In MountainCar, where rewards are unshaped, the critic is less smooth and mixture policies exhibit greater multimodality, preserving multiple action modes that support bidirectional exploration. In contrast, shaped rewards in ShapedMountainCar yield smoother critics and policies with fewer modes, reducing the need for exploration. This illustrates why mixture policies, which enhance exploration (\Cref{sec:multimodal_results}), show limited benefit in common benchmark environments with shaped rewards (\Cref{sec:exp_continuous_control_all}).

\section{Experiments on Gradient Estimators for Mixture Policies}
\label{sec:experiments_estimators}

\begin{wrapfigure}{r}{0.49\textwidth}
    \centering
    \resizebox{0.49\textwidth}{!}{
    \includegraphics[height=2.5cm,trim=0 5 0 0,clip]{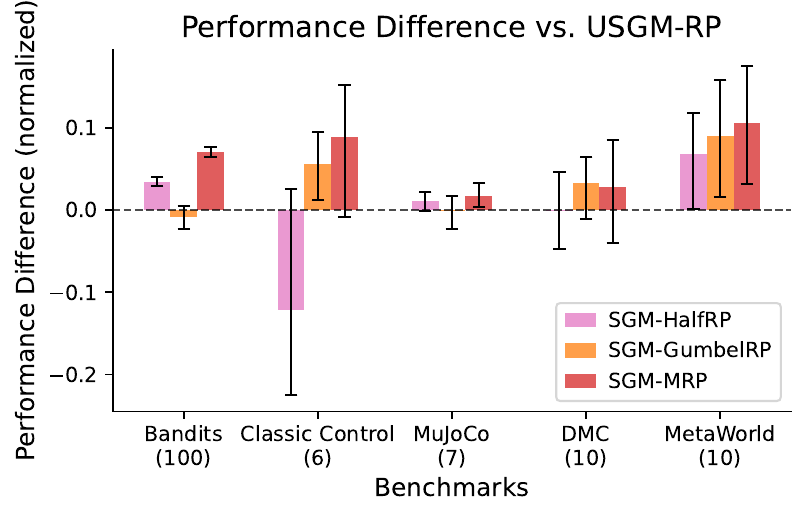}
    }
    \caption{
    Performance comparison of different estimators for mixture policies across common benchmarks. Positive values indicate that the estimator outperforms the USGM-RP baseline. Bars represent mean performance difference across tasks within each benchmark, with error bars denoting $95\%$ bootstrap CIs.
    }
    \label{fig:summarized_perf_mixture}
\end{wrapfigure}
In this section, we compare our proposed MRP estimator for mixture policies against a prior approach that fixes the mixing weights to avoid reparameterizing the weighting policy \citep{baram2021maximum}. Following \citeauthor{baram2021maximum}, we use uniform mixing weights and denote this baseline as USGM-RP. In addition to the MRP estimator, we also evaluate the HalfRP estimator (\Cref{eq:gradient_half_reparam}) and the GumbelRP estimator. The GumbelRP estimator uses the Gumbel-Softmax trick \citep{jang2016categorical} to obtain biased reparameterization samples for the discrete weighting policy. Please refer to \Cref{sec:app_gumbel_reparam_estimator} for more details.
We compare them in synthetic multimodal bandits, classic control, and three common continuous control benchmarks.
For each setting, we follow the same experimental setup as 
\Cref{sec:experiments}. \Cref{fig:summarized_perf_mixture} shows the summarized results. Learning curves and more plots can be found in \Cref{sec:app_experimet_details}.

\textbf{The MRP estimator is the most reliable overall.}
Across all settings, SGM-MRP consistently performs best, in some cases significantly outperforming both the baseline and other estimators. SGM-GumbelRP is competitive in classic control and continuous control benchmarks but performs worst in multimodal bandits, likely because the bias in its gradient is more pronounced in this simpler setting. Conversely, the unbiased SGM-HalfRP does well in multimodal bandits but fails badly in some classic control environments, suggesting that its high variance hinders stable learning.

\section{Conclusions}

Mixture policies are a simple way to increase the flexibility of the policy parameterization, but little has been documented about their efficacy, or lack of efficacy. Our aim was to start to fill this gap, to make this a more accessible tool when using entropy-regularized actor-critic algorithms with continuous actions, like Soft-Actor Critic (SAC). The clear outcome from the study is that mixture policies are comparable, and sometimes better than, a base unimodal policy. Through a few basic theoretical results and experiments in bandits, we highlighted that mixture policies are more robust to entropy scale, with 1) a preference for multimodality increasing with higher entropy, 2) less divergence (lack of stationary points) compared to the base Gaussian policy, 3) better balance between entropy and the expected reward objective, resulting in higher unregularized values in addition to higher regularized values, and 4) higher likelihood of finding maxima on a multimodal surface. This behavior seemed to manifest in better exploration in environments with unshaped, or uninformative rewards; in such environments, without shaped rewards, exploration is critical and the mixture policies performed better than the unimodal base policy. In particular, we found the base policy had more failed runs where it was unable to find the goal at all. 

To leverage the utility of mixture policies, however, we needed a small algorithmic improvement: reparameterization gradients. We proposed a new reparameterization gradient estimator for mixture policies that reduces variance effectively, filling in a gap in the literature. 
We proved that this estimator enjoys lower variance than the standard LR estimator and showed that it perform better than alternatives.

{
Finally, while mixture policies provide robust performance across domains, they introduce a modest additional computational cost and require choosing the number of mixture components. We provide a concrete analysis of the computational overhead in \Cref{sec:app_sub_computation_overhead} and discuss selection of the number of components in \Cref{sec:app_sub_num_of_components}.
}

\subsubsection*{Acknowledgments}

This research is supported by the Canada CIFAR AI Chairs program, the Reinforcement Learning and Artificial Intelligence (RLAI) laboratory, the Alberta Machine Intelligence Institute (Amii), the Natural Sciences and Engineering Research Council (NSERC) of Canada, and the Digital Research Alliance of Canada. The authors would like to thank Shivani Singal for feedback on an earlier version of this paper, as well as Qingfeng Lan, Rupam Mahmood, Haseeb Shah, and Pranaya Jajoo for helpful discussions.

\bibliography{main}
\bibliographystyle{tmlr}

\appendix

\clearpage

\section{Expanded Related Works}
\label{sec:app_related_works}

\paragraph{Implicit policies.} Existing works on modeling continuous distributions on continuous action spaces can be put into two categories: \emph{explicit policies} that directly output parametric distributions and \emph{implicit policies}. We have discussed the former and touched on the latter in the introduction of our paper. Here, we further discuss implicit policies. Implicit policies utilize deep generative models (e.g., energy-based models; \citealp{haarnoja2017reinforcement,messaoud2024sac}; normalizing flows; \citealp{tang2018implicit,mazoure2020leveraging}; diffusion models; \citealp{wang2023diffusion}) to model the policy. These models can model complex distributions and have improved learning efficiency, but they usually have more parameters and complex training pipelines. Compared with these more complex policies, explicit policies have several benefits. Firstly, they are simpler and more efficient to train.
Secondly, they have simple explicit probability density, which is useful in various ways in entropy-regularized RL. Note that some implicit policies do not hold such a property.
Further, mixture policies, which we consider in our paper, also have modeling power to model arbitrary distribution given a sufficiently large number of components.
Given these benefits, we think it is important to improve our understanding of mixture policies. Nevertheless, we agree that it is also important to investigate more complex but powerful implicit policies, and the benefits of mixture policies demonstrated in our paper can potentially be generalized to them.

\paragraph{Latent variable policies.} Between explicit policies and implicit policies, latent variable policies \citep{hafner2019learning,lee2020stochastic,zhang2023latent} attempt to increase the expressivity of the former by conditioning explicit policies on continuous latent variables, which are essential to the latter. While many implicit policies technically fall under this umbrella, we distinguish our scope by focusing on latent variables used for agent-state modeling rather than those used solely to parameterize complex output distributions. 
The most closely related work is \citet{zhang2023latent}, which introduces sampling-based estimators for policy entropy via marginalization of the latent variable.
While our approach also involves marginalizing a hidden component--specifically the mixing variable--the objectives differ fundamentally: whereas \citet{zhang2023latent} seek to overcome the computational intractability of marginalization in latent variable policies, we investigate the theoretical and empirical advantages of marginalization within the specific context of mixture policies.

\paragraph{Other uses of mixture policies.} Since mixture distributions are a widely known model, mixture policies have been investigated in various ways in the literature. \citet{daniel2012hierarchical} and \citet{celik2022specializing} focus on exploiting the hierarchy in mixture policies for problems with hierarchical structures. In their work, they design algorithms specific to mixture policies. Sharing the same motivation to use mixture policies to model diverse behaviors, \citet{nematollahi2022robot} learn a prior GMM using imitation learning and then use SAC to learn the changes to the prior GMM for adaptation. Here, the action space of SAC is the changes to GMM’s parameters. With a different motivation, \citet{seyde2022strength} use mixture policies to select from a diverse set of sub-policies to reduce the hyperparameter sensitivity of the algorithm. 

\paragraph{Mixture policies as the policy parameterization for SAC.} Different from the above previous works, our motivation for using mixture policies is to treat them as a more complex policy class and understand the effect of doing so under the entropy-regularization setting. Thus, our treatment does not include designing specific objective functions for mixture policies but using the standard regularized objective that is agnostic to policy parameterizations. In this regard, the closest related works are \citet{haarnoja2018softarxiv}, \citet{hou2020off}, \citet{baram2021maximum}, and \citet{ren2021probabilistic}. In the first version of SAC, \citeauthor{haarnoja2018softarxiv} did indeed test mixture policies but then did not pursue this further nor provide insights on this choice. We hypothesize the lack of reparameterization (RP) gradient estimators for mixture policies might be the reason that later versions of SAC switched to a single Gaussian as they found the RP estimator works better (see Footnote 3 on Page 67 of \citealp{haarnoja2018acquiring}). Later, \citeauthor{hou2020off} and \citeauthor{ren2021probabilistic} try to avoid reparameterization of the whole mixture policy in SAC by using a separate objective for the weighting policy. 
Further, \citeauthor{baram2021maximum} also revisit SAC with a mixture policy. However, their focus is to explore the utility of the upper and lower bounds of mixture models' entropy, without considering a learnable weighting policy.
We also note that the empirical analysis in these aforementioned work is on a restrictive set of MuJoCo environments and does not provide statistical significant evaluation results. 

It is not highly novel to use mixture policies, but rather to understand the effect of doing so. In our work: 1) We provide new insights into the effect of a more flexible policy class on the stationary points of the entropy regularized objective; 2) we propose a new RP gradient estimator for mixture policies with provable variance reduction properties and without compromising flexibility or altering the algorithm, filling in a gap in the literature; and 3) we explore the benefits of and provide insights into using mixture policies in environments without shaped rewards rigorously, complementing existing works on using mixture policies in entropy-regularized actor-critic.

\section{The Gumbel Reparameterization (GumbelRP) Estimator}
\label{sec:app_gumbel_reparam_estimator}

Since the output of the weighting policy $\pweightpolicy(\cdot|S_t)$ is a discrete random variable $K_t$, we can not directly reparameterize it. 
In \Cref{sec:half_rp_estimator}, we avoid using the LR estimator for the weighting policy by marginalizing the discrete random variable $K_t$. 
Here, we introduce an alternative to the MRP estimator through \textit{biased} reparameterizations of discrete random variables \citep{bengio2013estimating,maddison2016concrete,jang2016categorical}. Here, we investigate the straight-through Gumbel-Softmax reparameterization \citep{jang2016categorical}.

Given a weighting distribution $\pweightpolicy(\cdot|s)$ and i.i.d samples from Gumbel$(0,1)$, $\gumbelrndvar_1,\cdots,\gumbelrndvar_N$, we can obtain a sample from the corresponding Gumbel-Softmax distribution:
\begin{align*}
    y_\pparams(\vect\gumbelrndvar;s,k)
    = \frac{\exp({(\log{\pweightpolicy(k|s)} + \gumbelrndvar_k)/\tau})}{\sum_{k'=1}^N\exp({(\log{\pweightpolicy(k'|s)} + \gumbelrndvar_{k'})/\tau})} \quad \text{for } k=1, \dots, N,
\end{align*}
where we define $\vect\gumbelrndvar=[\gumbelrndvar_1, \cdots, \gumbelrndvar_N]$, and $\tau$ is a temperature parameter, controlling a bias-variance trade-off. 
Using the Gumbel-Max trick, we can obtain a sample from $\pweightpolicy(\cdot|s)$ using $y_\pparams(\vect\gumbelrndvar;s,k)$:
\begin{align*}
    \hat{\vect z} = \onehot \big(\arg\max_k \big(y_\pparams(\vect\gumbelrndvar;s,k)\big)\big) 
    =\onehot \big(\arg\max_k \big(\log{\pweightpolicy(k|s)} + \gumbelrndvar_k\big)\big).
\end{align*}
We can further use the straight-through trick to obtain a differentiable one-hot sample $\vect z = [z_\pparams(\vect\gumbelrndvar;s,1), \cdots, z_\pparams(\vect\gumbelrndvar;s,N)]$, where $z_\pparams(\vect\gumbelrndvar;s,k)$ is defined as follows:
\begin{align*}
    z_\pparams(\vect\gumbelrndvar;s,k) = \hat z_k + y_\pparams(\vect\gumbelrndvar;s,k) - y_{\vect\phi}(\vect\gumbelrndvar;s,k) |_{\vect\phi=\pparams}
    \quad \text{for } k=1, \dots, N.
\end{align*}
Finally, using the differential one-hot sample $\vect z$ from the weighting policy $\pweightpolicy$ and reparameterized samples from the component policies, we can obtain a differentiable action sample $a$:
\begin{align}
    \label{eq:gumbel_reparameterized_action}
    a= \textstyle \sum_{k=1}^N z_\pparams(\vect\gumbelrndvar;s,k) f_\pparams(\epsilon;s,k).
\end{align}
Plugging (\ref{eq:gumbel_reparameterized_action}) back to (\ref{eq:gradient_reparam}), we can obtain a full RP estimator, which we call the \textit{Gumbel-reparameterization} (GumbelRP) estimator:
\begin{align}
    \!\!\hat \nabla_\pparams \hat J(\pmixturepolicy)
    \!=\! \nabla_\pparams \big(Q_\qparams(S_t, A_t) \!-\! \alpha \log\pmixturepolicy(A_t|S_t)\big) \text{ with } A_t\!=\! \textstyle \sum_{k=1}^N \! z_\pparams(\vect\gumbelrndvar_t;S_t,k) f_\pparams(\epsilon_t;S_t,k). \!
    \label{eq:gradient_gumbel_reparam} 
\end{align}

\textbf{The temperature parameter $\tau$ controls a bias-variance trade-off.} When $\tau$ approaches $0$, the soft sample $\vect y=[y_\pparams(\vect\gumbelrndvar;s,1), \cdots, y_\pparams(\vect\gumbelrndvar;s,N)]$ will converge to a one-hot vector and recover the categorical sample. However, the variance of the gradients with respect to $\pweightpolicy(\cdot|s)$ will increase. Conversely, when $\tau$ becomes larger, the variance of the gradients will decrease, but the soft sample $\vect y$ will converge to a uniform vector. See \citet{jang2016categorical} for detailed discussions. In our study, we find that using a fixed temperature $\tau=1$ works well (see Section \ref{sec:app_gumbel_temperature} for a sensitivity study).

In \Cref{tab:estimators}, we compare different estimators in terms of their flexibility and bias-variance trade-off.

\begin{table}[htb]
  \caption{Comparison of gradient estimators for mixture policies. 
  Checkmarks indicate whether the method supports learnable mixture weights, 
  has low variance, and is unbiased. HalfRP has intermediate variance between the LR and MRP estimators. U-RP denotes mixture policies with uniform weightings and the RP estimator.
  }
  \label{tab:estimators}
  \centering
  \begin{tabular}{l|ccc}
    \toprule
    Estimator & Learnable weights & Low variance & Unbiased \\
    \midrule
    U-RP \citep{baram2021maximum} & \xmark & \cmark & \cmark \\
    LR \citep{williams1992simple} & \cmark & \xmark & \cmark \\
    \midrule
    HalfRP (\Cref{eq:gradient_half_reparam}) & \cmark & \pmark & \cmark \\
    GumbelRP (\Cref{eq:gradient_gumbel_reparam}) & \cmark & \cmark & \xmark \\
    MRP (\Cref{eq:gradient_expected_reparam}) & \cmark & \cmark & \cmark \\
    \bottomrule
  \end{tabular}
\end{table}

\newpage

\section{Proofs}
\label{sec:app_theory}

\subsection{Proofs for Results in Section \ref{sec:motivation_in_bandit}}

\begin{lemma} \label{thrm:stationary_point_relation}
For arbitrary $\pparams^\weight$ and any $\tilde{\pparams}^b$ such that $\nabla_{\pparams^\component} J(\pi^\component_{\tilde{\pparams}^\component})=0$, we have $\nabla_{\pparams^\mixture} J(\pi^\mixture_{\tilde\pparams^\mixture})=0$, where $\tilde\pparams^\mixture=\left[ \smash{{\tilde\pparams}^\component}^\top, \cdots, \smash{{\tilde\pparams}^\component}^\top, \smash{\pparams^\weight}^\top\right]^\top$.
\end{lemma}
\begin{proof}
From (3) of \citet{ahmed2019understanding}, we can obtain the policy gradient under the regularized objective:
\begin{align*}
    \nabla_{\pparams^\component} J(\pi^\component_{\tilde{\pparams}^\component})
    &= \int_s d_{\pi^\component_{\tilde\pparams^\component}}(s) \int_a \big( Q_{\pi^\component_{\tilde\pparams^\component}}(s,a) - \alpha \log\pi^\component_{\tilde{\pparams}^\component}(a|s)  \big) \nabla_{\pparams^\component} \pi^\component_{\tilde{\pparams}^\component}(a|s) \,da\,ds \\
    &= 0. \tag{by assumption}
\end{align*}

For any $\pparams^\weight$, to show $\nabla_{\pparams^\mixture} J(\pi^\mixture_{\tilde\pparams^\mixture})=0$, we can show $\nabla_{\pparams^\component_k} J(\pi^\mixture_{\tilde\pparams^\mixture})=0$ and $\nabla_{\pparams^\weight} J(\pi^\mixture_{\tilde\pparams^\mixture})=0$.

We first derive the gradient of $\pi^\mixture_{\pparams^\mixture}$ with respect to $\pparams^\component_k$ and $\pparams^\weight$:
\begin{align*}
    \nabla_{\pparams^\component_k} \pi^\mixture_{\pparams^\mixture}(a|s) &= \nabla_{\pparams^\component_k} \sum_{k=1}^N \pi^\weight_{\pparams^\weight}(k) \pi^\component_{\pparams^\component_k}(a|s)
    = \pi^\weight_{\pparams^\weight}(k) \nabla_{\pparams^\component_k}  \pi^\component_{\pparams^\component_k}(a|s), \\
    \nabla_{\pparams^\weight} \pi^\mixture_{\pparams^\mixture}(a|s) &= \nabla_{\pparams^\weight} \sum_{k=1}^N \pi^\weight_{\pparams^\weight}(k) \pi^\component_{\pparams^\component_k}(a|s)
    = \sum_{k=1}^N \nabla_{\pparams^\weight} \pi^\weight_{\pparams^\weight}(k) \pi^\component_{\pparams^\component_k}(a|s).
\end{align*}

Then, we can derive the gradient of $J(\pi^\mixture_{\pparams^\mixture})$ with respect to $\pparams^\component_k$:
\begin{align*}
    \nabla_{\pparams^\component_k} J(\pi^\mixture_{\pparams^\mixture})
    &= \int_s d_{\pi^\mixture_{\pparams^\mixture}}(s) \int_a \big( Q_{\pi^\mixture_{\pparams^\mixture}}(s,a)  - \log\pi^\mixture_{\pparams^\mixture}(a|s)  \big) \nabla_{\pparams^\component_k} \pi^\mixture_{\pparams^\mixture}(a|s) \,da\,ds \\
    &= \pi^\weight_{\pparams^\weight}(k) \int_s d_{\pi^\mixture_{\pparams^\mixture}}(s) \int_a \big( Q_{\pi^\mixture_{\pparams^\mixture}}(s,a) - \alpha \log\pi^\mixture_{\pparams^\mixture}(a|s)  \big) \nabla_{\pparams^\component_k}  \pi^\component_{\pparams^\component_k}(a|s) \,da\,ds.
\end{align*}
Plugging $\tilde\pparams^\mixture=\left[ \smash{{\tilde\pparams}^\component}^\top, \cdots, \smash{{\tilde\pparams}^\component}^\top, \smash{\pparams^\weight}^\top\right]^\top$ and $\pi^\mixture_{\tilde\pparams^\mixture}(\cdot|s) = \pi^\component_{\tilde\pparams^\component}(\cdot|s)$ in, we have
\begin{align*}
    \nabla_{\pparams^\component_k}J( \pi^\mixture_{\tilde\pparams^\mixture} ) = \pi^\weight_{\pparams^\weight}(k) \int_s d_{\pi^\component_{\tilde\pparams^\component}}(s) \int_a \big( Q_{\pi^\component_{\tilde\pparams^\component}}(s,a) - \alpha \log\pi^\component_{\tilde\pparams^\component}(a|s)  \big) \nabla_{\pparams^\component_k}  \pi^\component_{\tilde\pparams^\component}(a|s) \,da\,ds = 0.
\end{align*}

Next, we derive the gradient of $J(\pi^\mixture_{\pparams^\mixture})$ with respect to $\pparams^\weight$:
\begin{align*}
    \nabla_{\pparams^\weight} J(\pi^\mixture_{\pparams^\mixture})
    &= \int_s d_{\pi^\mixture_{\pparams^\mixture}}(s) \int_a \big( Q_{\pi^\mixture_{\pparams^\mixture}}(s,a) - \alpha \log\pi^\mixture_{\pparams^\mixture}(a|s)  \big) \nabla_{\pparams^\weight} \pi^\mixture_{\pparams^\mixture}(a|s) \,da\,ds \\
    &= \int_s d_{\pi^\mixture_{\pparams^\mixture}}(s) \int_a \big( Q_{\pi^\mixture_{\pparams^\mixture}}(s,a) - \alpha \log\pi^\mixture_{\pparams^\mixture}(a|s)  \big) \sum_{k=1}^N \nabla_{\pparams^\weight} \pi^\weight_{\pparams^\weight}(k) \pi^\component_{\pparams^\component_k}(a|s) \,da\,ds.
\end{align*}
Again, plugging $\tilde\pparams^\mixture=\left[ \smash{{\tilde\pparams}^\component}^\top, \cdots, \smash{{\tilde\pparams}^\component}^\top, \smash{\pparams^\weight}^\top\right]^\top$ and $\pi^\mixture_{\tilde\pparams^\mixture}(\cdot|s) = \pi^\component_{\tilde\pparams^\component}(\cdot|s)$ in, we have
\begin{align*}
    \nabla_{\pparams^\weight} J( \pi^\mixture_{\tilde\pparams^\mixture} ) &= \int_s d_{\pi^\component_{\tilde\pparams^\component}}(s) \int_a \big( Q_{\pi^\component_{\tilde\pparams^\component}}(s,a) - \alpha \log\pi^\component_{\tilde\pparams^\component}(a|s)  \big) \sum_{k=1}^N \nabla_{\pparams^\weight} \pi^\weight_{\pparams^\weight}(k) \pi^\component_{\tilde\pparams^\component}(a|s) \,da\,ds \\
    &= \int_s d_{\pi^\component_{\tilde\pparams^\component}}(s) \int_a \big( Q_{\pi^\component_{\tilde\pparams^\component}}(s,a) - \alpha \log\pi^\component_{\tilde\pparams^\component}(a|s)  \big) \pi^\component_{\tilde\pparams^\component}(a|s) \,da\,ds \nabla_{\pparams^\weight} \sum_{k=1}^N \pi^\weight_{\pparams^\weight}(k) \\
    &= \int_s d_{\pi^\component_{\tilde\pparams^\component}}(s) \int_a \big( Q_{\pi^\component_{\tilde\pparams^\component}}(s,a) - \alpha \log\pi^\component_{\tilde\pparams^\component}(a|s)  \big) \pi^\component_{\tilde\pparams^\component}(a|s) \,da\,ds \nabla_{\pparams^\weight} 1 \\
    &= 0.
\end{align*}
Thus, $\nabla_{\tilde\pparams}J( \pi^\mixture_{\tilde\pparams^\mixture} )=[\nabla_{\pparams^\component_1} J( \pi^\mixture_{\tilde\pparams^\mixture} )^\top, \cdots, \nabla_{\pparams^\component_N} J( \pi^\mixture_{\tilde\pparams^\mixture} )^\top, \nabla_{\pparams^\weight} J( \pi^\mixture_{\tilde\pparams^\mixture} )^\top]^\top=\vect 0$.
\end{proof}

\textbf{Proposition
\ref{thrm:entropy_regularized_optimality}.}
\textit{
When both $\pi^b_{\pparams^{b,*}}$ and $\pi^m_{\pparams^{m,*}}$ exist, then $J(\pi^m_{\pparams^{m,*}}) \ge J(\pi^b_{\pparams^{b,*}})$.
}
\begin{proof}
By \Cref{thrm:stationary_point_relation}, every stationary point of $J$ under the component parameterization $\pparams^b$ corresponds to a stationary point under the mixture parameterization $\pparams^m$. That is, the set of stationary points for $\pi^m_{\pparams^m}$ is a superset of that for $\pi^b_{\pparams^b}$. Since the global optimum is attained at a stationary point, optimizing over the larger set cannot yield a lower value, so $J(\pi^m_{\pparams^{m,*}}) \ge J(\pi^b_{\pparams^{b,*}})$.
\end{proof}

\textbf{Proposition
\ref{thrm:entropy_constrained_optimality}.}
\textit{
Consider entropy-constrained policy optimization $\max_{\pparams}J_0(\pi_\pparams)$ subject to $
\cH(\pi_\pparams)\ge H$ for some $H>0$ and define the optimal solution $\pparams'$, then $J_0(\pi^m_{\smash{\pparams^{m,}}'}) \ge J_0(\pi^b_{\smash{\pparams^{b,}}'})$.
}
\begin{proof}
Define $\tilde\pparams^\mixture=[\smash{\pparams^{b,}}'^\top, \cdots, \smash{\pparams^{b,}}'^\top, \smash{\pparams^\weight}^\top]^\top$ for any $\pparams^\weight$. Apparently, $\pi^\mixture_{\tilde\pparams^\mixture}(a)=\pi^\component_{\smash{\pparams^{b,}}'}(a)$. Then,
\begin{align*}
    J_0(\pi^m_{\smash{\pparams^{m,}}'}) \ge J_0(\pi^m_{\tilde\pparams^\mixture}) = J_0(\pi^b_{\smash{\pparams^{b,}}'}).
\end{align*}
\end{proof}

\textbf{Proposition
\ref{thrm:sad_fact_about_gaussian}.}
\textit{
Assume $r:\mathcal{A}\to\bR$ is an integrable function on $\mathcal{A}=\bR$. For all $\alpha>{2}r_{\max}$, $J(\pi_{\mu,\sigma})=\bE_{a\sim \mathcal{N}(\mu,\sigma)} [r(a) - \alpha \log \mathcal{N}(a;\mu,\sigma) ] $ does not have any stationary point.
}
\begin{proof}
To show that $J(\pi_{\mu,\sigma})$ does not have any stationary point, it is sufficient to show that its partial derivative with respect to $\sigma$ is lower bounded by zero:
\begin{align}
\label{eq:lower_bounding_partial_J}
    \frac{\partial J(\pi_{\mu,\sigma})}{\partial \sigma} > 0, \quad\forall \mu \in\bR, \sigma>0.
\end{align}

We first simplify the entropy term $H(\mathcal{N}(\cdot;\mu,\sigma))$ for the Gaussian policy:
\begin{align*}
    H(\mathcal{N}(\cdot;\mu,\sigma)) &= - \bE_{a\sim \mathcal{N}(\mu,\sigma)} [\log \mathcal{N}(a;\mu,\sigma)] \\
    &= - \bE_{a\sim \mathcal{N}(\mu,\sigma)} \left[ \log \left( \frac{1}{\sqrt{2\pi\sigma^2}} \exp \left(-\frac{(a-\mu)^2}{2\sigma^2} \right) \right)\right] \\
    &= \frac{1}{2} \log(2\pi\sigma^2) + \frac{1}{2\sigma^2} \bE_{a\sim \mathcal{N}(\mu,\sigma)} [ (a-\mu)^2 ] \\
    &= \frac{1}{2} \log(2\pi\sigma^2) + \frac{1}{2}.
\end{align*}
Further, we can derive its partial derivative with respect to $\sigma$: $\frac{\partial H(\mathcal{N}(\cdot;\mu,\sigma))}{\partial \sigma}=\frac{1}{\sigma}$.

Define $r(\mu, \sigma)=\bE_{a\sim \mathcal{N}(\mu,\sigma)} [r(a)]=\int_a r(a)\mathcal{N}(a|\mu,\sigma)\,da$. Then,
\begin{align*}
    \frac{\partial J(\pi_{\mu,\sigma})}{\partial \sigma}
    &= \frac{\partial}{\partial \sigma} \bE_{a\sim \mathcal{N}(\mu,\sigma)} [r(a) - \alpha \log \mathcal{N}(a;\mu,\sigma) ] \\
    &= \frac{\partial}{\partial \sigma} \left( r(\mu, \sigma) + \alpha H(\mathcal{N}(\cdot;\mu,\sigma)) \right) \\
    &= \frac{\partial r(\mu, \sigma)}{\partial \sigma} + \frac{\alpha}{\sigma}.
\end{align*}
To show \Cref{eq:lower_bounding_partial_J}, we just need to show
\begin{align}
\label{eq:lower_bounding_by_alpha}
    \sigma \frac{\partial r(\mu, \sigma)}{\partial \sigma} > - \alpha.
\end{align}

We first analyze the left hand side of \Cref{eq:lower_bounding_by_alpha}:
\begin{align*}
    \sigma \frac{\partial r(\mu, \sigma)}{\partial \sigma}
    &= \sigma \frac{\partial}{\partial \sigma} \int_a r(a)\mathcal{N}(a|\mu,\sigma)\,da \\
    &= \sigma \int_a r(a) \frac{\partial}{\partial \sigma} \mathcal{N}(a|\mu,\sigma)\,da \\
    &= \sigma \int_a r(a) \frac{\partial}{\partial \sigma} \left(\frac{1}{\sqrt{2\pi\sigma^2}} \exp \left(-\frac{1}{2\sigma^2}(a-\mu)^2 \right) \right) \,da \\
    &= \sigma \int_a r(a) \left( - \frac{1}{\sqrt{2\pi\sigma^4}} + \frac{(a-\mu)^2}{\sqrt{2\pi\sigma^8}} \right) \exp \left(-\frac{1}{2\sigma^2}(a-\mu)^2 \right) \,da \\
    &= \int_a r(a) \frac{1}{\sqrt{2\pi\sigma^2}} \exp \left(-\frac{(a-\mu)^2}{2\sigma^2} \right) \left( \frac{(a-\mu)^2}{\sigma^2} - 1 \right) \,da \\
    \overset{b=\frac{a-\mu}{\sigma}}&{=} \int_b r(\sigma b+\mu) \frac{1}{\sqrt{2\pi}} \exp \left(-\frac{b^2}{2} \right) \left( b^2 - 1 \right) \,db,
\end{align*}
which is bounded:
\begin{align*}
    \left| \sigma \frac{\partial r(\mu, \sigma)}{\partial \sigma} \right|
    &= \left| \int_b r(\sigma b+\mu) \frac{1}{\sqrt{2\pi}} \exp \left(-\frac{b^2}{2} \right) \left( b^2 - 1 \right) \,db \right| \\
    &\le \int_b \left| r(\sigma b+\mu) \right| \frac{1}{\sqrt{2\pi}} \exp \left(-\frac{b^2}{2} \right) \left| b^2 - 1 \right| \,db \\
    & \le \int_b r_{\max} \frac{1}{\sqrt{2\pi}} \exp \left(-\frac{b^2}{2} \right) \left| b^2 - 1 \right| \,db \\
    & \le r_{\max} \int_b \frac{1}{\sqrt{2\pi}} \exp \left(-\frac{b^2}{2} \right)  \left( b^2 + 1 \right) \,db \\
    & \le r_{\max} \bE_{b\sim \mathcal{N}(0,1)} \left[ b^2 +1 \right] \\
    & = {2} r_{\max}.
\end{align*}

Then for any $\alpha>{2} r_{\max}$, we have
$
    \sigma \frac{\partial r(\mu, \sigma)}{\partial \sigma} \ge - {2} r_{\max} > -\alpha.
$
\end{proof}

\textbf{\Cref{cor:alpha_relation}}
\textit{
The minimum $\alpha$ after which the mixture policy no longer has a stationary point 
is at least as large as that of the base policy, i.e.,
$
\alpha_{\min}^{\pi^m} \ge \alpha_{\min}^{\pi^b},
$
where $\alpha_{\min}^\pi=\inf\{\alpha \mid \nabla_{\pparams}J(\ppolicy)\neq \vect 0,\, \forall \pparams\}$ for policy $\ppolicy$.
}
\begin{proof}
    Similar to \Cref{thrm:entropy_regularized_optimality}, this is a direct consequence of \Cref{thrm:stationary_point_relation}.
\end{proof}

\subsection{Proofs for Results in Section \ref{sec:reparam_mixture_policies}}

Define the (discounted) occupancy measure under $\pmixturepolicy$ as  $d_\pmixturepolicy(s)\doteq\sum_{t=0}^\infty\bE_\pmixturepolicy[\gamma^t\bI(S_t=s)]$. We first prove the half-reparameterization policy gradient theorem, which is a special case of Theorem \ref{thrm:half_reparam_entropy_pg} with $\alpha=0$.

\textbf{Assumption 
\ref{assm:nice_set}.}
$\mathcal{S}$ and $\mathcal{A}$ are compact.

\textbf{Assumption 
\ref{assm:nice_function}.}
$p(s'|s,a)$, $d_0(s)$, $r(s,a)$ $f_\pparams(\epsilon;s,k)$, $f_\pparams^{-1}(a;s,k)$, $\pweightpolicy(k|s)$, $\pcomppolicy(a|s,k)$, $p(\epsilon)$, and their derivatives are continuous in variables $s$, $a$, $s'$, $\pparams$, and $\epsilon$.

\begin{theorem}[Half-Reparameterization Policy Gradient Theorem]
\label{thrm:half_reparam_pg_app}
Under Assumptions \ref{assm:nice_set} and \ref{assm:nice_function}, we have
\begin{align*}
    \nabla_\pparams J_0(\pmixturepolicy)
    = \bE_{s\sim d_\pmixturepolicy,k\sim\pweightpolicy(\cdot|s),\epsilon\sim p} \Big[ 
        & Q_\pmixturepolicy(s,f_\pparams(\epsilon;s,k)) \nabla_\pparams \log\pweightpolicy(k|s) \\
        &+ \nabla_\pparams f_\pparams(\epsilon;s,k) \nabla_a Q_\pmixturepolicy(s,a)|_{a=f_\pparams(\epsilon;s,k)} \Big].
\end{align*}
\end{theorem}
\begin{proof}
We start with the policy gradient theorem \citep{sutton1999policy}, which shows
\begin{equation*}
    \nabla_\pparams J_0(\pmixturepolicy) = \int_{s,a} d_\pmixturepolicy(s)\pmixturepolicy(a|s)Q_\pmixturepolicy(s,a)\nabla_\pparams \log\pmixturepolicy(a|s) \,da\,ds.
\end{equation*}
Then
\begin{align*}
    &\nabla_\pparams J_0(\pmixturepolicy)
    = \int_{s,a} d_\pmixturepolicy(s)\pmixturepolicy(a|s)Q_\pmixturepolicy(s,a)\nabla_\pparams \log\pmixturepolicy(a|s) \,da\,ds \\
    =& \int_{s} d_\pmixturepolicy(s) \bigg( \int_a Q_\pmixturepolicy(s,a)\nabla_\pparams \pmixturepolicy(a|s) \,da \bigg) \,ds \numberthis \label{eq:intermediate_step_value} \\
    =& \int_{s} d_\pmixturepolicy(s) \left( \int_a Q_\pmixturepolicy(s,a) \nabla_\pparams \left( \sum_{k=1}^N \pweightpolicy(k|s) \pcomppolicy(a|s,k) \right) \,da \right) \,ds \\
    =& \int_{s} d_\pmixturepolicy(s) \sum_{k=1}^N \left( \int_a Q_\pmixturepolicy(s,a) \nabla_\pparams \big( \pweightpolicy(k|s) \pcomppolicy(a|s,k) \big) \,da \right) \,ds \\
    =& \int_{s} d_\pmixturepolicy(s) \sum_{k=1}^N \bigg( \int_a Q_\pmixturepolicy(s,a) \nabla_\pparams \pweightpolicy(k|s) \pcomppolicy(a|s,k) \,da \\
        &+ \int_a Q_\pmixturepolicy(s,a) \pweightpolicy(k|s) \nabla_\pparams \pcomppolicy(a|s,k) \,da \bigg) \,ds \\
    =& \int_{s} d_\pmixturepolicy(s) \sum_{k=1}^N \Bigg( \int_a Q_\pmixturepolicy(s,a) \nabla_\pparams \log\pweightpolicy(k|s) \pweightpolicy(k|s) \pcomppolicy(a|s,k) \,da \\
        &+ \pweightpolicy(k|s) \bigg( \int_a \nabla_\pparams \big( Q_\pmixturepolicy(s,a) \pcomppolicy(a|s,k) \big) \,da 
        - \int_a \pcomppolicy(a|s,k) \nabla_\pparams Q_\pmixturepolicy(s,a) \,da \bigg) \Bigg) \,ds \\
    =& \int_{s} d_\pmixturepolicy(s) \sum_{k=1}^N \pweightpolicy(k|s) \Bigg( \int_a Q_\pmixturepolicy(s,a) \nabla_\pparams \log\pweightpolicy(k|s) \pcomppolicy(a|s,k) \,da \\
        &+ \nabla_\pparams \int_a Q_\pmixturepolicy(s,a) \pcomppolicy(a|s,k) \,da 
        - \int_a \pcomppolicy(a|s,k) \nabla_\pparams Q_\pmixturepolicy(s,a) \,da \Bigg) \,ds \\
    \overset{a=f_\pparams(\epsilon;s,k)}&{=} \int_{s} d_\pmixturepolicy(s) \sum_{k=1}^N \pweightpolicy(k|s) \Bigg( \int_\epsilon p(\epsilon) Q_\pmixturepolicy(s,f_\pparams(\epsilon;s,k)) \nabla_\pparams \log\pweightpolicy(k|s) \,d\epsilon \\
        &+ \nabla_\pparams \int_\epsilon p(\epsilon) Q_\pmixturepolicy(s,f_\pparams(\epsilon;s,k)) \,d\epsilon 
        - \int_\epsilon p(\epsilon) \nabla_\pparams Q_\pmixturepolicy(s,a)|_{a=f_\pparams(\epsilon;s,k)} \,d\epsilon \Bigg) \,ds \\
    =& \int_{s} d_\pmixturepolicy(s) \sum_{k=1}^N \pweightpolicy(k|s) \Bigg( \int_\epsilon p(\epsilon) Q_\pmixturepolicy(s,f_\pparams(\epsilon;s,k)) \nabla_\pparams \log\pweightpolicy(k|s) \,d\epsilon \\
        &+ \int_\epsilon p(\epsilon) \nabla_\pparams f_\pparams(\epsilon;s,k) \nabla_a Q_\pmixturepolicy(s,a)|_{a=f_\pparams(\epsilon;s,k)} \,d\epsilon \Bigg) \,ds \\
    =& \bE_{s\sim d_\pmixturepolicy,k\sim\pweightpolicy(\cdot|s),\epsilon\sim p} \Big[ Q_\pmixturepolicy(s,f_\pparams(\epsilon;s,k)) \nabla_\pparams \log\pweightpolicy(k|s) \\
        & \qquad\qquad\qquad\quad\quad\quad\, + \nabla_\pparams f_\pparams(\epsilon;s,k) \nabla_a Q_\pmixturepolicy(s,a)|_{a=f_\pparams(\epsilon;s,k)} \Big],
\end{align*}
where the second last equality is due to
\begin{align*}
    & \nabla_\pparams Q_\pmixturepolicy(s,f_\pparams(\epsilon;s,k)) - \nabla_\pparams Q_\pmixturepolicy(s,a)|_{a=f_\pparams(\epsilon;s,k)} \\
    =& \nabla_\pparams f_\pparams(\epsilon;s,k) \nabla_a Q_\pmixturepolicy(s,a)|_{a=f_\pparams(\epsilon;s,k)} + \nabla_\pparams Q_\pmixturepolicy(s,a)|_{a=f_\pparams(\epsilon;s,k)} - \nabla_\pparams Q_\pmixturepolicy(s,a)|_{a=f_\pparams(\epsilon;s,k)} \\
    =& \nabla_\pparams f_\pparams(\epsilon;s,k) \nabla_a Q_\pmixturepolicy(s,a)|_{a=f_\pparams(\epsilon;s,k)}.
\end{align*}
\end{proof}
\begin{remark}
\label{thrm:proof_insight}
    The key contribution of this proof is the decoupling of the gradient of the weighting policy $\pweightpolicy$ and the gradient of the component policies $\pcomppolicy$. The former, $\nabla_\pparams \pweightpolicy$, is converted back to the likelihood-ratio gradient, while the latter, $\nabla_\pparams \pcomppolicy$, is handled in the same way as in the proof the reparameterization policy gradient theorem \citep{lan2022model}.
\end{remark}

Combining the insight from Remark \ref{thrm:proof_insight} with the proof of the entropy-regularized reparameterization policy gradient theorem in \citet{lan2022model}, we can obtain 
Theorem \ref{thrm:half_reparam_entropy_pg}.

\textbf{\Cref{thrm:half_reparam_entropy_pg}} (Entropy-Regularized Half-Reparameterization Policy Gradient Theorem)\textbf{.}
\textit{
Under Assumptions \ref{assm:nice_set} and \ref{assm:nice_function}, we have
\begin{align*}
    \nabla_\pparams J(\pmixturepolicy)
    \!=\! \bE_{s\sim d_\pmixturepolicy,k\sim\pweightpolicy(\cdot|s), \epsilon\sim p} \Big[& 
    \nabla_\pparams \log\pweightpolicy(k|s) \big(Q_\pmixturepolicy(s, f_\pparams(\epsilon;s,k)) \!-\! \alpha \log\pmixturepolicy(f_\pparams(\epsilon;s,k)|s)\big) \\
    & + {
    \nabla_\pparams f_\pparams(\epsilon;s,k) \nabla_a \big(Q_\pmixturepolicy(s,a) \!-\! \alpha \log\pmixturepolicy(a|s)\big)|_{a=f_\pparams(\epsilon;s,k)}
    }\Big].
\end{align*}
}
{
\begin{proof}
From (3) of \citet{ahmed2019understanding}, we have the entropy-regularized policy gradient for the regularized objective:
\begin{equation*}
    \nabla_\pparams J(\pmixturepolicy) = \int_{s,a} d_\pmixturepolicy(s)\pmixturepolicy(a|s) \left(
    Q_\pmixturepolicy(s,a)\nabla_\pparams \log\pmixturepolicy(a|s) + \alpha \nabla_\pparams \cH(\pmixturepolicy(\cdot|s))
    \right)\,da\,ds.
\end{equation*}
The first term, $Q_\pmixturepolicy(s,a)\nabla_\pparams \log\pmixturepolicy(a|s)$, can be directly handled by Theorem \ref{thrm:half_reparam_pg_app}. Here, we analyze the second term, $\alpha \nabla_\pparams \cH(\pmixturepolicy(\cdot|s))$. Notice that
\begin{align*}
    \nabla_\pparams \cH(\pmixturepolicy(\cdot|s))
    &= - \nabla_\pparams \int_a \pmixturepolicy(a|s) \log\pmixturepolicy(a|s) \,da \\
    &= - \int_a \left(
    \nabla_\pparams \pmixturepolicy(a|s) \log\pmixturepolicy(a|s) + \pmixturepolicy(a|s) \nabla_\pparams \log\pmixturepolicy(a|s)
    \right) \,da \\
    &= - \int_a \left(
    \nabla_\pparams \pmixturepolicy(a|s) \log\pmixturepolicy(a|s) + \nabla_\pparams \pmixturepolicy(a|s)
    \right) \,da \\
    \overset{\int_a \nabla_\pparams \pmixturepolicy(a|s)\,da=0}&{=} - \int_a \nabla_\pparams \pmixturepolicy(a|s) \log\pmixturepolicy(a|s) \,da,
\end{align*}
then we have
\begin{align*}
    & \int_{s,a} d_\pmixturepolicy(s)\pmixturepolicy(a|s) \alpha \nabla_\pparams \cH(\pmixturepolicy(\cdot|s)) \,da\,ds \\
    =& \alpha \int_{s} d_\pmixturepolicy(s) \nabla_\pparams \cH(\pmixturepolicy(\cdot|s)) \,ds \\
    =& - \alpha \int_{s} d_\pmixturepolicy(s) \int_a \nabla_\pparams \pmixturepolicy(a|s) \log\pmixturepolicy(a|s) \,dads. \numberthis \label{eq:intermediate_step_entropy}
\end{align*}
Since \Cref{eq:intermediate_step_entropy} resembles \Cref{eq:intermediate_step_value}, by following the same steps in the proof of Theorem \ref{thrm:half_reparam_pg_app}, we can obtain
\begin{align*}
    & \int_{s,a} d_\pmixturepolicy(s)\pmixturepolicy(a|s) \alpha \nabla_\pparams \cH(\pmixturepolicy(\cdot|s)) \,da\,ds \\
    =& \bE_{s\sim d_\pmixturepolicy,k\sim\pweightpolicy(\cdot|s),\epsilon\sim p} \Big[ -\alpha \log\pmixturepolicy(f_\pparams(\epsilon;s,k)|s) \nabla_\pparams \log\pweightpolicy(k|s) \\
    & \qquad\qquad\qquad\quad\quad\quad\, -\alpha \nabla_\pparams f_\pparams(\epsilon;s,k) \nabla_a \log\pmixturepolicy(a|s)|_{a=f_\pparams(\epsilon;s,k)} \Big],
\end{align*}
Combining the above gradient term with the gradient term from Theorem \ref{thrm:half_reparam_pg_app} concludes the proof.
\end{proof}
}

By using the same technique, we can obtain the half-reparameterization gradient of SAC's objective in \Cref{eq:objective_surrogate}.

\begin{assumption}
\label{assm:nice_function_sac}
$Q_\qparams(s,a)$ and its derivatives are continuous in variables $s$ and $a$.
\end{assumption}

\begin{proposition}
Under Assumptions \ref{assm:nice_set}, \ref{assm:nice_function}, and \ref{assm:nice_function_sac}, we have
\begin{align*}
    \nabla_\pparams \hat J(\pweightpolicy)
    = \bE_{s\sim\mathcal{B}, k\sim\pweightpolicy(\cdot|s), \epsilon\sim p} \Big[
    &\nabla_\pparams \log\pweightpolicy(k|s) \big(Q_\qparams(s, f_\pparams(\epsilon;s,k)) - \alpha \log\pmixturepolicy(f_\pparams(\epsilon;s,k)|s)\big) \\
    &+ \nabla_\pparams \big(Q_\qparams(s, f_\pparams(\epsilon;s,k)) - \alpha \log\pmixturepolicy(f_\pparams(\epsilon;s,k)|s)\big) \Big].
\end{align*}
\end{proposition}
\begin{proof}
We first rewrite (\ref{eq:objective_surrogate}) with reparameterized component policies:
\begin{align*}
    \hat J(\pweightpolicy)
    &= \bE_{S_t\sim\mathcal{B},A_t\sim\pmixturepolicy} \left[ Q_\qparams(S_t, A_t) - \alpha \log\pmixturepolicy(A_t|S_t) \right] \\
    &= \int_s d_\mathcal{B}(s) \int_a \pmixturepolicy(a|s) \big(Q_\qparams(s,a) - \alpha \log\pmixturepolicy(a|s)\big) \,da\,ds \\
    &= \int_s d_\mathcal{B}(s) \int_a \sum_{k=1}^N \pweightpolicy(k|s) \pcomppolicy(a|s,k) \big(Q_\qparams(s,a) - \alpha \log\pmixturepolicy(a|s)\big) \,da\,ds \\
    \overset{a=f_\pparams(\epsilon;s,k)}&{=} \int_s d_\mathcal{B}(s) \int_\epsilon \sum_{k=1}^N \pweightpolicy(k|s) p(\epsilon) \big(Q_\qparams(s, f_\pparams(\epsilon;s,k)) - \alpha \log\pmixturepolicy(f_\pparams(\epsilon;s,k)|s)\big) \,d\epsilon\,ds \\
    &= \int_s d_\mathcal{B}(s) \int_\epsilon p(\epsilon) \sum_{k=1}^N \pweightpolicy(k|s) \big(Q_\qparams(s, f_\pparams(\epsilon;s,k)) - \alpha \log\pmixturepolicy(f_\pparams(\epsilon;s,k)|s)\big) \,d\epsilon\,ds.
\end{align*}
We can then derive its gradient:
\begin{align*}
    &\nabla_\pparams \hat J(\pmixturepolicy) \\
    =& \nabla_\pparams \int_s d_\mathcal{B}(s) \int_\epsilon p(\epsilon) \sum_{k=1}^N \pweightpolicy(k|s) \big(Q_\qparams(s, f_\pparams(\epsilon;s,k)) - \alpha \log\pmixturepolicy(f_\pparams(\epsilon;s,k)|s)\big) \,d\epsilon\,ds \\
    =& \int_s d_\mathcal{B}(s) \int_\epsilon p(\epsilon) \sum_{k=1}^N \nabla_\pparams \Big( \pweightpolicy(k|s) \big(Q_\qparams(s, f_\pparams(\epsilon;s,k)) - \alpha \log\pmixturepolicy(f_\pparams(\epsilon;s,k)|s)\big) \Big) \,d\epsilon\,ds \\
    =& \int_s d_\mathcal{B}(s) \int_\epsilon p(\epsilon) \sum_{k=1}^N \Big( \nabla_\pparams \pweightpolicy(k|s) \big(Q_\qparams(s, f_\pparams(\epsilon;s,k)) - \alpha \log\pmixturepolicy(f_\pparams(\epsilon;s,k)|s)\big) \\
    &\qquad + \pweightpolicy(k|s) \nabla_\pparams \big(Q_\qparams(s, f_\pparams(\epsilon;s,k)) - \alpha \log\pmixturepolicy(f_\pparams(\epsilon;s,k)|s)\big) \Big) \,d\epsilon\,ds \\
    =& \int_s d_\mathcal{B}(s) \int_\epsilon p(\epsilon) \sum_{k=1}^N \Big( \pweightpolicy(k|s) \nabla_\pparams \log\pweightpolicy(k|s) \big(Q_\qparams(s, f_\pparams(\epsilon;s,k)) - \alpha \log\pmixturepolicy(f_\pparams(\epsilon;s,k)|s)\big) \\
    &\qquad + \pweightpolicy(k|s) \nabla_\pparams \big(Q_\qparams(s, f_\pparams(\epsilon;s,k)) - \alpha \log\pmixturepolicy(f_\pparams(\epsilon;s,k)|s)\big) \Big) \,d\epsilon\,ds \\
    =& \int_s d_\mathcal{B}(s) \int_\epsilon p(\epsilon) \sum_{k=1}^N \pweightpolicy(k|s) \Big( \nabla_\pparams \log\pweightpolicy(k|s) \big(Q_\qparams(s, f_\pparams(\epsilon;s,k)) - \alpha \log\pmixturepolicy(f_\pparams(\epsilon;s,k)|s)\big) \\
    &\qquad + \nabla_\pparams \big(Q_\qparams(s, f_\pparams(\epsilon;s,k)) - \alpha \log\pmixturepolicy(f_\pparams(\epsilon;s,k)|s)\big) \Big) \,d\epsilon\,ds \\
    =& \bE_{s\sim\mathcal{B}, k\sim\pweightpolicy(\cdot|s), \epsilon\sim p} \Big[
    \nabla_\pparams \log\pweightpolicy(k|s) \big(Q_\qparams(s, f_\pparams(\epsilon;s,k)) - \alpha \log\pmixturepolicy(f_\pparams(\epsilon;s,k)|s)\big) \\
    &\qquad + \nabla_\pparams \big(Q_\qparams(s, f_\pparams(\epsilon;s,k)) - \alpha \log\pmixturepolicy(f_\pparams(\epsilon;s,k)|s)\big) \Big].
\end{align*}
\end{proof}

\subsection{Variance reduction properties of the MRP estimator}

In this section, we prove that the MRP estimator has a lower variance than the LR estimator under some smoothness conditions. This variance reduction property of the MRP estimator parallels that of the RP estimator \citep{gal2016uncertainty,xu2019variance}. Specifically, we show that the marginal variance of MRP can be expressed or controlled by the variance of the corresponding RP estimator for the individual components. Through this reduction, we can focus on analyzing the variance of the individual components, which is much easier to work with.

While the following results can be easily extend to the full reinforcement learning setting with regularization, we consider the bandit setting without regularization for the ease of presentation:
\begin{align}
    J(\ppolicy) = \mathbb{E}_{a\sim\ppolicy}[r(a)].
\end{align}
The entropy term can also be included by redefining $r(a)$ to be the sum of the reward and sample entropy.
We focus on analyzing the gradient of different estimators with respect to the {distribution parameters} and the weighting probabilities of Gaussian mixture policies. In this case, the mixture policy can be expressed as $\pmixturepolicy(a) = \sum_{k=1}^N \pi^\weight_{\pparams^\weight} (\compidx) \rawpcomppolicyatidx (a) = \sum_{k=1}^N w_\compidx \cN(a;\mu_\compidx, \sigma_\compidx^2)$ with $\pparams=\left[\smash{\pparams^{\component_1}}^\top, \cdots, \smash{\pparams^{\component_N}}^\top, \smash{\pparams^\weight}^\top\right]^\top = \left[ [\mu_1, \sigma_1]^\top, \cdots, [\mu_N, \sigma_N]^\top, [w_1, \cdots, w_N]^\top \right]^\top\in \bR^{3N}$, where $\mu_\compidx$, $\sigma_\compidx$, and $w_\compidx$ may depend on shared parameters implicitly. The corresponding estimators then are
\begin{align}
    \text{LR: }\quad &\hat \nabla_\pparams^\text{LR} J(\pmixturepolicy) = \nabla_\pparams \log \pmixturepolicy(A) r(A), \label{eq:mixture_lr_estimator} \\
    \text{MRP: }\quad &\hat \nabla_\pparams^\text{MRP} J(\pmixturepolicy) = \nabla_\pparams \sum_{\compidx=1}^N \pi^\weight_{\pparams^\weight}(\compidx) r\bigl(f_{\pparams^{\component_\compidx}}(\epsilon)\bigr), \label{eq:mixture_mrp_estimator}
\end{align}
where $A\sim \pmixturepolicy(\cdot)$, $\epsilon\sim \cN(0,1)$, $\pi^\weight_{\pparams^\weight} (\compidx)=w_\compidx$, and $f_{\pparams^{\component_\compidx}}(\epsilon)=\mu_\compidx+\epsilon\sigma_\compidx$.

Before we present the variance comparison between different estimators, we first analyze the relationship between the variance for the mixture policy and its individual components. Assuming only one of the component $\rawpcomppolicyatidx (a)=\cN(a;\mu_\compidx,\sigma_\compidx^2)$ is used, the corresponding estimators are
\begin{align}
    \text{LR: }\quad &\hat \nabla_{\pparams^{\component_\compidx}}^\text{LR} J(\rawpcomppolicyatidx ) = \nabla_{\pparams^{\component_\compidx}} \log \rawpcomppolicyatidx (A) r(A), \\
    \text{RP: }\quad &\hat \nabla_{\pparams^{\component_\compidx}}^\text{RP} J(\rawpcomppolicyatidx ) = \nabla_{\pparams^{\component_\compidx}} r\bigl(f_{\pparams^{\component_\compidx}}(\epsilon)\bigr).
\end{align}

Define $\mixturepolicy\doteq\pmixturepolicy$, $\comppolicyatidx\doteq \rawpcomppolicyatidx$, $\hat \partial_{\theta_i^{\component_\compidx}} \cdot \doteq \big[ \hat \nabla_{\pparams^{\component_\compidx}} \cdot \big]_i$, $\hat \partial_{\weight_\compidx} \cdot \doteq \big[ \hat \nabla_{\pparams^{\weight}} \cdot \big]_\compidx$ and $\rho^{\component_\compidx}_{\mixture}(A) =\frac{\comppolicyatidx(A)}{\mixturepolicy(A)}$, where $\theta_i^{\component_\compidx}$ and $\weight_\compidx$ are the $i$-th distribution parameter and the weight of the $k$-th component policy, respectively. For Gaussian component policies, $\theta_i^{\component_\compidx} \in \{ \mu_\compidx, \sigma_\compidx \}$. We have the following relationships.

\begin{lemma} \label{prop:mixture_component_variance_lr}
    The marginal variances of the LR estimator for mixture policies satisfy 
    \begin{align*}
        \bV_{\mixturepolicy(A)} \left( \hat \partial_{\theta_i^{\component_\compidx}}^\text{LR} J(\pmixturepolicy) \right) 
        &= w_k^2 \bV_{\mixturepolicy(A)} \left( \rho^{\component_\compidx}_{\mixture}(A) \hat \partial_{\theta_i^{\component_\compidx}}^\text{LR} J(\rawpcomppolicyatidx) \right)
        , \\
        \bV_{\mixturepolicy(A)} \left( \hat \partial_{\weight_\compidx}^\text{LR} J(\pmixturepolicy) \right) 
        &= \bV_{\mixturepolicy(A)} \bigl( \rho^{\component_\compidx}_{\mixture}(A) r(A) \bigr).
    \end{align*}
\end{lemma}

\begin{lemma} \label{prop:mixture_component_variance_mrp}
    The marginal variances of the MRP estimator for mixture policies satisfy
    \begin{align*}
        \bV_{\cN(\epsilon;0,1)} \left( \hat \partial_{\theta_i^{\component_\compidx}}^\text{MRP} J(\pmixturepolicy) \right) 
        &= w_k^2 \bV_{\cN(\epsilon;0,1)} \left( \hat \partial_{\theta_i^{\component_\compidx}}^\text{RP} J(\rawpcomppolicyatidx) \right) 
        , \\
        \bV_{\cN(\epsilon;0,1)} \left( \hat \partial_{\weight_\compidx}^\text{MRP} J(\pmixturepolicy) \right) 
        &= \bV_{\comppolicyatidx(A)} \bigl( r(A) \bigr).
    \end{align*}
\end{lemma}

Similar to \citet{gal2016uncertainty}, we assume the reward function $r$ satisfies certain smoothness conditions (see Assumption \ref{assm:variance_reduction_condition_full}). 
\Cref{assm:variance_reduction_condition_full} gives the condition of the reward function under which the MRP estimator has lower variance. Heuristically, it requires the reward function to be smooth relative to the noise scale. For example, $\sin(x)$ satisfies the condition for $\cN(0,1)$ while $\sin(10x)$ does not (see Table 3.3 in \citeauthor{gal2016uncertainty} for a numerical experiment).

\begin{assumption}[Expanded form of \Cref{assm:variance_reduction_condition_full_main}] \label{assm:variance_reduction_condition_full}
    $r: \bR \to \bR$ is twice differentiable with first and second derivatives $r'$ and $r''$. 
    For all $\compidx\in\{1,\cdots,N\}$ and $g(a)\in\{ (a-\mu_\compidx)r(a), r'(a), \bigl((a-\mu_\compidx)^2-\sigma_\compidx^2\bigr) r(a), (a-\mu_\compidx)r'(a) \}$, $g$ has finite variance and its first derivative $g'$ is absolutely integrable under $\comppolicyatidx$: $\bV_{\comppolicyatidx(A)}\bigl(g(A)\bigr)< \infty$ and $\bE_{\comppolicyatidx(A)}\bigl[ \left| g'(A) \right| \bigr] < \infty$, with $\comppolicyatidx(a)=\cN(a;\mu_\compidx,\sigma_\compidx)$.
    Further, it holds that
    \begin{align*}
        \sum_{k=1}^N \left(
        \bE_{\comppolicyatidx(A)} \bigl[ (A-\mu_\compidx) r'(A) + r(A) \bigr]^2 
        - \sigma_\compidx^4 \bE_{\comppolicyatidx(A)} \left[ r''(A)^2 \right] 
        \right) &\ge 0, \\
        \sum_{k=1}^N \left(
        \bE_{\comppolicyatidx(A)} \left[ \bigl((A-\mu_\compidx)^2-\sigma_\compidx^2\bigr) r'(A) + 2 (A-\mu_\compidx) r(A) \right]^2 \right. & \\
        \left. - \sigma_\compidx^4 \bE_{\comppolicyatidx(A)} \left[ \bigl( (A-\mu_\compidx) r''(A) + r'(A) \bigr)^2 \right] 
        \right) &\ge 0.
    \end{align*}
\end{assumption}

\begin{assumption}[Restatement of \Cref{assm:when_importance_sampling_is_bad_main}] \label{assm:when_importance_sampling_is_bad}
    The sum of the variance of the importance-sampling LR estimator over all components is larger than that of the on-policy LR estimator:
    \begin{align*}
        \sum_{k=1}^N \left(
        \bV_{\mixturepolicy(A)} \left( \rho^{\component_\compidx}_{\mixture}(A) \hat \partial_{\theta_i^{\component_\compidx}}^\text{LR} J(\rawpcomppolicyatidx) \right) - 
        \bV_{\comppolicyatidx(A)} \left( \hat \partial_{\theta_i^{\component_\compidx}}^\text{LR} J(\rawpcomppolicyatidx) \right)
        \right) &\ge 0 \quad \text{for} \quad \theta_i^{\component_\compidx}\in\{\mu_\compidx,\sigma_\compidx\}, \\
        \sum_{k=1}^N \left(
        \bV_{\mixturepolicy(A)} \left( \rho^{\component_\compidx}_{\mixture}(A) r(A) \right) - 
        \bV_{\comppolicyatidx(A)} \bigl( r(A) \bigr)
        \right) &\ge 0.
    \end{align*}
\end{assumption}

\begin{proposition}[Restatement of \Cref{prop:variance_reduction_main}] \label{prop:variance_reduction}
    Under Assumptions \ref{assm:variance_reduction_condition_full} and \ref{assm:when_importance_sampling_is_bad}, the trace of the covariance matrix of the MRP estimator in \Cref{eq:mixture_mrp_estimator} is smaller than that of the LR estimator in \Cref{eq:mixture_lr_estimator}:
    \begin{align*}
        \Tr\left( \bC_{\cN(\epsilon;0,1)} \left( \hat \nabla_{\pparams}^\text{MRP} J(\pmixturepolicy) \right) \right) \le \Tr\left( \bC_{\mixturepolicy(A)} \left( \hat \nabla_{\pparams}^\text{LR} J(\pmixturepolicy) \right) \right).
    \end{align*}
\end{proposition}

\subsubsection{Proofs}

\textit{Proof of \Cref{prop:mixture_component_variance_lr}.}

For any fixed component index $\compidx$ and any of its distribution parameters
$\theta^{\component_k}_i\in\{\mu_\compidx,\sigma_\compidx\}$, we begin with inspecting the corresponding partial derivative
\[
\hat\partial_{\theta^{\component_k}_i}^{\text{LR}}J(\pmixturepolicy)
   \;=\;r(A)\,{\partial_{\theta^{\component_k}_i}}\!\log\pmixturepolicy(A).
\]
Because only the $k$-th component depends on~$\theta^{\component_\compidx}_i$,
\begin{align*}
\partial_{\theta^{\component_\compidx}_i}\pmixturepolicy(A)
   &=w_\compidx\,\partial_{\theta^{\component_\compidx}_i}\rawpcomppolicyatidx(A), \\
\partial_{\theta^{\component_\compidx}_i}\!\log\pmixturepolicy(A)
   &=w_\compidx\,\frac{\rawpcomppolicyatidx(A)}{\pmixturepolicy(A)}
          \,\partial_{\theta^{\component_\compidx}_i}\!\log\rawpcomppolicyatidx(A) \\
   &=w_\compidx\,\rho^{\component_\compidx}_{\mixture}(A)\,
          \partial_{\theta^{\component_\compidx}_i}\!\log\rawpcomppolicyatidx(A).
\end{align*}
Hence
\[
\hat\partial_{\theta^{\component_\compidx}_i}^{\text{LR}}J(\pmixturepolicy)
   =w_\compidx\,\rho^{\component_\compidx}_{\mixture}(A)\,
        \hat\partial_{\theta^{\component_\compidx}_i}^{\text{LR}}J(\rawpcomppolicyatidx).
\]
Taking variances under $A\sim\pmixturepolicy$ gives
\[
\bV_{\mixturepolicy(A)}
    \!\left(\hat\partial_{\theta^{\component_\compidx}_i}^{\text{LR}}J(\pmixturepolicy)\right)
   =w_\compidx^2\,\bV_{\mixturepolicy(A)}
           \!\left(\rho^{\component_\compidx}_{\mixture}(A)\,
                  \hat\partial_{\theta^{\component_\compidx}_i}^{\text{LR}}J(\rawpcomppolicyatidx)
             \right),
\]
which is exactly the first equality in the proposition.

For the mixture weight parameter $\weight_\compidx$, we note that
\[
\partial_{\weight_\compidx}\!\log\pmixturepolicy(A)
      =\frac{1}{\pmixturepolicy(A)}\,{\partial_{\weight_\compidx}}
        \Bigl(\sum_{\ell=1}^N\weight_\ell
               \rawpcomppolicyat{\ell}(A)\Bigr)
      =\frac{\rawpcomppolicyatidx(A)}{\pmixturepolicy(A)}
      =\rho^{\component_\compidx}_{\mixture}(A),
\]
so that
\(
\hat\partial_{\weight_\compidx}^{\text{LR}}J(\pmixturepolicy)
   =\rho^{\component_\compidx}_{\mixture}(A)\,r(A).
\)
Taking the variance under $A\sim\pmixturepolicy$ immediately yields the
second equality:
\[
\bV_{\mixturepolicy(A)}
   \!\left(\hat\partial_{\weight_\compidx}^{\text{LR}}J(\pmixturepolicy)\right)
   =\bV_{\mixturepolicy(A)}
      \!\bigl(\rho^{\component_\compidx}_{\mixture}(A)\,r(A)\bigr).
\]
\hfill\(\square\)

\textit{Proof of \Cref{prop:mixture_component_variance_mrp}.} 

For an MRP sample we first draw
\(\epsilon\sim\cN(0,1)\) \emph{once} and deterministically
construct the \(N\) actions 
\(A_k=f_{\pparams^{\component_k}}(\epsilon)=\mu_k+\epsilon\sigma_k\).
Hence
\[
\hat\partial_{\theta^{\component_k}_i}^{\text{MRP}}J(\pmixturepolicy)
     =\partial_{\theta^{\component_k}_i}
        \Bigl(w_k\,r\!\bigl(f_{\pparams^{\component_k}}(\epsilon)\bigr)\Bigr)
     =w_k\,\hat\partial_{\theta^{\component_k}_i}^{\text{RP}}J(\rawpcomppolicyatidx),
\]
so that
\[
    \bV_{\cN(\epsilon;0,1)} \left( \hat \partial_{\theta_i^{\component_\compidx}}^\text{MRP} J(\pmixturepolicy) \right) 
        = w_k^2 \bV_{\cN(\epsilon;0,1)} \left( \hat \partial_{\theta_i^{\component_\compidx}}^\text{RP} J(\rawpcomppolicyatidx) \right).
\]
For the weight parameter,
\(
\hat\partial_{\theta_k^{\weight}}^{\text{MRP}}J(\pmixturepolicy)
      =r\!\bigl(f_{\pparams^{\component_k}}(\epsilon)\bigr).
\)
Because \(A=f_{\pparams^{\component_k}}(\epsilon)\) with
\(A\sim\rawpcomppolicyatidx\), the stated identity of variances follows
immediately.
\hfill\(\square\)

\textit{Proof of \Cref{prop:variance_reduction}.}

To prove the inequality in the proposition, it is sufficient to show, for $\theta_i^{\compidx}\in\{\mu_\compidx,\sigma_\compidx,w_\compidx\}$, 
\begin{align}
    \sum_{k=1}^N \left(
    \bV_{\mixturepolicy(A)} \left( \hat \partial_{\theta_i^{\compidx}}^\text{LR} J(\pmixturepolicy) \right) -
    \bV_{\cN(\epsilon;0,1)} \left( \hat \partial_{\theta_i^{\compidx}}^\text{MRP} J(\pmixturepolicy) \right) 
    \right) \ge 0. \label{eq:key_inequality_to_show}
\end{align}

For $\theta_i^{\compidx} = w_\compidx$, this is immediate after applying Propositions \ref{prop:mixture_component_variance_lr} and \ref{prop:mixture_component_variance_mrp} under \Cref{assm:when_importance_sampling_is_bad}.

For $\theta_i^{\compidx} \in \{\mu_\compidx, \sigma_\compidx\}$, by Propositions \ref{prop:mixture_component_variance_lr} and \ref{prop:mixture_component_variance_mrp}, \Cref{eq:key_inequality_to_show} holds if the following is satisfied
\begin{align*}
    \sum_{k=1}^N \left(
        \bV_{\mixturepolicy(A)} \left( \rho^{\component_\compidx}_{\mixture}(A) \hat \partial_{\theta_i^{\component_\compidx}}^\text{LR} J(\rawpcomppolicyatidx) \right) - 
        \bV_{\cN(\epsilon;0,1)} \left( \hat \partial_{\theta_i^{\component_\compidx}}^\text{RP} J(\rawpcomppolicyatidx) \right)
    \right) \ge 0.
\end{align*}
Under \Cref{assm:variance_reduction_condition_full}, it is sufficient to show that the sum of the variance of the LR estimator over all component policies is larger than that of the RP estimator:
\begin{align}
    \sum_{k=1}^N \left(
        \bV_{\comppolicyatidx(A)} \left( \hat \partial_{\theta_i^{\component_\compidx}}^\text{LR} J(\rawpcomppolicyatidx) \right) 
        - 
        \bV_{\cN(\epsilon;0,1)} \left( \hat \partial_{\theta_i^{\component_\compidx}}^\text{RP} J(\rawpcomppolicyatidx) \right)
    \right) \ge 0. \label{eq:key_inequality_to_show_simplified}
\end{align}
Next, we show the above inequality holds for $\theta_i^{\component_\compidx}=\mu_\compidx$ under our assumptions, following by the result for $\theta_i^{\component_\compidx}=\sigma_\compidx$.

Since $\rawpcomppolicyatidx$ is Gaussian, we have
\begin{align*}
    \hat \partial_{\mu_\compidx}^\text{LR} J(\rawpcomppolicyatidx) = r(A) \frac{A-\mu_\compidx}{\sigma_k^2},
    \quad
    \hat \partial_{\mu_\compidx}^\text{RP} J(\rawpcomppolicyatidx) = r'( \mu_k+\epsilon\sigma_k ).
\end{align*}
Following \citet{gal2016uncertainty}, we use Proposition $3.2$ in \citet{cacoullos1982upper}, which states if a real-valued function $g(a)$ has finite variance and its first derivative $g'(a)$ is absolutely integrable under $\cN(\mu,\sigma)$, then
$$
    \sigma^2 \bE_{\cN(A;\mu,\sigma)} \left[ g'(A) \right]^2 
    \le \bV_{\cN(A;\mu,\sigma)} \left[ g(A) \right]
    \le \sigma^2 \bE_{\cN(A;\mu,\sigma)} \left[ g'(A)^2 \right].
$$
Invoking the above proposition with $g(a)=(a-\mu_k) r(a)$ and $g(a)=r'(a)$, respectively, we have
\begin{align*}
    \sigma_k^2 \bE_{\comppolicyatidx(A)} \bigl[ (A-\mu_k) r'(A) + r(A) \bigr]^2 
    &\le \bV_{\comppolicyatidx(A)} \bigl( (A-\mu_\compidx) r(A) \bigr), \\
    \bV_{\comppolicyatidx(A)} \bigl( r'(A) \bigr)
    &\le \sigma_k^2 \bE_{\comppolicyatidx(A)} \bigl[ r''(A)^2 \bigr].
\end{align*}
Under \Cref{assm:variance_reduction_condition_full}, we then have
\begin{align*}
    \sum_{k=1}^N \bV_{\cN(\epsilon;0,1)} \left( \hat \partial_{\mu_\compidx}^\text{RP} J(\rawpcomppolicyatidx) \right)
    &= \sum_{k=1}^N \bV_{\cN(\epsilon;0,1)} \left( r'( \mu_k+\epsilon\sigma_k ) \right) \\
    &= \sum_{k=1}^N \bV_{\comppolicyatidx} \left( r'(A) \right) \\
    &\le \sum_{k=1}^N \sigma_k^2 \bE_{\comppolicyatidx(A)} \left[ r''(A)^2 \right] \\
    &\le \sum_{k=1}^N \frac{1}{\sigma_k^2} \bE_{\comppolicyatidx(A)} \bigl[ (A-\mu_k) r'(A) + r(A) \bigr]^2 \\
    &\le \sum_{k=1}^N \frac{1}{\sigma_k^4} \bV_{\comppolicyatidx(A)} \bigl( (A-\mu_\compidx) r(A) \bigr) \\
    &= \sum_{k=1}^N \bV_{\comppolicyatidx(A)} \left( \hat \partial_{\mu_\compidx}^\text{LR} J(\rawpcomppolicyatidx) \right),
\end{align*}
which proves \Cref{eq:key_inequality_to_show_simplified} for $\theta_i^\compidx=\mu_\compidx$.

Similarly, for $\theta_i^\compidx=\sigma_\compidx$, we have
\begin{align*}
    \hat \partial_{\sigma_\compidx}^\text{LR} J(\rawpcomppolicyatidx) = r(A) \frac{(A-\mu_\compidx)^2-\sigma_k^2}{\sigma_k^3},
    \quad
    \hat \partial_{\sigma_\compidx}^\text{RP} J(\rawpcomppolicyatidx) = \epsilon r'(\mu_k + \epsilon \sigma_k).
\end{align*}
Invoking Proposition $3.2$ in \citet{cacoullos1982upper} with $g(a)=\bigl((a-\mu_\compidx)^2-\sigma_\compidx^2\bigr) r(a)$ and $g(a)=(a-\mu_\compidx)r'(a)$, respectively, we have
\begin{align*}
    \sigma_k^2 \bE_{\comppolicyatidx(A)} \left[ \bigl((A-\mu_\compidx)^2-\sigma_\compidx^2\bigr) r'(A) + 2 (A-\mu_\compidx) r(A) \right]^2
    &\le \bV_{\comppolicyatidx(A)} \left( \bigl((A-\mu_\compidx)^2-\sigma_\compidx^2\bigr) r(A) \right), \\
    \bV_{\comppolicyatidx(A)} \bigl( (A-\mu_\compidx)r'(A) \bigr)
    &\le \sigma_k^2 \bE_{\comppolicyatidx(A)} \left[ \bigl( (A-\mu_\compidx) r''(A) + r'(A) \bigr)^2 \right].
\end{align*}
Under \Cref{assm:variance_reduction_condition_full}, we have
\begin{align*}
    \sum_{k=1}^N \bV_{\cN(\epsilon;0,1)} \left( \hat \partial_{\sigma_\compidx}^\text{RP} J(\rawpcomppolicyatidx) \right)
    &= \sum_{k=1}^N \bV_{\cN(\epsilon;0,1)} \left( \epsilon r'(\mu_k + \epsilon \sigma_k) \right) \\
    &= \sum_{k=1}^N \bV_{\comppolicyatidx} \left( r'(A)\frac{A-\mu_k}{\sigma_k} \right) \\
    &= \sum_{k=1}^N \frac{1}{\sigma_k^2} \bV_{\comppolicyatidx} \bigl( (A-\mu_\compidx)r'(A) \bigr) \\
    &\le \sum_{k=1}^N \bE_{\comppolicyatidx(A)} \left[ \bigl( (A-\mu_\compidx) r''(A) + r'(A) \bigr)^2 \right] \\
    &\le \sum_{k=1}^N \frac{1}{\sigma_k^4} \bE_{\comppolicyatidx(A)} \left[ \bigl((A-\mu_\compidx)^2-\sigma_\compidx^2\bigr) r'(A) + 2 (A-\mu_\compidx) r(A) \right]^2 \\
    &\le \sum_{k=1}^N \frac{1}{\sigma_k^6} \bV_{\comppolicyatidx(A)} \left( \bigl((A-\mu_\compidx)^2-\sigma_\compidx^2\bigr) r(A) \right) \\
    &= \sum_{k=1}^N \bV_{\comppolicyatidx(A)} \left( \hat \partial_{\sigma_\compidx}^\text{LR} J(\rawpcomppolicyatidx) \right),
\end{align*}
which proves \Cref{eq:key_inequality_to_show_simplified} for $\theta_i^\compidx=\sigma_\compidx$.
\hfill\(\square\)

\clearpage

\section{Stationary Point Study Details}
\label{sec:app_numerical}

In this section, we provide more details on the stationary point study presented in Section \ref{sec:stationary_point_theory}. The reward function of the bimodal bandit is the normalized summation of two Gaussians' density functions whose standard deviations are both $0.5$ and whose means are $-1$ and $1$, respectively. We use the default optimization algorithm for variable with bounds in SciPy \citep{virtanen2020scipy} to optimize the entropy regularized objective $J(\ppolicy)=\bE_{a\sim\ppolicy}[r(a) -\alpha\log\ppolicy(a)]$, where $r(a)$ is the value of the action depicted in Figure \ref{fig:analytical_bimodal_bandit}. We then sort the obtained stationary points based on their regularized values $J(\ppolicy)$ and use the ones that have the highest values as the parameters of optimal policies for Figure \ref{fig:analytical_bimodal_bandit}. 

For each policy class, we run the default optimization algorithm for $100$ trials, each with a set of randomly sampled initial policy parameters. Specifically, the initial means, log standard deviations, and mixing weights are randomly sampled from $[-2, 2]$, $[-3,0]$, and $[0,1]$, respectively. To avoid numerical issues in numerical integral when the standard deviation gets too large, we impose an upper bound of $3$ for the log standard deviation. In addition, the mixing weights are defined and bounded within $[0, 1]$. We initially run this optimization procedure for seven different entropy scales: $\alpha\in\{0.05, 0.1, 0.2, 0.3, 0.4, 0.5, 0.6\}$. With difference shown when $\alpha$ is between $0.2$ and $0.5$ (see Figure \ref{fig:analytical_bimodal_bandit} (Middle)), we additionally run another eight entropy scales: $\alpha\in\{0.22, 0.24, 0.26, 0.28, 0.325, 0.35, 0.375, 0.45\}$ to obtain more insights when $\alpha$ within this range. Note that, we did not obtain any convergent results for the Gaussian policy when $\alpha>=0.325$ and for the Gaussian mixture (GM) policy when $\alpha>0.5$.

\begin{figure}[htb]
    \centering
    \includegraphics[width=0.8\linewidth]{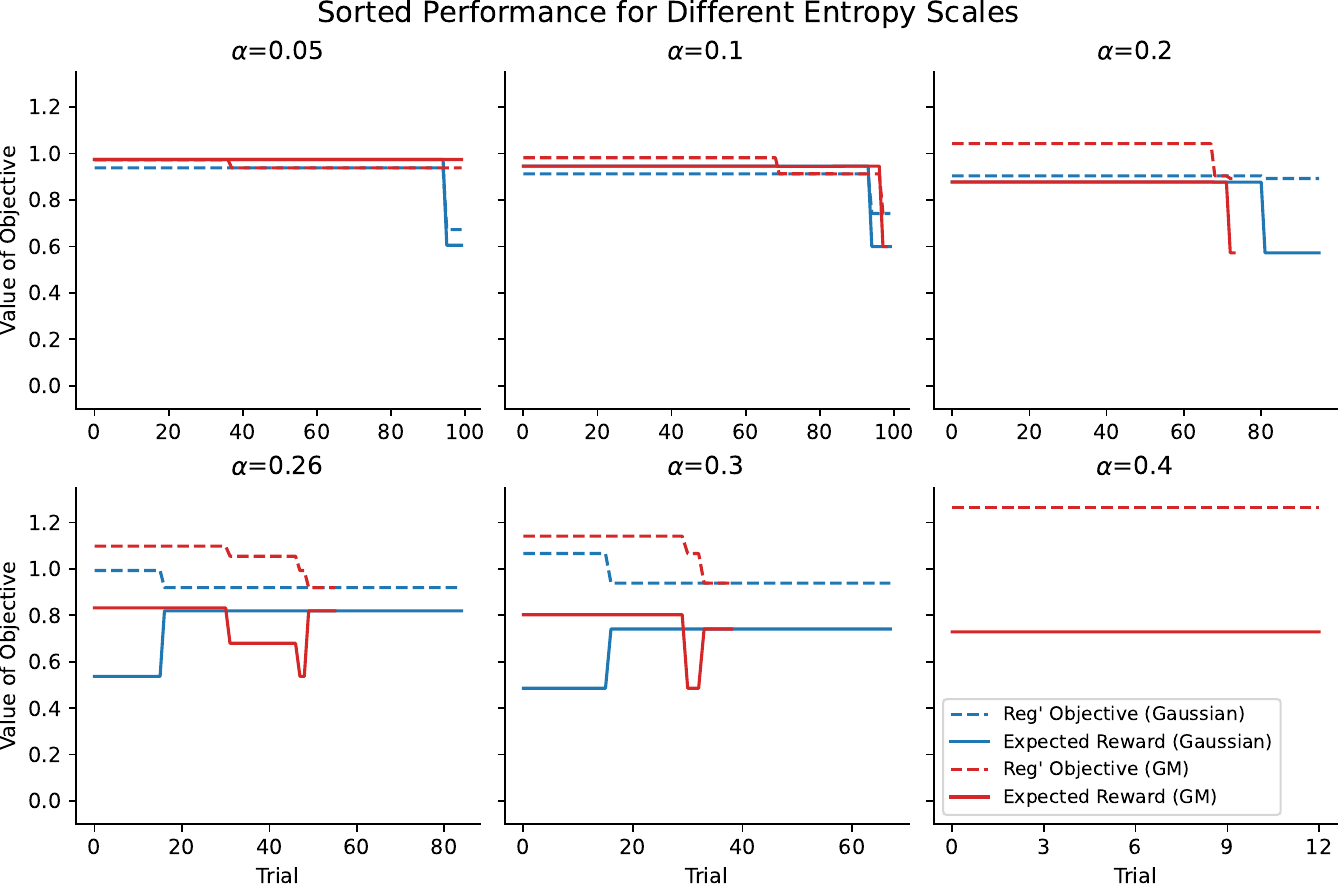}
    \caption{Expected reward and objective value of the regularized objective's stationary points found in $100$ trials for the Gaussian and the Gaussian mixture policies. Note that the stationary points are sorted by their regularized objective value for clarity (dashed line). When a line ends before Trial $100$, it means that the rest of the trials  either diverged or encountered a numerical issue. The difference between the dash line and the solid line represents the differential entropy scaled by $\alpha$.}
    \label{fig:analytical_bimodal_bandit_sorted_perf}
\end{figure}

Figure \ref{fig:analytical_bimodal_bandit_sorted_perf} shows the expected reward and objective value of all the stationary points found for representative $\alpha$. Apart from the observation discussed in Section \ref{sec:stationary_point_theory}, we can see that both the Gaussian and the GM mixture often have different types of stationary points. The optimal Gaussian policy concentrates on one of the reward modes with a high expected reward with small $\alpha$, but it then shifts to the middle of two modes in $\alpha=0.26$ and $\alpha=0.3$ with a lower expected reward but a much higher entropy. The optimal GM policy, on the other hand, always has two modes covering two reward modes: it obtains a high reward while maintaining a higher entropy.

\section{Experiment Details}
\label{sec:app_experimet_details}

\subsection{Experimental Details}
\label{sec:app_sub_experimental_details_app}

In this section, we supply the omitted experimental details in Section \ref{sec:experiments}. In addition, we will open source our code for reproducibility after publication.

\paragraph{Common configurations.} For SAC with a likelihood-ratio gradient for the actor, we sample $N_b=30$ actions for the given state and use the average of the corresponding action values as the baseline.
We also use such a baseline for the likelihood-ratio part of the half-reparameterization gradient estimator. We use the Adam optimizer \citep{kingma2014adam} with $\beta_1=0.9$ and $\beta_2=0.999$ for all experiments.

\paragraph{Synthetic multimodal bandits.} We create $100$ continuous bandits with the reward function proportional to the summation of $30$ randomly generated Gaussian density functions. The means and standard deviations are uniformly sampled from $[-3, 3]$ and $[0.1, 1.0]$, respectively. We use a two-layer feedforward network with a hidden dimension of $16$, a replay buffer size of $5000$, and a batch size of $32$. Note that we replace the critic with the true value function as these bandits with a deterministic reward function are for illustration purposes. We sweep the initial actor step size $\eta_{p,0}=10^x$ for $x\in\{-4, -3, -2, -1\}$ and the entropy scale $\alpha=10^y$ for $y\in\{-4, -3, -2, -1\}$. 
We report additional plots in \Cref{sec:app_sub_multimodal_bandits}.

\paragraph{Common continuous control benchmarks.} We use a two-layer feedforward network with a hidden dimension of $256$, a replay buffer size of $1,000,000$, and a batch size of $100$. We use the automatic entropy tuning and the same initial step size $3\times10^{-4}$ for the actor, critic, and entropy scale. We use a double Q network and a target network for the critic, which is an exponential moving average of the critic with a smoothing factor of $0.005$. In the initial $10,000$ steps, the actions are uniformly sampled. We report the state and action dimensions of Gym MuJoCo environments in Table \ref{tab:mujoco_env_dimensions} and learning curves in each individual environment in \Cref{sec:app_additional_mujoco_dmc_plots}.

\paragraph{Classic control environments.} We use a two-layer feedforward network with a hidden dimension of $64$, a replay buffer size of $100,000$, and a batch size of $32$. We use a double Q network and a target network for the critic, which is an exponential moving average of the critic with a smoothing factor of $0.01$. We sweep the entropy scale $\alpha=10^y$ for $y\in\{-3,-2,-1,0\}$. In addition, we sweep the initial critic step size $\eta_{q,0}=10^x$ for $x\in\{-5, -4, -3,-2\}$ and the initial actor step size $\eta_{p,0}=\kappa \eta_{q,0}$ for $\kappa\in\{10^{-2}, 10^{-1}, 1, 10\}$. Details of the versions and reward functions of these environment is in Section \ref{sec:app_classic_control_rewards}. We report the best hyperparameter setting in Table \ref{tab:best_hypers_classic_control} and additional plots in Section \ref{sec:app_sub_visualization}.

\subsection{Computational Overhead of Using Mixture Policies}
\label{sec:app_sub_computation_overhead}

\paragraph{Network architecture of mixture policies.} Compared to the base policy's actor, the mixture policy's actor has additional heads for the additional parameters in the mixture distribution. For example, the last layer of a squashed Gaussian policy's actor has two outputs (one for the mean and one for the standard deviation), while the last layer of a squashed Gaussian policy's actor with five components has $15$ outputs (five for the means, five for the standard deviations, and five for the mixing weights).
{
Such an architecture induces only negligible additional memory usage (see \Cref{tab:training_time}). Similarly, the inference time should also be marginal.
}

{

\paragraph{Computational overhead during training.}
Using the MRP estimator does slightly increase actor training cost as it requires $N$ Q-evaluations per state, compared to $1$ for Gaussian policies. However, the additional wall-clock training time is modest, since computation can be batched and parallelized. Further, a lot of SAC’s wall-clock time comes from critic updates (each with 4 separate Q-evaluations per transition, a factor of 2 from double-Q and another factor of 2 from two consecutive states).

\paragraph{Reference training and inference time.}
We provide training and inference time samples for \texttt{Pendulum} and \texttt{HalfCheetah-v3} in Table \ref{tab:training_time} for reference. We can see that mixture policies require modest additional training and inference time. However, we use PyTorch \citep{paszke2019pytorch} in our experiments and have not optimized our code for more efficient training and inference. We expect the gap between mixture policies and base policies to be much smaller if one switch to a JAX implementation \citep{jax2018github}.

For transparency, the training time samples are obtained via an example run when the server is idle and no other active program is running. The CPU of the server is AMD Ryzen 9 5900X 12-Core Processor, and the GPU of the server is NVIDIA Geforce RTX 3080 Ti.
}

\begin{table}[ht]
  \caption{State and action dimensions of MuJoCo environments.}
  \label{tab:mujoco_env_dimensions}
  \centering
  \begin{tabular}{l|cc}
    \toprule
    \textbf{Environment} & \textbf{State Dimension} & \textbf{Action Dimension} \\
    \midrule
    Hopper-v3           & 11                      & 3                        \\
    Walker2d-v3         & 17                      & 6                        \\
    HalfCheetah-v3      & 17                      & 6                        \\
    Ant-v3              & 111                     & 8                        \\
    Swimmer-v3          & 8                       & 2                        \\
    Humanoid-v3         & 376                     & 17                       \\
    HumanoidStandup-v2  & 376                     & 17                       \\
    \bottomrule
  \end{tabular}
\end{table}

\begin{table}[htb]
  \caption{Tuned hyperparameters in classic control environments.}
  \label{tab:best_hypers_classic_control}
  \centering
  \begin{tabular}{l|l|ccc}
    \toprule
    Environment & Algorithm & $\eta_{q,0}$ & $\kappa$ & $\alpha$ \\
    \midrule
    \multirow{5}{*}{\texttt{ShapedPendulum}} 
    & SG-RP & $10^{-2}$ & $10^{-1}$ & $10^{-1}$ \\
    & USGM-RP & $10^{-2}$ & $10^{-1}$ & $10^{-1}$ \\
    & SGM-HalfRP & $10^{-2}$ & $10^{-1}$ & $10^{-1}$ \\
    & SGM-MRP & $10^{-2}$ & $10^{-1}$ & $10^{-2}$ \\
    & SGM-GumbelRP & $10^{-2}$ & $10^{-1}$ & $10^{-1}$ \\
    \midrule
    \multirow{5}{*}{\texttt{ShapedAcrobot}} 
    & SG-RP & $10^{-2}$ & $10^{-2}$ & $10^{-2}$ \\
    & USGM-RP & $10^{-2}$ & $10^{-1}$ & $10^{-3}$ \\
    & SGM-HalfRP & $10^{-2}$ & $10^{-1}$ & $10^{-2}$ \\
    & SGM-MRP & $10^{-2}$ & $10^{-1}$ & $10^{-2}$ \\
    & SGM-GumbelRP & $10^{-2}$ & $10^{-1}$ & $10^{-2}$ \\
    \midrule
    \multirow{5}{*}{\texttt{ShapedMountainCar}} 
    & SG-RP & $10^{-3}$ & $1$ & $10^{-1}$ \\
    & USGM-RP & $10^{-2}$ & $10^{-1}$ & $10^{-1}$ \\
    & SGM-HalfRP & $10^{-2}$ & $10^{-1}$ & $10^{-1}$ \\
    & SGM-MRP & $10^{-3}$ & $1$ & $10^{-1}$ \\
    & SGM-GumbelRP & $10^{-3}$ & $1$ & $10^{-1}$ \\
    \midrule
    \multirow{5}{*}{\texttt{Pendulum}} 
    & SG-RP & $10^{-3}$ & $1$ & $10^{-2}$ \\
    & USGM-RP & $10^{-3}$ & $1$ & $10^{-2}$ \\
    & SGM-HalfRP & $10^{-3}$ & $1$ & $10^{-2}$ \\
    & SGM-MRP & $10^{-3}$ & $1$ & $10^{-3}$ \\
    & SGM-GumbelRP & $10^{-3}$ & $1$ & $10^{-2}$ \\
    \midrule
    \multirow{5}{*}{\texttt{Acrobot}} 
    & SG-RP & $10^{-3}$ & $10^{-2}$ & $10^{-3}$ \\
    & USGM-RP & $10^{-3}$ & $1$ & $10^{-3}$ \\
    & SGM-HalfRP & $10^{-3}$ & $1$ & $10^{-3}$ \\
    & SGM-MRP & $10^{-3}$ & $1$ & $10^{-2}$ \\
    & SGM-GumbelRP & $10^{-3}$ & $10^{-1}$ & $10^{-3}$ \\
    \midrule
    \multirow{5}{*}{\texttt{MountainCar}} 
    & SG-RP & $10^{-4}$ & $10$ & $10^{-2}$ \\
    & USGM-RP & $10^{-4}$ & $10$ & $10^{-3}$ \\
    & SGM-HalfRP & $10^{-3}$ & $1$ & $10^{-2}$ \\
    & SGM-MRP & $10^{-4}$ & $10$ & $10^{-2}$ \\
    & SGM-GumbelRP & $10^{-4}$ & $10$ & $10^{-2}$ \\
    \bottomrule
  \end{tabular}
\end{table}

\begin{table}[htb]
  
  \caption{Training time for $100k$ steps, inference time for $110$ episodes, and GPU memory usage during training.}
  \label{tab:training_time}
  \centering
  \begin{tabular}{l l|cccc}
    \toprule
    & & SG-RP & SGM-MRP & SGM-HalfRP & SGM-GumbelRP \\
    \midrule
    \multirow{2}{*}{Pendulum} 
      & CPU Training & 317s   & 455s & 445s   & 435s \\
      & CPU Inference & 2.1s   & 3.0s & 3.0s   & 3.0s \\
    \midrule
    \multirow{3}{*}{HalfCheetah} 
      & GPU Training & 618s & 818s & 797s & 807s  \\
      & GPU Inference & 23s & 30s & 30s & 31s  \\
      & GPU Memory  & 554MB   & 556MB & 554MB   & 554MB    \\
    \bottomrule
  \end{tabular}
\end{table}

\subsection{Additional Plots for Synthetic Bandits Experiments}
\label{sec:app_sub_multimodal_bandits}

Figure \ref{fig:multimodal_bandit_each} shows the learning curves of SG-RP and SGM-MRP with their best hyperparameter setting in all $100$ synthetic multimodal bandits. We can see that SGM-MRP outperforms SG-RP in many bandits where they are not tied.

\begin{figure}[htb]
    \centering
    \includegraphics[width=0.8\linewidth]{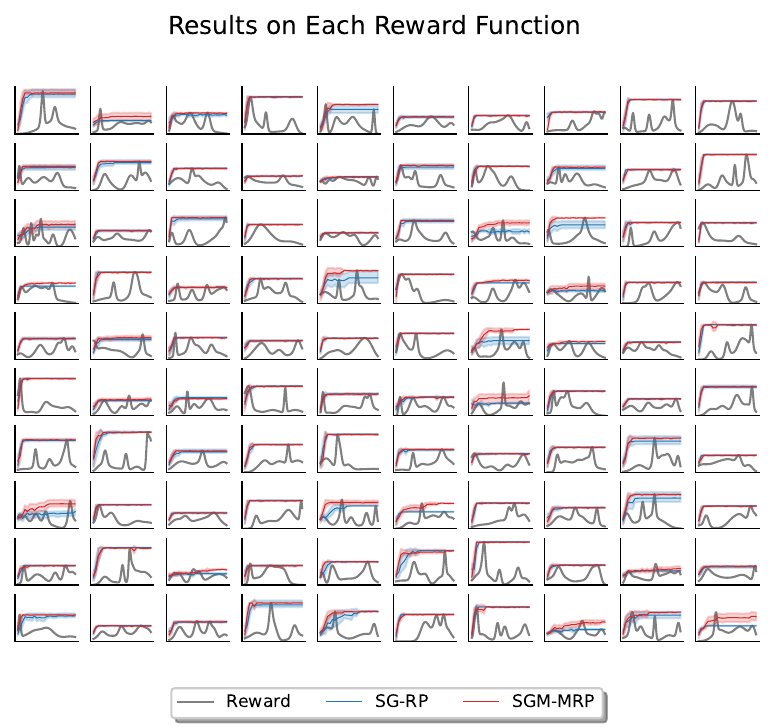}
    \caption{Reward function and the corresponding learning curves for each synthetic bandit. Each bandit is run for $10$ seed. The shade area shows the $95\%$ bootstrap CIs.}
    \label{fig:multimodal_bandit_each}
\end{figure}

\Cref{fig:multimodal_bandit_mixture} shows the average reward in the final $10\%$ of training steps for the hyperparameter setting with the best final reward across different $\alpha$ for different estimators. We can see that the SGM-MRP is consistently better than alternatives.

\begin{figure}[htb]
    \centering
    \resizebox{.34\textwidth}{!}{
    \includegraphics[height=2.8cm]{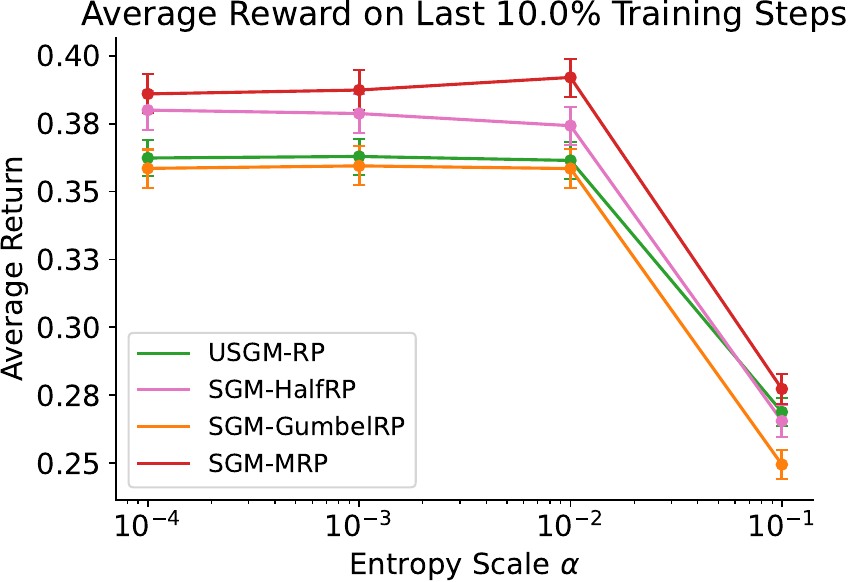}
    }
    \caption{Average performance across $100$ bandits at different entropy scales. The shaded areas and error bars show $95\%$ bootstrap CIs across $100$ bandits ($10$ runs each).}
    \label{fig:multimodal_bandit_mixture}
\end{figure}

\subsection{Additional Plots for Common Continuous Control Benchmarks}
\label{sec:app_additional_mujoco_dmc_plots}

\Cref{fig:mujoco_dmc_learning_curves,fig:mujoco_dmc_learning_curves_additional} show the learning curves in each of the MuJoCo, DeepMind Control (DMC), MetaWorld, and MyoSuite  environments. We can see that the performance of SGM-MRP is quite similar to SG-RP across different environments except for a few environments in MetaWorld.

\begin{figure}[htb]
    \centering
    \includegraphics[width=\textwidth]{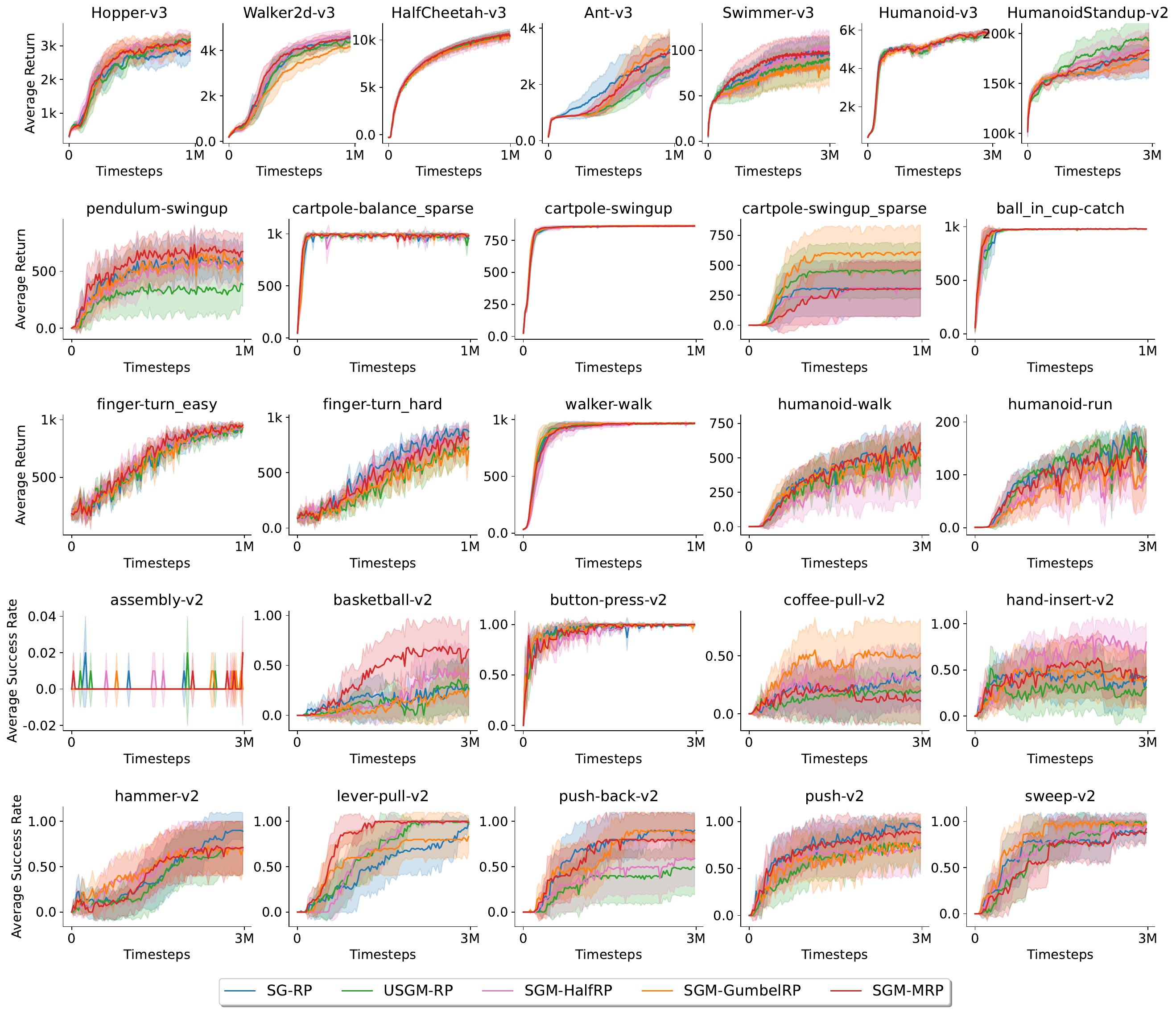}
    \caption{Learning curves in $27$ environments from Gym MuJoCo, DMC, and MetaWorld.}
    \label{fig:mujoco_dmc_learning_curves}
\end{figure}

\begin{figure}[htb]
    \centering
    \includegraphics[width=\textwidth]{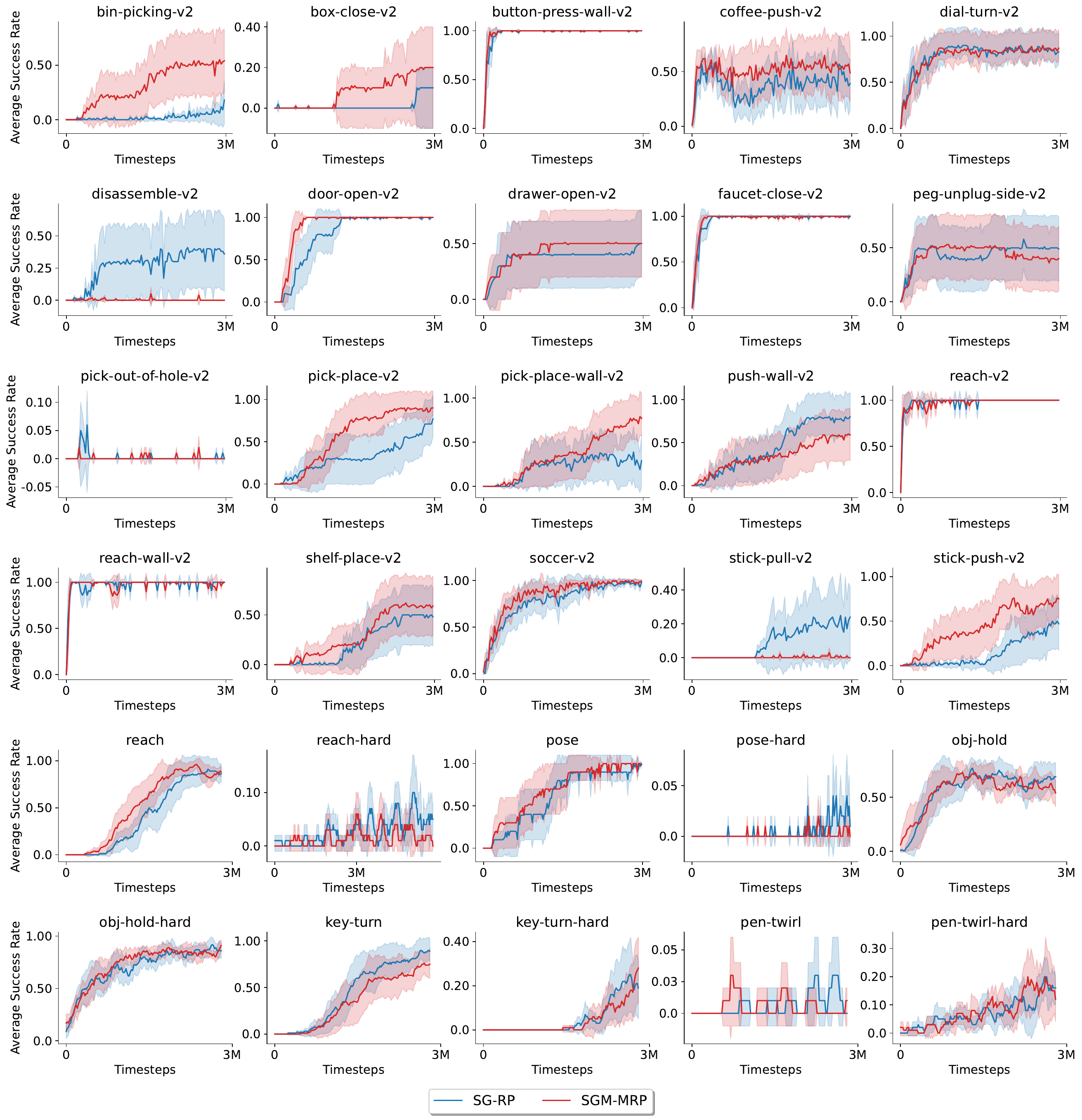}
    \caption{Learning curves in $30$ additional environments from MetaWorld and MyoSuite.}
    \label{fig:mujoco_dmc_learning_curves_additional}
\end{figure}

{
\Cref{fig:mw_failures_learning_curves} shows the learning curves in two MetaWorld environments where mixture policies fail to achieve meaningful successful rate. We can see that SGM-MRP still learns steadily in terms of average return in these cases. The insignificant success rate might be due to SGM-MRP's convergence to a stable but non-successful mode, whereas SG-RP occasionally discovers the success trajectory.
}

\begin{figure}[htb]
    \centering
    \resizebox{.5\textwidth}{!}{
    \includegraphics[height=3cm]{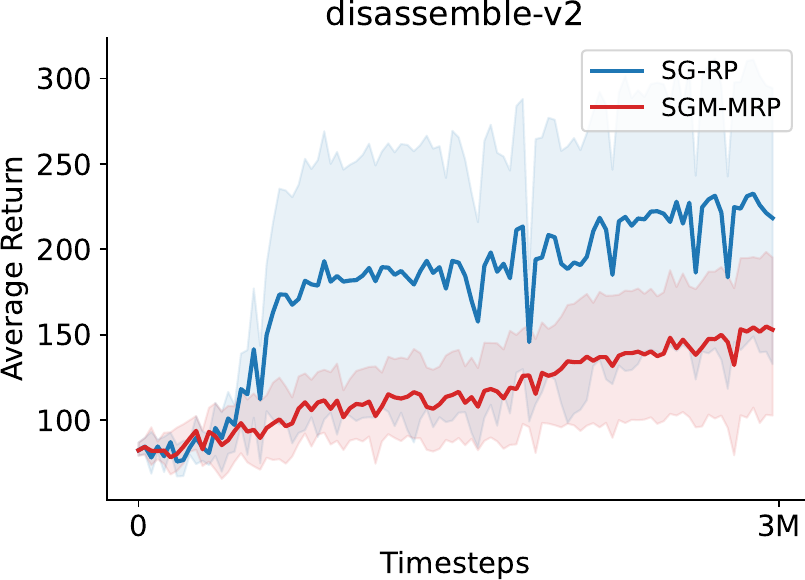}
    \includegraphics[height=3cm]{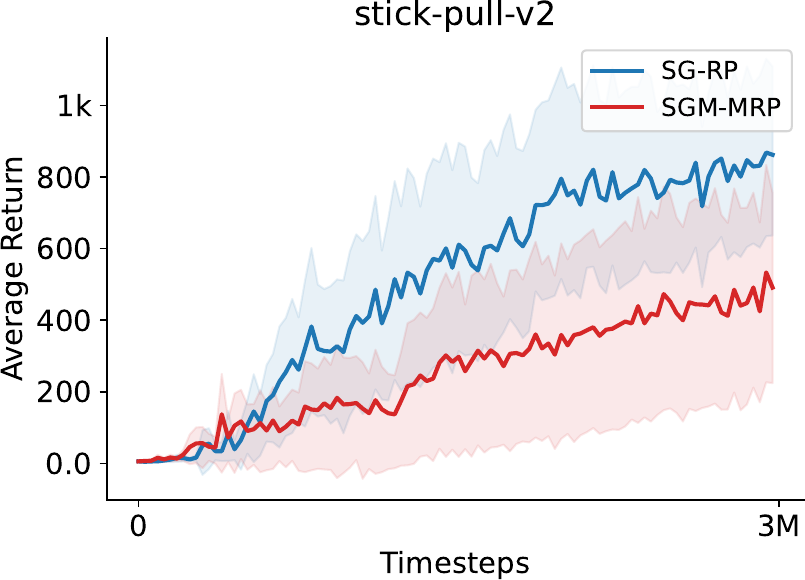}
    }
    \caption{\postsubmission Learning curves in two MetaWorld environments where mixture policies fail to achieve meaningful successful rate.}
    \label{fig:mw_failures_learning_curves}
\end{figure}

\Cref{fig:sensitivity_mujoco_two} shows the sensitivity of mixture policies and base policies in the remaining two MuJoCo environments, in which the difference is smaller than those in \Cref{fig:sensitivity_mujoco}.

\begin{figure}[htb]
    \centering
    \resizebox{0.5\textwidth}{!}{
    \includegraphics[height=3cm,trim=0 5 0 5,clip]{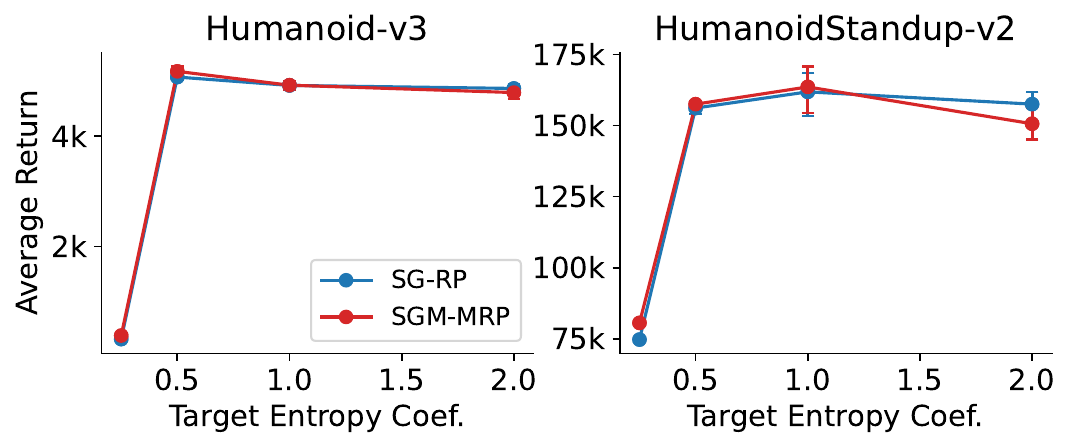}
    }
    \captionsetup{skip=6pt}
    \caption{Performance with different target entropy coefficients for mixture policies and base policies in the remaining two MuJoCo environments. Results are averaged over $10$ runs.
    }
    \label{fig:sensitivity_mujoco_two}
\end{figure}

\subsection{Classic Control Environment Details}
\label{sec:app_classic_control_rewards}

We use the $\mathsf{v1}$ version of Pendulum from OpenAI Gym for \texttt{ShapedPendulum}. For \texttt{Pendulum}, we set the reward to $1$ if the angle of the pendulum from the upright position is smaller than $0.25$ and $0$ otherwise. For \texttt{ShapedAcrobot}, we set the reward to be $-\cos(\theta_1) - \cos(\theta_2 + \theta_1) - 1.0$, where $\theta_1$ and $\theta_2$ are the first two dimensions of the state. For \texttt{ShapedMountainCar}, we set the reward to be $x-0.6$, where $x$ is first dimension of the state. For \texttt{Acrobot} and \texttt{MountainCar}, we adapt the discrete version in Gym to the continuous action case as it is done in \citet{neumann2022greedy}. All environments use a discount factor of $0.99$. The episode cut-offs for the Pendulum, Acrobot, and MountainCar are $200$, $1000$, and $1000$, respectively.

\subsection{Additional Plots for Classic Control Environments}
\label{sec:app_sub_visualization}

\paragraph{Results comparing SGM-MRP and SG-RP in the unshaped-reward setting.} \Cref{fig:classic_learning_curves_shaped} shows learning and sensitivity curves. Compared to those in \Cref{fig:classic_learning_curves}, the improvement of SGM-MRP over SG-RP is much smaller (note the range of y-axis in the sensitivity curves).

\paragraph{Results comparing different estimators for mixture policies.} 
From \Cref{fig:classic_learning_curves_mixture}, we can see that MRP is consistently the best estimator for mixture policies. Note that SGM-HalfRP has poor performance in MountainCar environments, potentially due to its higher variance compared to others.

\paragraph{Limitations of fixed weighting policies.}
We discuss the drawbacks of HalfRP and Gumbel RP estimators in \Cref{sec:experiments_estimators}. Here, we discuss the limitation of the other alternative, using a fixed weighting policy. First, restricting the weighting scheme reduces the flexibility of the policy class, which may be undesirable. Second, mixture policies with fixed weights often require more significant parameter updates when transitioning between distributions. For instance, when all modes have collapsed to a single mode, it is more challenging for fixed-weight mixture policies to introduce a new mode far from the current mode, as the component locations are constrained near the existing mode due to the non-zero fixed weights. In contrast, mixture policies with learnable weights can focus on a specific mode while keeping other components positioned far away, as their negligible weights minimize their impact on the resulting distribution. In scenarios where the mixture policy needs to introduce a new mode, a learnable-weight policy can simply adjust the mixing weight of a component that is already near the desired mode, enabling more efficient adaptation.

\begin{figure}[htb]
    \centering
    \resizebox{\textwidth}{!}{
    \includegraphics[height=3cm,trim=0 5 0 5,clip]{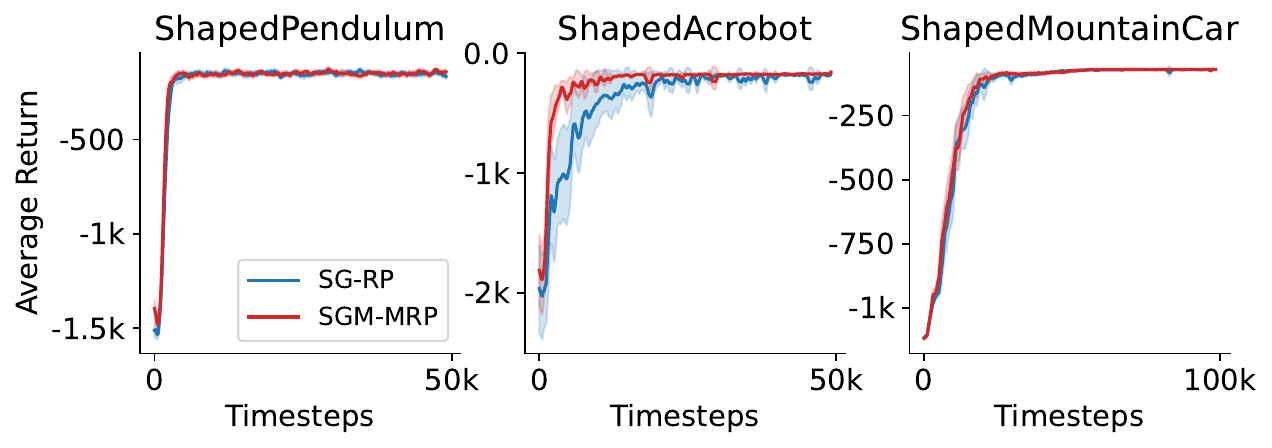}
    \includegraphics[height=3cm,trim=0 5 0 5,clip]{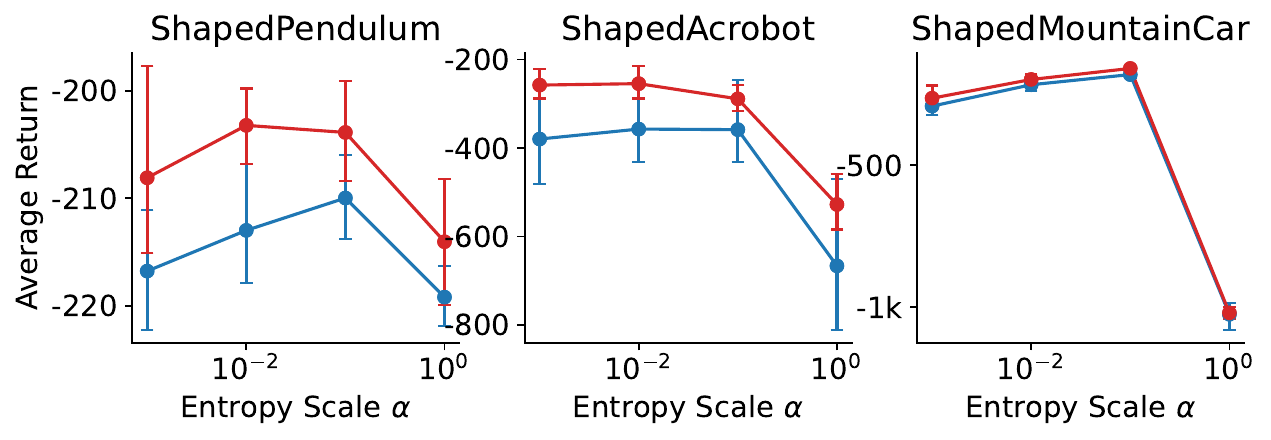}
    }
    \caption{Learning and sensitivity curves for classic control environments with shaped rewards.}
    \label{fig:classic_learning_curves_shaped}
\end{figure}

\begin{figure}[htb]
    \centering
    \includegraphics[width=\textwidth]{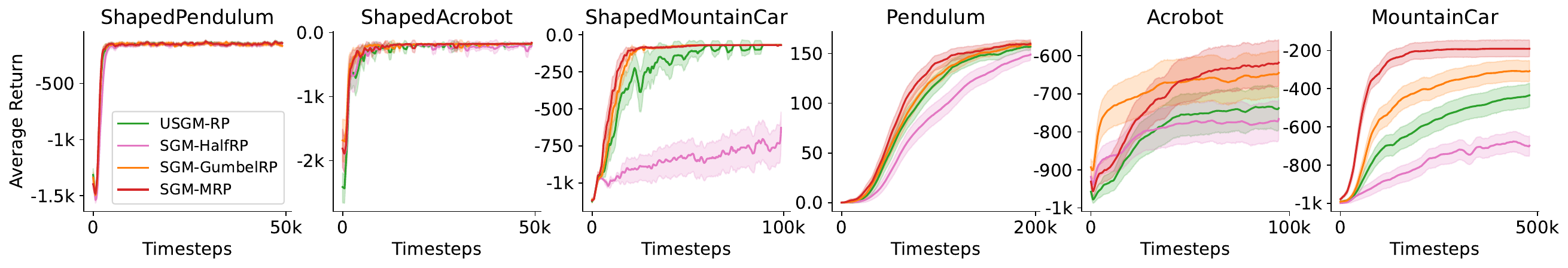}
    \includegraphics[width=\textwidth]{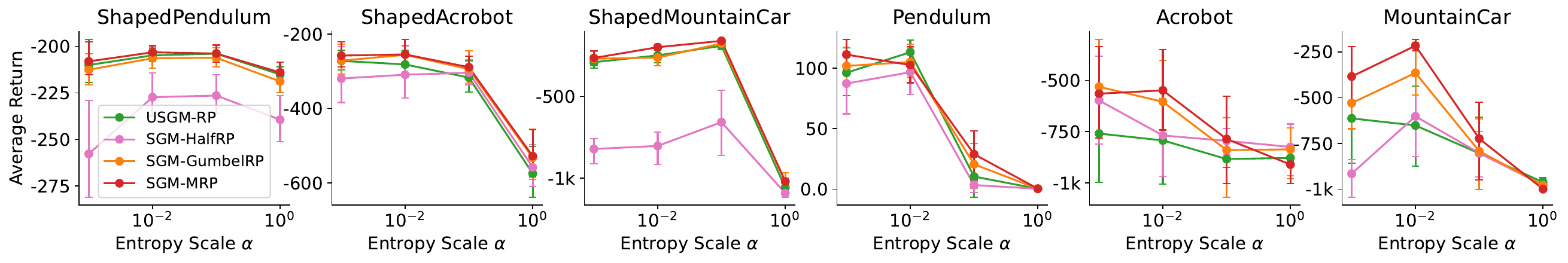}
    \caption{Learning curves of the best hyperparameter setting of different estimators for mixture policies in classic control environments.}
    \label{fig:classic_learning_curves_mixture}
\end{figure}

{

\paragraph{Visualization of state visitation.} To quantify exploration coverage and visualize state visitation, we collect visited states in $30$ reruns of the best hyperparameter setting for SG-RP and SGM-MRP in $\texttt{MountainCar}$. 
Specifically, we sample $10,000$ states after learning for $y\in\{0, 4000, 8000\}$ training steps. We then bin the states aggregating across different reruns and steps into $50\times50$ bins, with $50$ bins per state dimension. Finally, we subtract the frequency of each bin for SGM-MRP from that for SG-RP and plot the difference in log scale. In \Cref{fig:vis_heatmap_mc_more}, we show the results as well as per-algorithm visitation.
The starting states are around $(-0.5, 0.0)$, and a successful trajectory would spiral away from $(-0.5, 0.0)$ and reach the dash line.
In \Cref{fig:vis_heatmap_mc_more}, the difference in visitation on states far away from the starting states demonstrates the better exploration efficiency of mixture policies.
}

\paragraph{Visualization of the action-value estimates and policy density for base policies.} 
We presented the visualization of the action-value estimates and policy density at the starting state (\Cref{fig:mc_starting_state}) for mixture policies in Figure \ref{fig:vis_critic_mc} to highlight the difference between \texttt{ShapedMountainCar} and \texttt{MountainCar}, one with shaped rewards and the other with unshaped rewards. 
In Figure \ref{fig:vis_critic_mc_base_policy}, we show the same plot for base policies, SG-RP. The difference between the two types of environments is not so much different from that in Figure \ref{fig:vis_critic_mc}, but we can see the density is always unimodal during training, indicating less efficient exploration.

{
\paragraph{Mixture policy statistics during training.}
We investigate three statistics of mixture policies: 
\begin{itemize}
    \item Entropy. We measure the entropy of the weighting policies to see how much mixture policies collapse to choosing only one components.
    \item Component separation. This metric measures the pairwise distance between the (squashed) means of all components.
    \item Active component separation. This metric measures the pairwise distance between the means of components whose weight is larger than $0.01$.
\end{itemize}
From \Cref{fig:mixture_statistics}, there are several interesting observations. First, we can see that while the entropy of the weighting policies decay as training goes on, it stays above zero. Second, in most environments, both component separation metrics maintain a relatively high value (note that the action space is $[-1,1]$), showing that components maintain meaningful separation during training. Third, the level of entropy and component separation varies across environments. This makes sense as the rewards and the optimal entropy coefficient (see \Cref{tab:best_hypers_classic_control}) for different environments differ from each other.
}

\begin{figure}[htb]
    \centering
    \resizebox{.9\textwidth}{!}{
    \includegraphics[height=3cm]{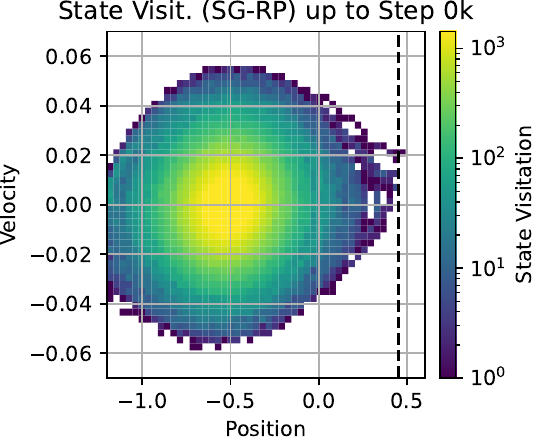}
    \includegraphics[height=3cm]{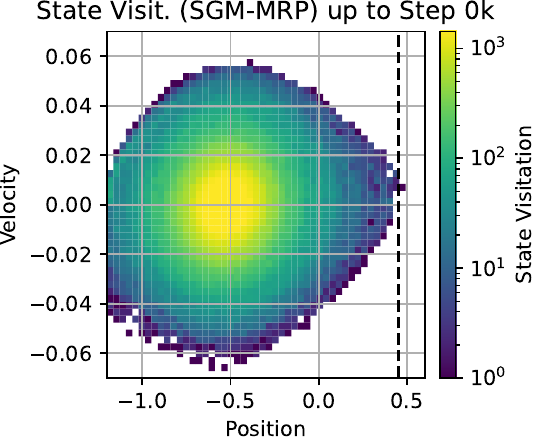}
    \includegraphics[height=3cm]{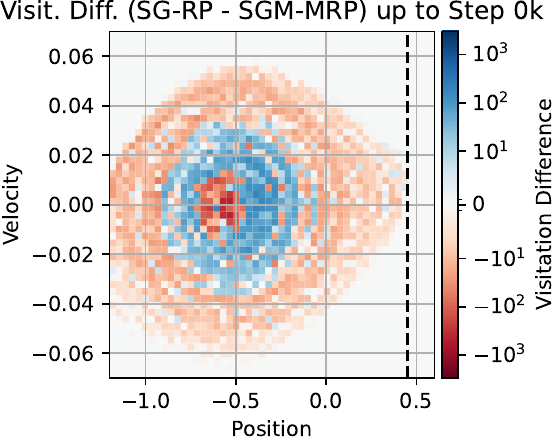}
    }
    \resizebox{.9\textwidth}{!}{
    \includegraphics[height=3cm]{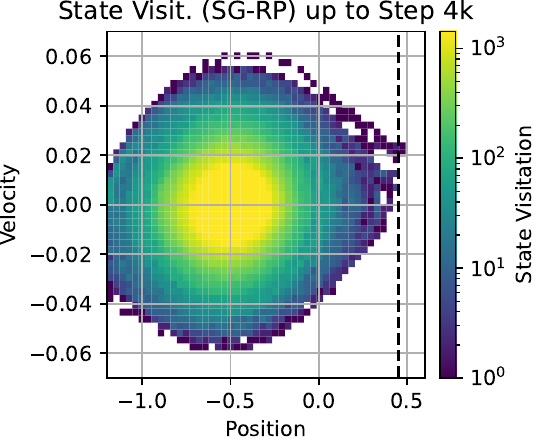}
    \includegraphics[height=3cm]{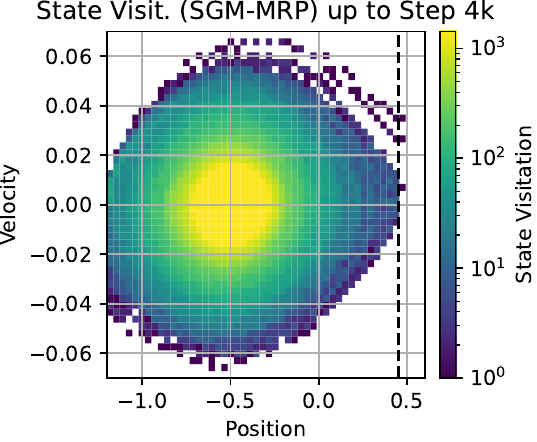}
    \includegraphics[height=3cm]{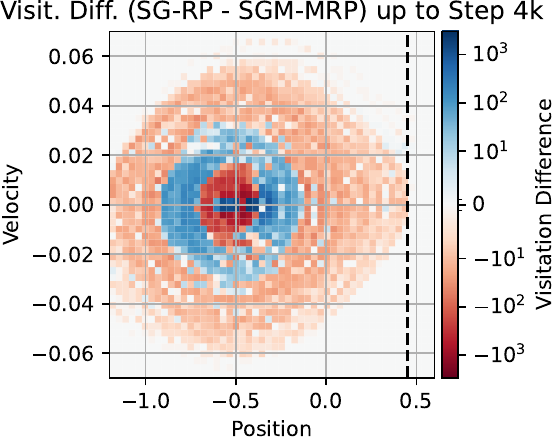}
    }
    \resizebox{.9\textwidth}{!}{
    \includegraphics[height=3cm]{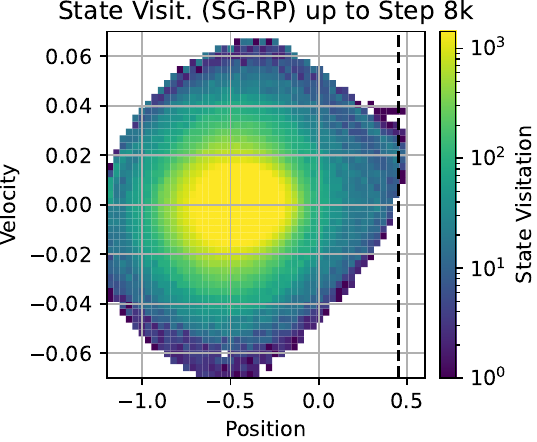}
    \includegraphics[height=3cm]{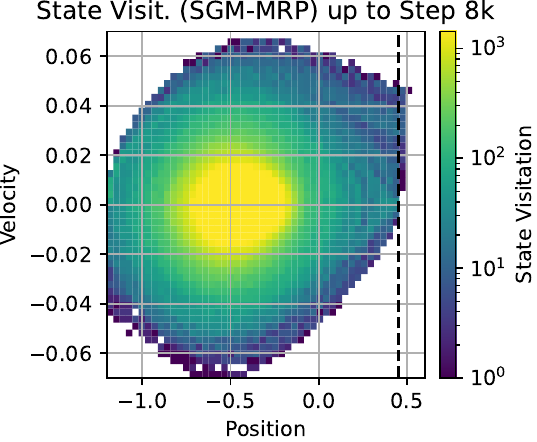}
    \includegraphics[height=3cm]{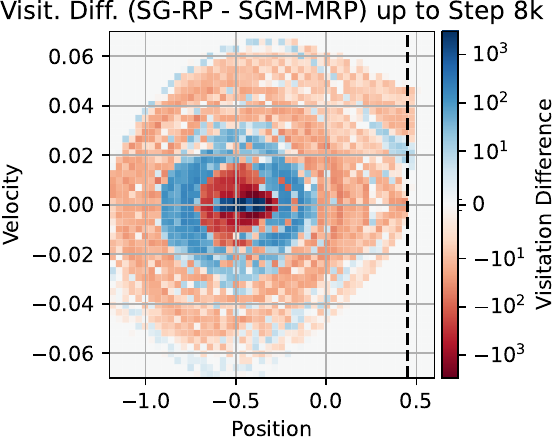}
    }
    \caption{\postsubmission State visitation of SG-RP and SGM-MRP during early training in \texttt{MountainCar}.}
    \label{fig:vis_heatmap_mc_more}
\end{figure}

\begin{figure}
    \centering
    \resizebox{.25\textwidth}{!}{
    \includegraphics[height=3cm, trim=0 3.5 0 2.5, clip]{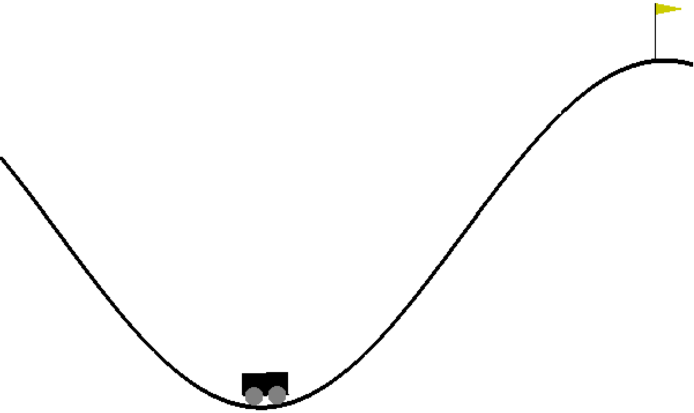}
    }
    \captionsetup{skip=6pt}
    \caption{The starting state at the bottom in MountainCar.}
    \label{fig:mc_starting_state}
\end{figure}

\begin{figure}[htb]
    \centering
    \includegraphics[width=.98\linewidth]{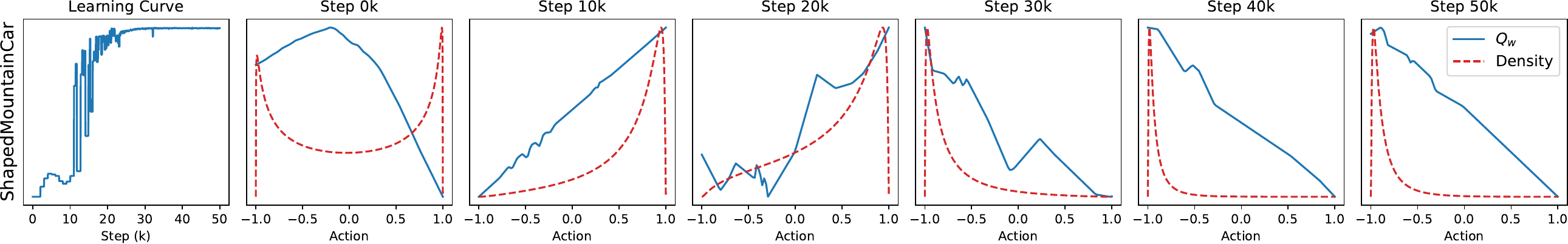}
    \includegraphics[width=.98\linewidth]{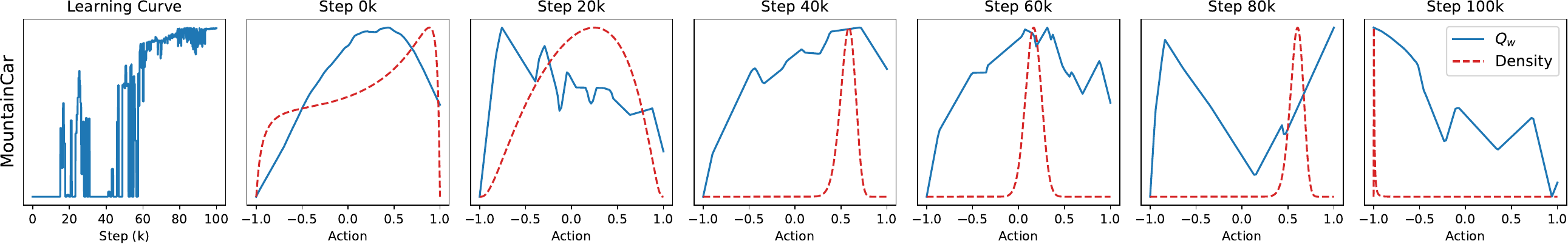}
    \caption{Learning curve, action-value estimates, and policy density at a starting state from a sample run of SG-RP in two MountainCar variants. The y-axes differ across plots and are not shown with ticks to highlight the shape of the curves rather than their exact values. The observations are similar to those in Figure \ref{fig:vis_critic_mc} with the exception that the density is always unimodal except for the starting step. It indicates that SG-RP explore less efficiently, potentially explaining its worse performance in \texttt{MountainCar}.}
    \label{fig:vis_critic_mc_base_policy}
\end{figure}

\begin{figure}[htb]
    \centering
    \resizebox{.75\textwidth}{!}{
    \includegraphics[height=3cm]{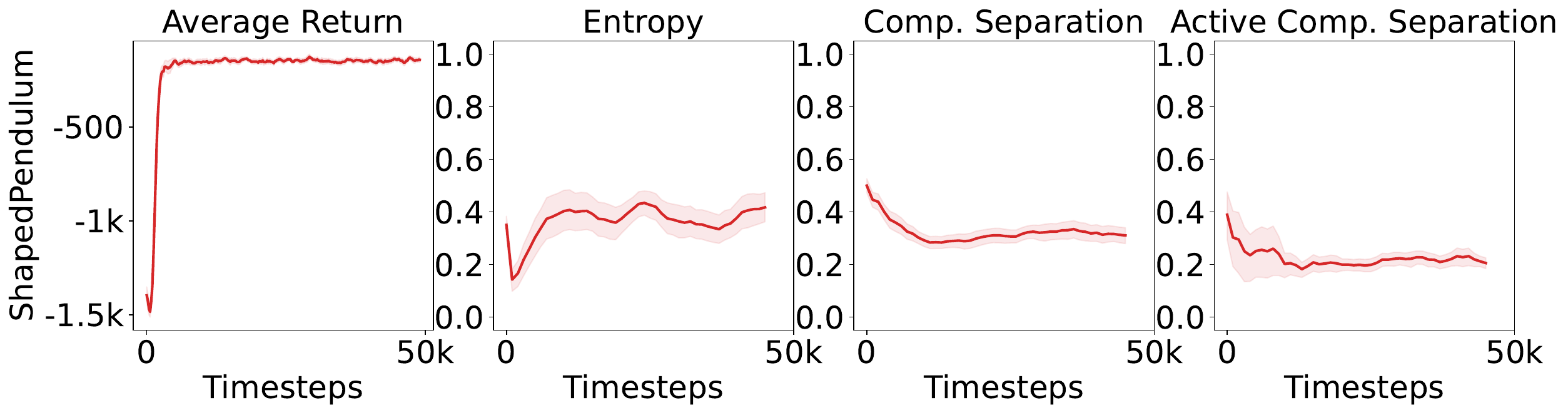}
    }
    \resizebox{.75\textwidth}{!}{
    \includegraphics[height=3cm]{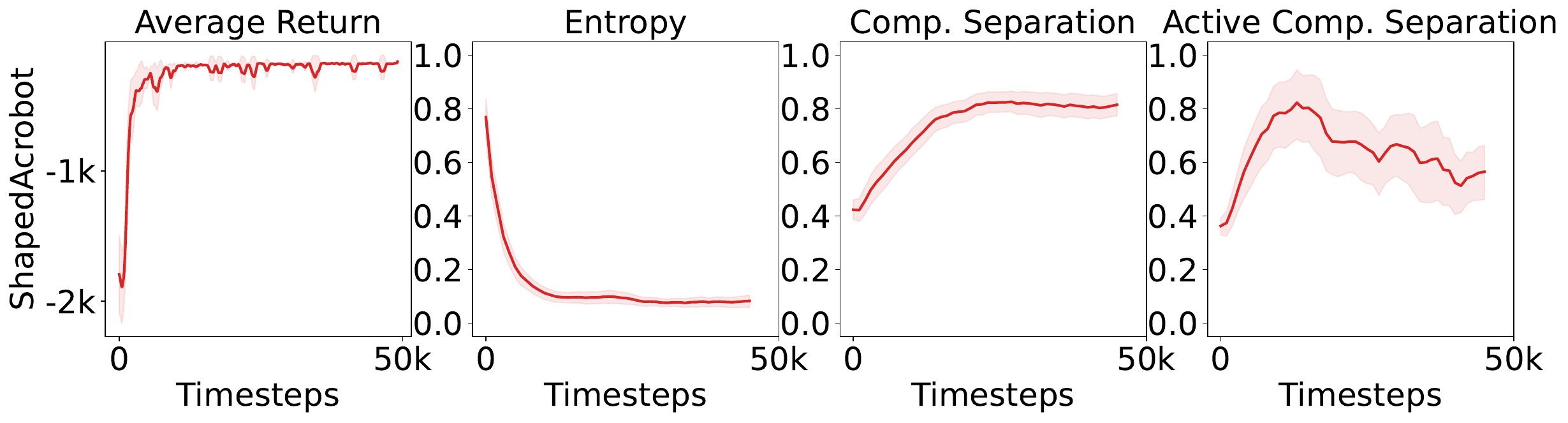}
    }
    \resizebox{.75\textwidth}{!}{
    \includegraphics[height=3cm]{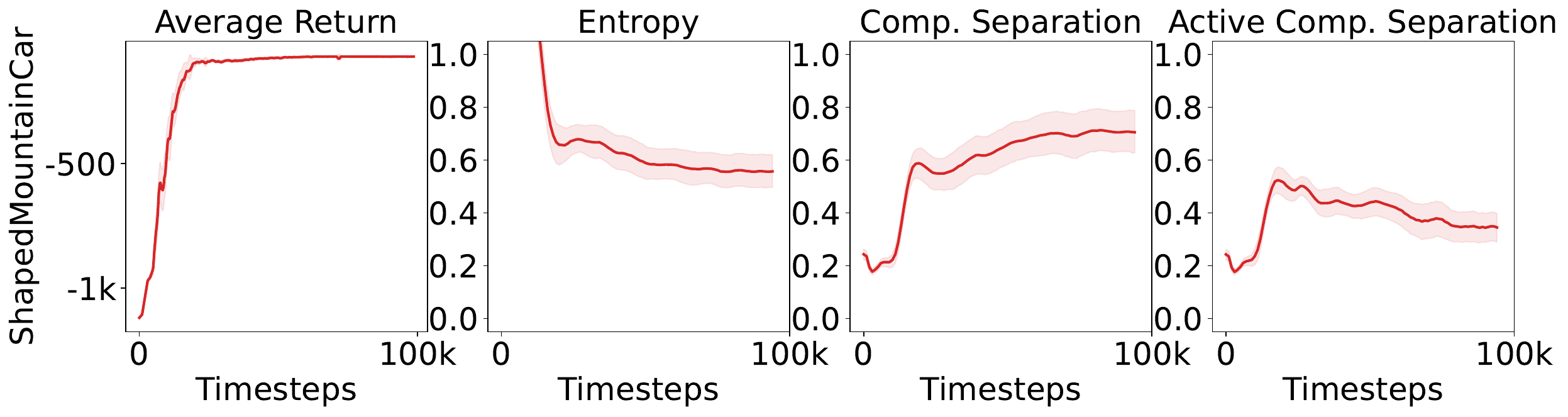}
    }
    \resizebox{.75\textwidth}{!}{
    \includegraphics[height=3cm]{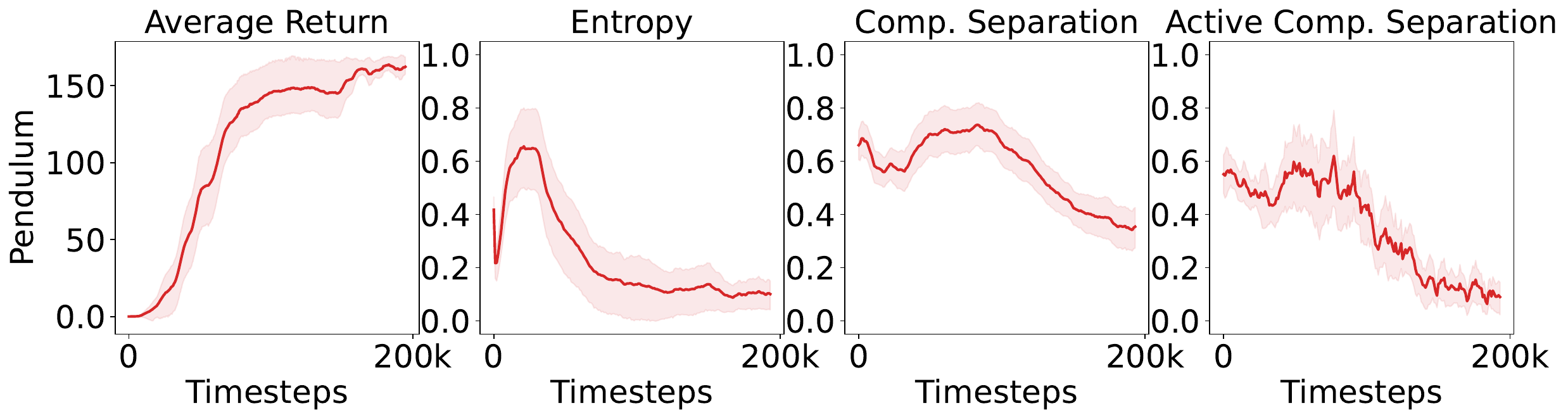}
    }
    \resizebox{.75\textwidth}{!}{
    \includegraphics[height=3cm]{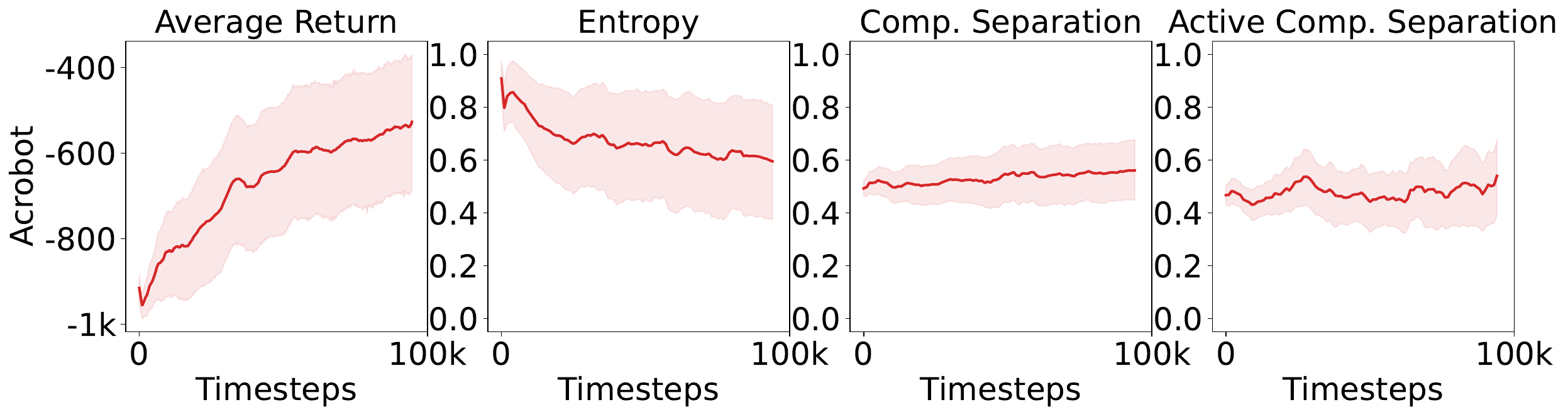}
    }
    \resizebox{.75\textwidth}{!}{
    \includegraphics[height=3cm]{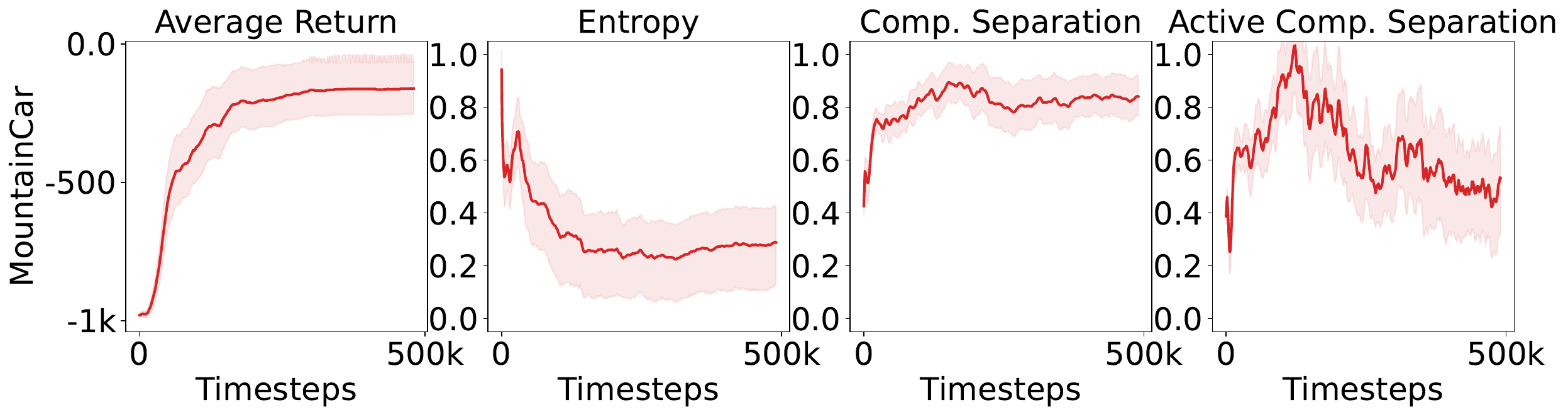}
    }
    \caption{\postsubmission Mixture policy statistics during training in classic control environments. Results are averaged across $30$ runs. The shaded areas plot the $95\%$ bootstrap CIs.}
    \label{fig:mixture_statistics}
\end{figure}

\clearpage

\section{Additional Experiments}
\label{sec:app_additional_experiments}

In this section, we provide additional experiments that further study different aspects when using mixture policies.

\subsection{Comparison Between the RP and the LR Estimators}
\label{sec:exp_lr_vs_rp}

In this section, we compare the MRP estimator with the LR estimator for optimizing mixture policies. To reduce variance of the LR estimator, we employ a baseline calculated using the average of the action values of $30$ sampled actions for a given state. 
We conduct experiments on seven Gym MuJoCo environments, following the same experimental setup as \Cref{sec:experiments}.

\textbf{The LR estimator could be unstable in high-dimensional environments.}
Figure \ref{fig:exp_lr_vs_rp} contrasts the performance distribution of SGM-LR to that of SGM-MRP. We can see that SGM-LR generally performs worse than SGM-MRP. Despite using a baseline to reduce variance, SGM-LR remains unstable, with divergent runs in several environments, particularly those with higher state and action dimensions (see Table \ref{tab:mujoco_env_dimensions} in Section \ref{sec:app_sub_experimental_details_app}). This instability is evident in environments like \texttt{Humanoid} and \texttt{HumanoidStandup}, where SGM-LR exhibits a significant number of divergent runs (shown in black), while SGM-MRP consistently achieves high returns. Additionally, even for non-divergent runs, SGM-LR often fails to learning meaningful policies, resulting in lower return distributions (as in \texttt{Walker2d}, \texttt{Ant}, and \texttt{Swimmer}). The variance in gradient estimates remains a significant issue for likelihood-ratio-based methods, especially in high-dimensional environments.
Overall, these results demonstrate that RP-based estimators provide greater stability than the LR estimator. This makes them a more effective choice for training mixture policies.

To empirical verify assumptions of \Cref{prop:variance_reduction_main} and compare the variance of the MRP and LR estimators, we perform a variance analysis on representative tasks. Specifically, we measure the variance terms in \Cref{assm:when_importance_sampling_is_bad_main} and the variance of three different estimators: LR without a baseline, LR (with a baseline; same as above), and MRP. Between every $1000$ training steps, we estimate the noise of each gradient estimator by fixing the current actor parameters and mini-batch, then recomputing the actor gradient multiple times $M=128$ with independent stochastic samples from that estimator. For estimator $e$, this yields gradients $\{g_e^{(j)}\}_{j=1}^M$, from which we compute the sample mean gradient $\bar g_e$ and per-parameter sample variance; we report the variance trace of $\sum_i \mathrm{Var}[g_{e,i}]$ (total gradient variance). To ensure fair comparison, all estimators are evaluated on the same mini-batch with the same repetition budget, and measurements are diagnostic only (no parameter update during variance estimation). The policy is updated separately with the MRP estimator after variance estimation as the LR estimator is unstable even with a baseline as discussed above.

\textbf{The MRP estimator has significantly lower variance.}
From \Cref{fig:exp_lr_vs_rp_variance} (columns 2-3), we can see that the variance assumptions do often hold in practice. We directly show the average frequency (across seeds) with which the assumptions hold: column 2 shows the results for the trace of the covariance of gradient variance (\Cref{eq:variance_assm_gradient}), and column 3 shows the results for the reward (in this case, critic) variance (\Cref{eq:variance_assm_reward}).
From Figure 2 (column 1), we can see that even if Assumption 4.6 does not always hold at all steps, the MRP estimator shows significantly lower variance than the LR estimator, even when the LR estimator is equipped with a baseline.

\begin{figure}[htb]
    \centering
    \includegraphics[width=\linewidth]{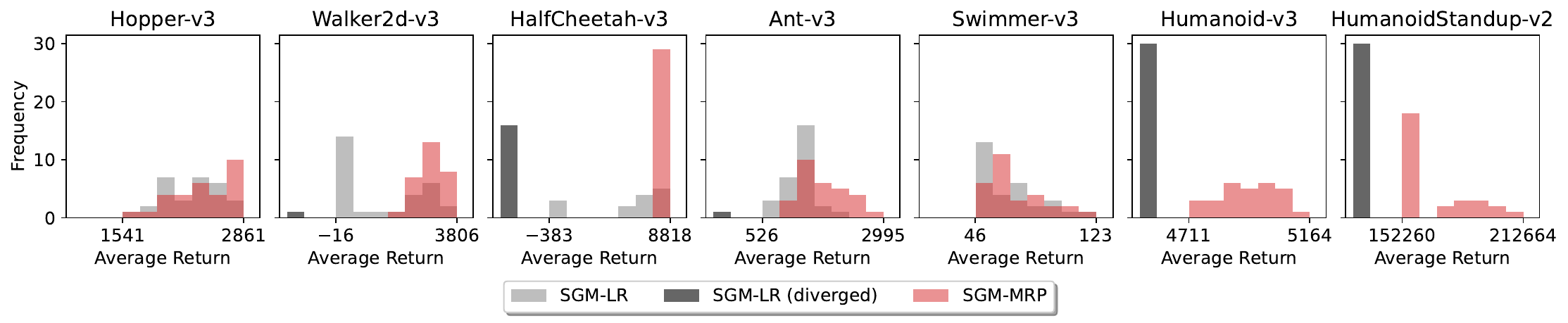}
    \caption{Comparative distribution ($30$ runs) of average return for squashed Gaussian mixture (SGM) policies in Gym MuJoCo environments, comparing the LR and the ERP estimators.
    }
    \label{fig:exp_lr_vs_rp}
\end{figure}

\begin{figure}[htb]
    \centering
    \includegraphics[width=\linewidth]{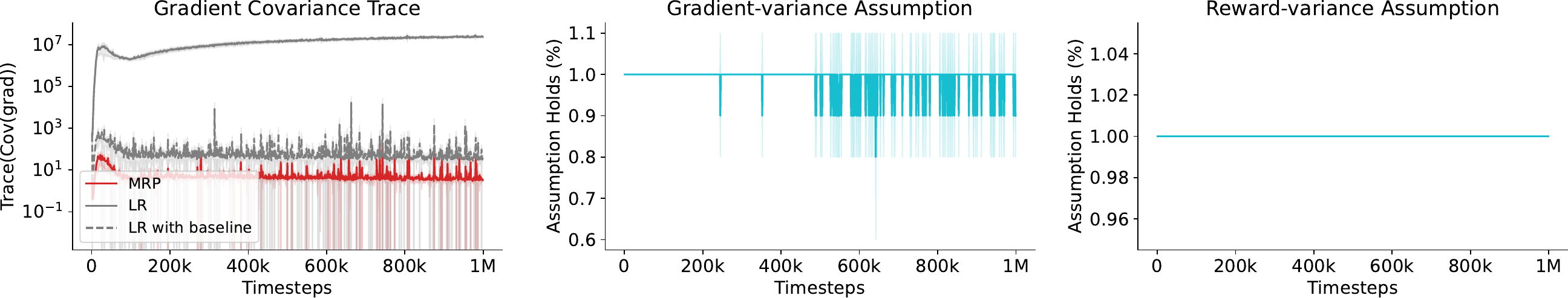}
    \includegraphics[width=\linewidth]{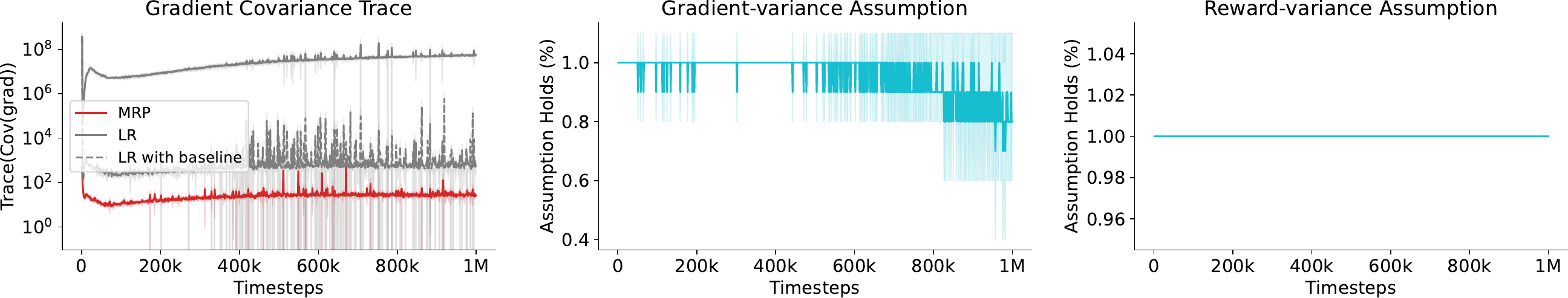}
    \includegraphics[width=\linewidth]{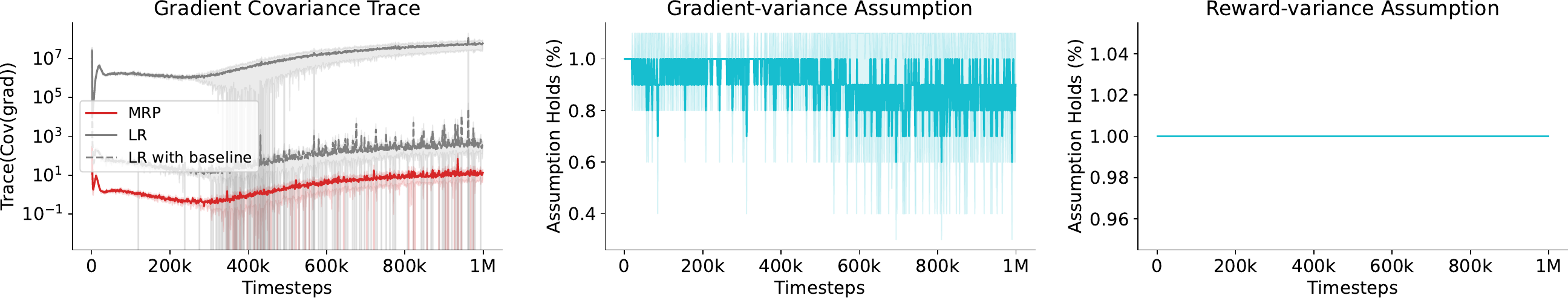}
    \includegraphics[width=\linewidth]{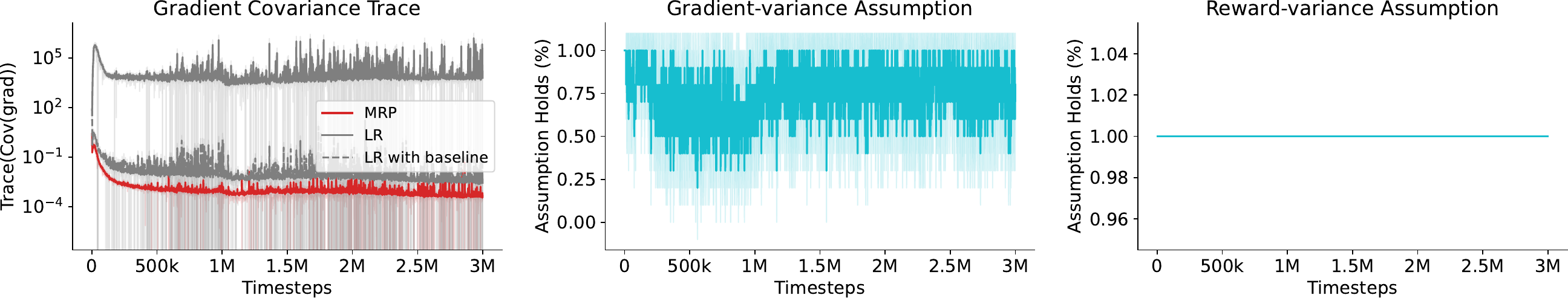}
    \caption{
    Gradient covariance trace of different estimators (col 1), percentage of seeds in which gradient-variance (col 2) and reward-variance (col 3) assumptions hold. Results are based on 10 random seeds, with shaded areas showing $95\%$ bootstrap confidence intervals. From top to bottom: \texttt{Hopper-v3}, \texttt{Walker2D-v3}, \texttt{Ant-v3}, \texttt{Swimmer-v3}. We can see that 1) MRP has the lowest variance, and 2) the variance assumptions in \Cref{assm:when_importance_sampling_is_bad_main} more often hold than not.
    }
    \label{fig:exp_lr_vs_rp_variance}
\end{figure}

\subsection{Additional Robotic Environments with Unshaped Rewards}
\label{sec:app_exp_unshaped_rewards}

Other than the classic control environments in \Cref{sec:exp_shaped_versus_unshaped}, we also investigate robotic environments with unshaped rewards. Specifically, we test \texttt{FetchReach} and \texttt{FetchSlide} from \citet{plappert2018multi} as well as $10$  in-hand manipulation environments from the ShadowHand \citet{huang2021generalization}. However, both mixture policies and base policies fail without any learning progress in in-hand manipulation environments as they are too difficult. On the other hand, mixture policies appear to improve performance in the Fetch environments (\Cref{fig:fetch_envs}).

\begin{figure}[htb]
    \centering
    \resizebox{.67\textwidth}{!}{
    \includegraphics[height=3cm]{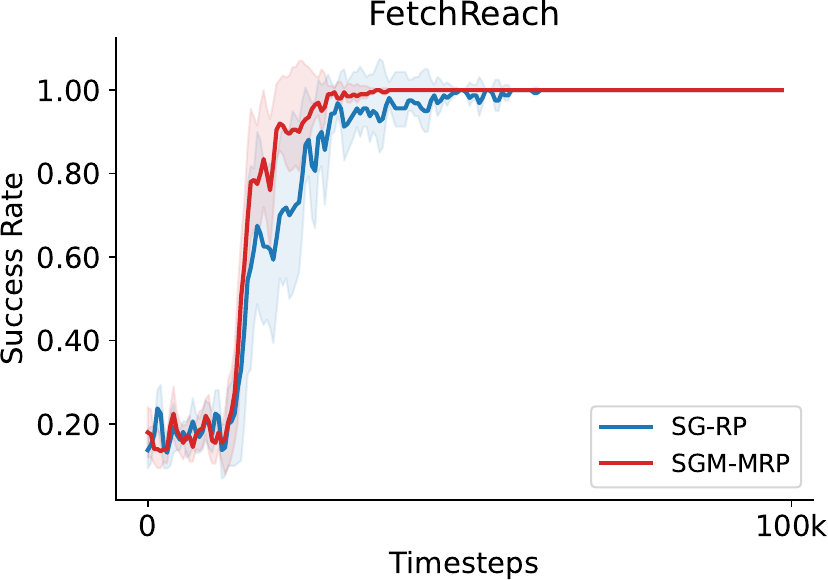}
    \includegraphics[height=3cm]{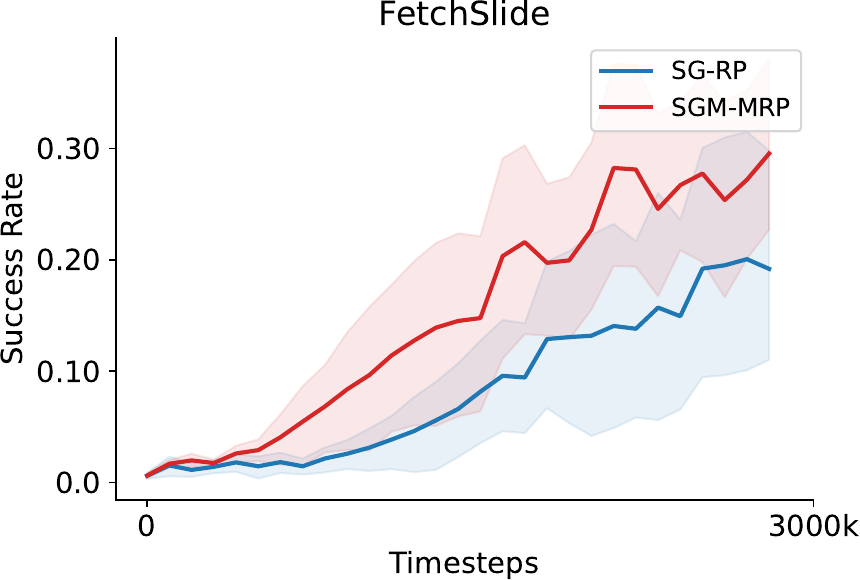}
    }
    \caption{Learning curves in two robotic Fetch environments from \citet{plappert2018multi}. The shaded areas plot the $95\%$ bootstrap CIs across $8$ runs.}
    \label{fig:fetch_envs}
\end{figure}

\subsection{{Effect of the Number of Components}}
\label{sec:app_sub_num_of_components}

We use mixture policies with five components in all experiments in the main text. Here, we study the effect of this choice in classic control environments with unshaped rewards, where differences between the base and the mixture policies are more prominent. We use the same experiment protocol as in Section \ref{sec:exp_shaped_versus_unshaped} and test mixture policies with two and eight components using the MRP gradient estimator. Specifically, we sweep the same hyperparameters for each variant of mixture policies.

Figure \ref{fig:sgm_num_components} shows 
the sensitivity to the entropy scale. We can see that while the results are noisy and not quite consistent across environments, mixture policies with various numbers of components generally outperform the base policy. 
Despite significant noise in the results, we can see that mixture policies with various numbers of components are similarly effective.

{\postsubmission
We further test mixture policies with different number of components in high-dimensional continuous control environments. We use the same environments in \Cref{fig:meta_world_highlight} where the performance gap between mixture policies and base policies is more obvious. The results are shown in \Cref{fig:meta_world_highlight_sgm_num_comp}. While performance varies with $N$, the overall conclusion is similar to that in classic control domains: while $N=2$ always slightly underperforms $N=5$, mixture policies are similarly effective with different numbers of components.
}

\begin{figure}[htb]
    \centering
    \resizebox{.6\textwidth}{!}{
    \includegraphics[height=3cm]{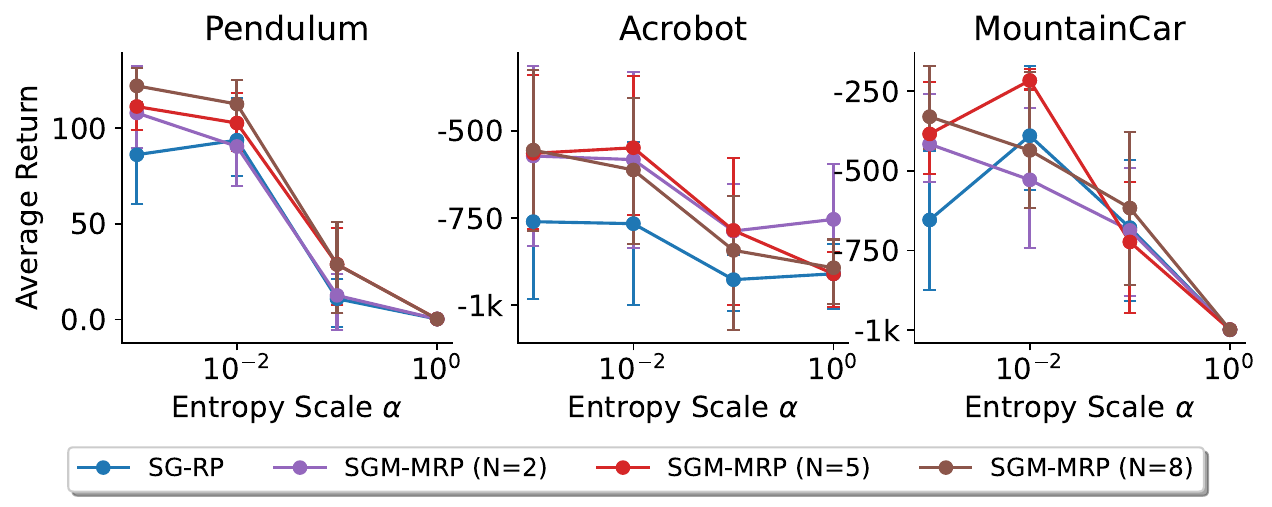}
    }
    \caption{Sensitivity curves for mixture policies with different number of components. The error bars plot the $95\%$ bootstrap CIs across $10$ runs.}
    \label{fig:sgm_num_components}
\end{figure}

\begin{figure}[htb]
\vspace{-0.4cm}
    \centering
    \resizebox{\textwidth}{!}{
    \includegraphics[height=3cm,trim=0 5 0 5,clip]{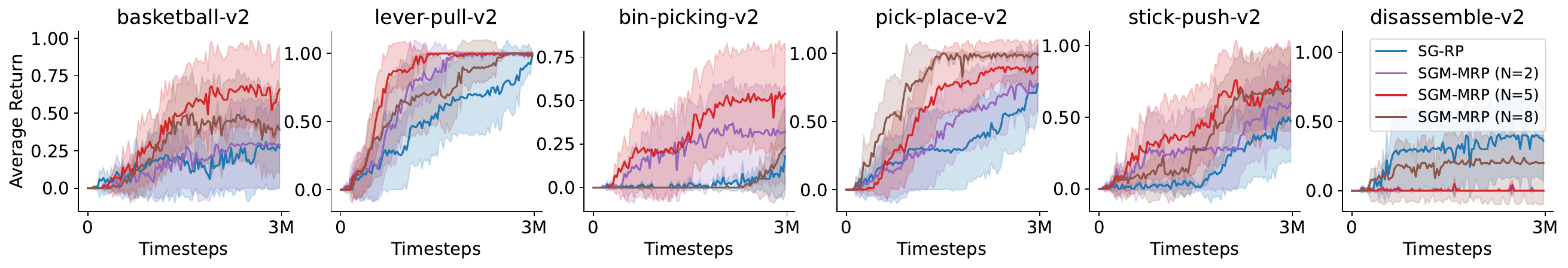}
    }
    \captionsetup{skip=6pt}
    \vspace{-0.3cm}
    \caption{Learning curves for mixture policies with different number of components in selected MetaWorld tasks where the performance gap between mixture policies and base policies is most pronounced. 
    Shaded areas show 95\% bootstrap CIs over $10$ runs.
    }
    \label{fig:meta_world_highlight_sgm_num_comp}
\end{figure}

{
\textbf{Guidance on selecting the number of components.}
While our theoretical analysis in \Cref{sec:motivation_in_bandit,sec:reparam_mixture_policies} has not revealed potential downsides of using more components to increase the flexibility of the policy, there are practical considerations for using a restricted number of them. First, using a large number of components can increase the computational and memory cost (e.g., if $|\cA|\times N$ is close to the  hidden dimension size). Second, from the above empirical results, increasing $N$ from 5 to 8 does not consistently bring benefits, whereas decreasing $N$ from 5 to 2 appears to consistently worsen the performance.
Based on the above points, we recommend to use a default of $N=5$ and tune upwards or slightly downwards if needed.

}

{
\subsection{{Effect of Using a Heavy-Tailed Base Policy}}
\label{sec:app_sub_heavy_tailed_base_policy}

In principle, the base policy can be any policy and even be different across different components. Here, we consider Cauchy policy \citep{bedi2024sample} as the base policy, which is heavy-tailed and promotes persistent exploration. Our hypothesis is that \textit{using mixture policies would also provide benefits when the base policy is Cauchy policy.} We adopt the same experiment protocol as in Section \ref{sec:exp_shaped_versus_unshaped} and test Cauchy and Cauchy mixture (CM) policies with reparameterization gradient estimators.

Figure \ref{fig:classic_control_cauchy} shows the learning curves of the best hyperparameter setting (rerun for $200$ seeds) and the sensitivity to the entropy scale. In \texttt{Pendulum} and \texttt{MountainCar}, Cauchy mixture policies significantly outperform base Cauchy policies.
Notice that, in \texttt{Pendulum}, Cauchy-based policies perform significantly worse than the corresponding Gaussian-based policies. This might potentially be due to two reasons: 1) Cauchy-based policies are generally not suitable for this environment, or 2) the preset search range of $\alpha$ is too large for Cauchy-based policies. In summary, we conclude that even when the base policy is heavy-tailed, using mixture policies may still provide benefits in environments with unshaped rewards.

\begin{figure}[htb]
    \centering
    \resizebox{\textwidth}{!}{
    \includegraphics[height=3cm,trim=0 5 0 5,clip]{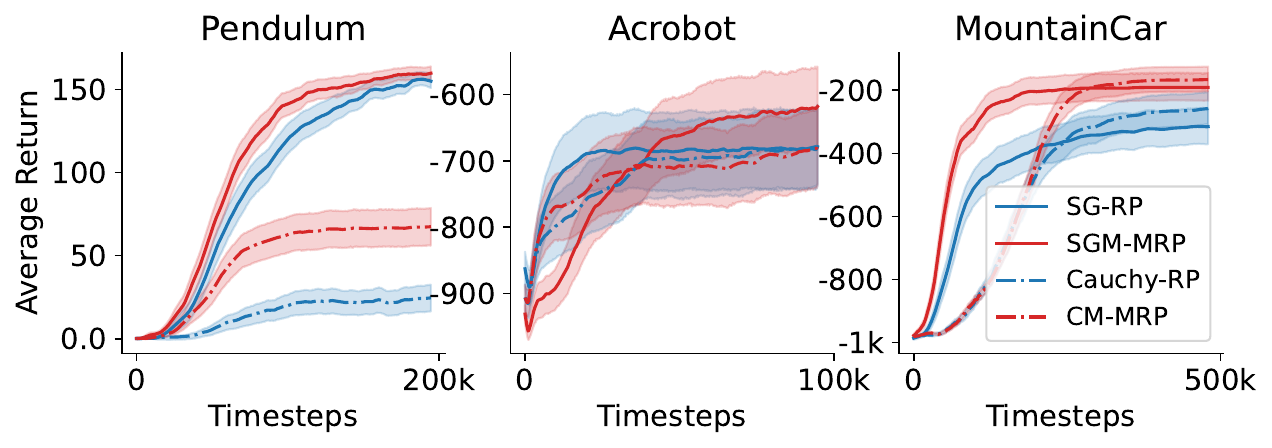}
    \includegraphics[height=3cm,trim=0 5 0 5,clip]{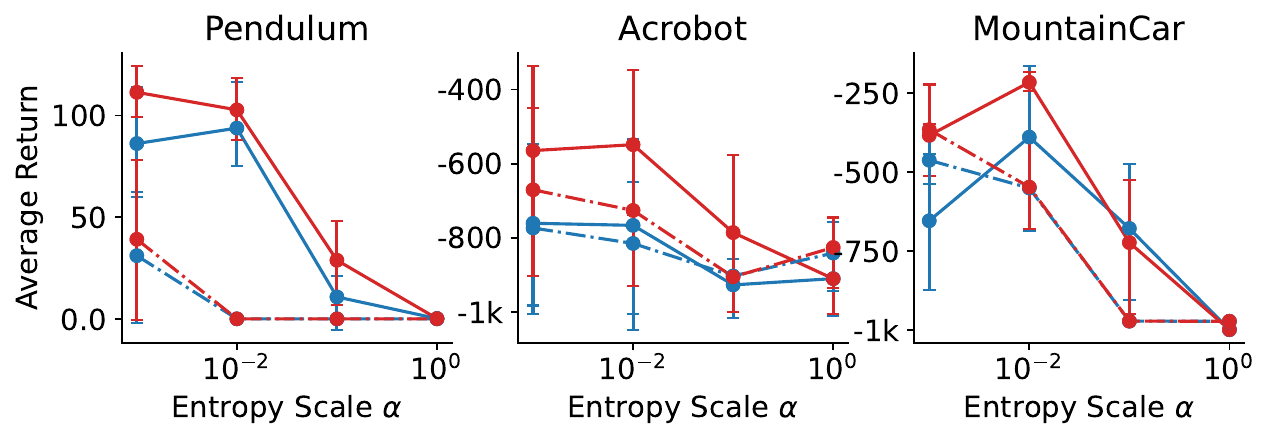}
    }
    \caption{Learning and sensitivity curves for Cauchy policies and Cauchy Mixture (CM) policies in classic control environments with unshaped rewards.}
    \label{fig:classic_control_cauchy}
\end{figure}
}

\subsection{Effect of the Temperature Parameter of GumbelRP Estimator}
\label{sec:app_gumbel_temperature}

In this section, we provide additional experiments that investigate the effect of the temperature parameter $\tau$ in GumbelRP estimator. We use two settings in this study: multimodal bandits and continuous control. In the bandit setting, we test GumbelRP with wide range of temperatures. Other hyperparameters are consistent with the best hyperparameters selected based on SGM-GumbelRP ($\tau=1.0$). For continuous control, we test two additional temperature values in \texttt{Hopper-v3}. From \cref{fig:multimodal_hopper_gumbel_temperature}, we can see that the GumbelRP estimator is not sensitive to its temperature parameter in our experiments.

\begin{figure}[htb]
    \centering
    \resizebox{.67\textwidth}{!}{
    \includegraphics[height=3cm]{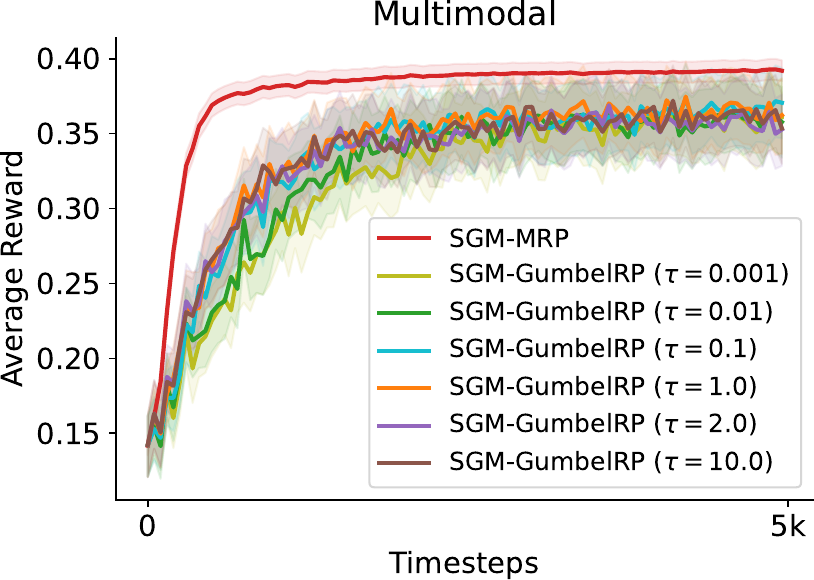}
    \includegraphics[height=3cm]{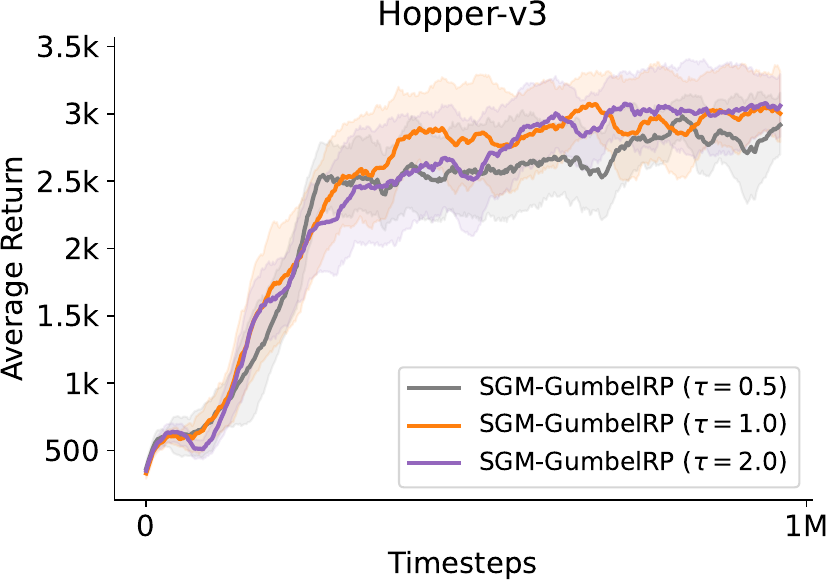}
    }
    \caption{Sensitivity analysis of the GumbelRP estimator to the temperature parameter. The shaded areas plot the $95\%$ bootstrap CIs across different runs. For \texttt{Multimodal}, the results are averaged over $100$ bandits ($10$ runs each for SGM-MRP; 1 run each for SGM-GumbelRP). For \texttt{Hopper-v3}, the results are averaged over $10$ runs.}
    \label{fig:multimodal_hopper_gumbel_temperature}
\end{figure}

\end{document}